\documentclass[10pt,twocolumn]{article} 
\usepackage{simpleConference}
\usepackage{times}
\usepackage{graphicx}
\usepackage{amssymb}
\usepackage{url,hyperref}

\usepackage{helvet}  
\usepackage{courier}  
\usepackage{graphicx} 
\usepackage{natbib}  
\usepackage{caption} 
\DeclareCaptionStyle{ruled}{labelfont=normalfont,labelsep=colon,strut=off} 
\usepackage{algorithm}
\usepackage{algorithmic}

\usepackage{xcolor}
\usepackage{color}
\usepackage{inputenc} 		
\usepackage[T1]{fontenc}    
\usepackage{hyperref}       
\usepackage{url}            
\usepackage{booktabs}       
\usepackage{amsfonts}       
\usepackage{nicefrac}       
\usepackage{microtype}      
\usepackage{float}
\usepackage{algorithm}
\usepackage{makecell}
\usepackage{multirow}
\usepackage{wrapfig}
\usepackage{lipsum}
\usepackage{hhline}
\usepackage{algorithmic}
\usepackage{amsmath,bm}
\usepackage{amsthm}
\usepackage[shortlabels]{enumitem}
\usepackage{mathrsfs}
\usepackage{xcolor}
\usepackage{graphicx}
\usepackage{color}
\usepackage{soul}
\usepackage{wrapfig}
\usepackage{subfigure}
\usepackage{bbding}
\usepackage{booktabs, caption, makecell}

\usepackage{threeparttable}

\newcommand{\alg}{$\mathsf{PRECISION}$~}
\newcommand{\algns}{$\mathsf{PRECISION}$}
\newcommand{\algplus}{$\mathsf{PRECISION}^+$~}
\newcommand{\algplusns}{$\mathsf{PRECISION}^+$}

\usepackage{pifont}
\newcommand{\cmark}{\ding{51}}%
\newcommand{\xmark}{\ding{55}}%

\allowdisplaybreaks[2]

\renewcommand{\algorithmicrequire}{}
\renewcommand{\algorithmicensure}{\textbf{Output:}}

\newcommand{\bx}{\bar{x}}

\renewcommand{\d}{\mathbf{d}}
\newcommand{\bd}{\bar{\d}}
\newcommand{\E}{\mathbf{E}}
\newcommand{\Ec}{\mathcal{E}}

\newcommand{\Eb}{\mathbb{E}}

\newcommand{\F}{\mathbf{F}}

\newcommand{\I}{\mathbf{I}}

\newcommand{\Lc}{\mathcal{L}}

\newcommand{\M}{\mathbf{M}}

\newcommand{\Mt}{\widetilde{\mathbf{M}}}

\newcommand{\Nc}{\mathcal{N}}

\newcommand{\p}{\mathbf{p}}
\newcommand{\bp}{\bar{\p}}

\newcommand{\Rb}{\mathbb{R}}
\newcommand{\Rc}{\mathcal{R}}

\newcommand{\Sc}{\mathcal{S}}

\renewcommand{\u}{\mathbf{u}}

\renewcommand{\v}{\mathbf{v}}

\newcommand{\W}{\mathbf{W}}

\newcommand{\x}{\mathbf{x}}

\newcommand{\y}{\mathbf{y}}

\newcommand{\xb}{\mathbf{\bar{x}}}

\newcommand{\1}{\mathbf{1}}

\newtheorem{thm}{Theorem}

\newtheorem{cor}[thm]{Corollary}
\newtheorem{lem}{Lemma}

\newtheorem{defn}{Definition}

\newtheorem{rem}{Remark}

\newtheorem{assum}{Assumption}

\newcommand{\vtheta}{\boldsymbol \theta}
\newcommand{\vomega}{\boldsymbol \omega}

\renewcommand{\algorithmicrequire}{\text{}}
\renewcommand{\algorithmicensure}{\textbf{Output:}}

\usepackage{hyperref}
\allowdisplaybreaks

\usepackage{bibentry}

\begin{document}
\title{
 PRECISION: Decentralized Constrained Min-Max Learning with Low Communication and Sample Complexities}

\author{ Zhuqing Liu$^{+}$, Xin Zhang$^*$, Songtao Lu$^{\circ}$, and Jia Liu$^{+}$\\
$^+$Department of Electrical and Computer Engineering, The Ohio State University\\$^*$Department of Statistics, Iowa State University\\$^\circ$IBM Research AI, Thomas J. Watson Research Center
}


\maketitle
\thispagestyle{empty}

\begin{abstract}
Recently, min-max optimization problems have received increasing attention due to their wide range of applications in machine learning (ML).
However, most existing min-max solution techniques are either single-machine or distributed algorithms coordinated by a central server.
In this paper, we focus on the {\em decentralized} min-max optimization for learning with domain constraints, where multiple agents collectively solve a nonconvex-strongly-concave min-max saddle point problem without coordination from any server.
Decentralized min-max optimization problems with domain constraints underpins many important ML applications, including multi-agent ML fairness assurance, and policy evaluations in multi-agent reinforcement learning.
We propose an algorithm called \alg (\ul{pr}oximal gradi\ul{e}nt-tra\ul{c}king and stochast\ul{i}c recur\ul{s}ive var\ul{i}ance reducti\ul{on}) that enjoys a convergence rate of \(\mathcal{O}(1/T)\), where \(T\) is the maximum number of iterations.
To further reduce sample complexity, we propose \algplus with an adaptive batch size technique.
We show that the fast \(\mathcal{O}(1/T)\) convergence of \alg and \algplus to an \(\epsilon\)-stationary point imply \(\mathcal{O}(\epsilon^{-2})\) communication complexity and \(\mathcal{O}(m\sqrt{n}\epsilon^{-2})\) sample complexity, where \(m\) is the number of agents and $n$ is the size of dataset at each agent.
To our knowledge, this is the first work that achieves \(\mathcal{O}(\epsilon^{-2})\) in both sample and communication complexities in decentralized min-max learning with domain constraints. 
Our experiments also corroborate the theoretical results.
\end{abstract}

\maketitle


\section{Introduction}\label{Section: introduction}

In recent years, machine learning (ML) has achieved a great success in many areas, including robotics\cite{siau2018building}, image recognition\cite{ozyurt2020efficient}, natural language processing\cite{nozaki2018predictive}, recommender systems\cite{deldjoo2020adversarial}, to name just a few.
Traditionally, the training of ML models is deployed in high-performance computer clusters co-located at large-scale data centers with easy access to big training datasets.
However, with more diverse ML applications emerging, the deployment of ML  has also been migrating to the edge of computing and communication networks due to the following reasons: 
First, in many ML applications, data are generated and collected through diverse data sources that are geographically disperse (e.g., smart mobile devices, vehicles, environmental sensors, satellite imagery).
Second, because of the limited communication capabilities of the devices and data privacy concerns, it is expensive or even infeasible to send the data collected at the edge networks to the cloud for centralized processing. 
These real-world limitations have spawned the rapid development of {\em decentralized learning} over edge networks in recent years, which can leverage highly flexible peer-to-peer edge computing networks with arbitrary topologies \cite{nedic2009distributed,lian2017can}.
Also, thanks to the resilience to single-point-of-failure, data privacy, and simple implementations, decentralized learning has attracted growing interest recently, and has found various science and engineering applications, 
such as distributedrobotics control \cite{ren2007information,zhou2011multirobot} and network resource allocation \cite{jiang2018consensus,rhee2012effect},
such as dictionary learning \cite{chen2014dictionary}, multi-agent systems \cite{cao2012overview,zhou2011multirobot}, multi-task learning \cite{wang2018distributed,zhang2019distributed}, and information retrieval \cite{ali2004tivo}.

From a mathematical perspective, conducting decentralized learning over a computing network amounts to solving an optimization problem {\em distributively} and {\em collaboratively} by a group of agents in the network.
However, among the existing literature of decentralized learning, most works are focused on the standard loss minimization formulation, i.e., $\min_{\x \in \mathbb{R}^{d}}f(\x)$, where $f(\cdot)$ denotes the loss objective function of learning and $\x$ denotes the global model parameters to be learned, and $d$ is the model dimension.
While this standard loss minimization formulation is sufficiently general to cover a wide range of ML applications (e.g., robotic network  \cite{smart2002effective,kober2013reinforcement,polydoros2017survey}), sensor network  \cite{cortes2004coverage,ogren2004cooperative,rabbat2004distributed}), power network  \cite{callaway2010achieving,dall2013distributed,ernst2004power,glavic2017reinforcement}), it has become increasingly apparent that its mathematical structure is not rich enough to capture new requirements of ever-emerging ML applications.
Notably, many sophisticated ML problems nowadays necessitates the so-called ``min-max'' optimization in the form of $\min_{\x \in \mathcal{X}} \max_{\y \in \mathcal{Y}} f(\x,\y)$, where $\x$ and $\y$ are both parameters to be learned (may have different dimensionality), and $\mathcal{X}$ and $\mathcal{Y}$ are some conforming real subspaces for $\x$ and $\y$, respectively.
Although min-max optimization also has a long history that dates back to 1945 \cite{wald1945statistical}, research on {\em decentralized min-max optimization} remains in its infancy so far and results in this area are surprisingly limited.

In this paper, rather than studying the unstructured general decentralized min-max problems  as in \cite{liu2019decentralized2,liu2020decentralized}, we focus on a subclass of interesting decentralized min-max optimization, where multiple agents collectively solve a {\em domain-constrained} nonconvex-strongly-concave (NCX-SCV) min-max problem.
The decentralized constrained NCX-SCV min-max problem is important because it arises naturally from many recently emerging multi-agent ML applications, such as multi-agent fairness constraints in adversarial training \cite{xu2021robust}, policy evaluation in multi-agent reinforcement learning (MARL) \cite{qiu2020single}, and multi-agent fairness assurance in ML \cite{baharlouei2019r, sattigeri2018fairness} (see Section~\ref{Section: Preliminary} for more in-depth discussions).

However, designing effective and efficient algorithms for solving decentralized constrained NCX-SCV min-max problems is highly non-trivial due to the following technical challenges:
First, min-max optimization tackles a composition of an inner maximization problem
and an outer minimization problem.
This tightly coupled inner-outer mathematical structure, together with the decentralized nature and the non-convexity of the outer problem, render the design and theoretical analysis of the algorithms rather difficult.
Moreover, the constrained structures in both the inner and outer problems impose yet another layer of challenges in the algorithmic design for decentralized constrained NCX-SCV min-max problems. 
Second, the decentralization over edge computing networks faces two fundamentally {\em conflicting} performance metrics.
On one hand, due to the high dimensionality of deep learning models and large datasets, it is infeasible to exploit information beyond first-order stochastic gradients to determine search directions in algorithm design.
Although the variance of stochastic gradients can be reduced by increasing the number of training samples in mini-batches, doing so incurs higher computational costs for the stochastic gradients.
On the other hand, if one uses fewer training samples in each iteration to trade for a lower computational cost, the larger variance in the stochastic gradients inevitably leads to more  communication rounds to reach a certain training accuracy (i.e., slower convergence).
The high communication complexity is particularly problematic in wireless edge networks, where communication connections could be low-speed and highly unreliable.
Third, constrained decentralized min-max optimization presents a significantly greater challenge than its unconstrained counterpart. This is primarily due to the non-smooth nature of the domain constraints and the intricate coupling between these constraints and the min-max problem structure.

The major contribution of this paper is that we propose a series of new algorithmic techniques to address the challenges above and achieve low sample and communication complexities in decentralized constrained NCX-SCV min-max problems.
Our main technical results and their significance are summarized as follows:
\begin{list}{\labelitemi}{\leftmargin=1em \itemindent=-0.09em \itemsep=.2em}
\item 
We propose a decentralized constrained min-max optimization algorithm called \alg (\ul{pr}oximal gradi\ul{e}nt-tra\ul{c}king and stochast\ul{i}c recur\ul{s}ive var\ul{i}ance reducti\ul{on}) and show that, to achieve an $\epsilon$-stationary point, \alg enjoys a convergence rate of $\mathcal{O}(1/T)$ ($T$ is the maximum number of iterations).
This result further implies an $[\mathcal{O}(m\sqrt{n}\epsilon^{-2}), \mathcal{O}(\epsilon^{-2})]$ sample-communication complexity scalings, where $m$ is the number of agents, and $n$ is the size of the local dataset at each agent. 

\item To relax the full gradient evaluation requirement in \algns, we propose an enhanced algorithm called \algplus, which is based on an adaptive batch size technique.
\algplus further reduces the sample complexity of \algns, while retaining the same $[\mathcal{O}(m\sqrt{n}\epsilon^{-2})$, $\mathcal{O}(\epsilon^{-2})]$ sample-communication complexity scaling laws as those of \algns.
Moreover, a lower sample complexity can be obtained in \algplusns by slightly trading off its communication complexity (the trade-off is only reflected in the hidden Big-O constants).

\item We note that both \alg and \algplus algorithms integrate two proximal operators for both the inner and outer constraints (on $\x$ and $\y$), variance reduction techniques for both inner and outer updates, and gradient-tracking-based updates in both inner and outer variables.
In this sense, both \algns-based algorithms can be viewed as a triple hybrid approach, which necessitates new performance analysis and proof techniques.
It is also worth pointing out that the proposed algorithmic and proof techniques in \alg could be of independent interest in decentralized min-max learning theory in general.



\end{list}

The rest of the paper is organized as follows.
In Section~\ref{Section: Preliminary}, we first provide the preliminaries of the decentralized min-max optimization problems and discuss related works.
In Section~\ref{Sec: GT-SRVR}, we propose two stochastic variance reduced algorithms, namely \alg and \algplusns. The convergence rate, communication complexity, and sample complexity of \alg and \algplus are also provided in Section~\ref{Sec: GT-SRVR}.
%
%
Section~\ref{Section: experiment} provides numerical results to verify our theoretical findings, and Section \ref{Section: conclusion} concludes this paper.

\section{Preliminaries and related work}\label{Section: Preliminary}
To facilitate subsequent technical discussions, in Section~\ref{Sec: formulation}, we first provide the basics of decentralized min-max optimization and its consensus formulation.
Then, we formally define the notions of sample and communication complexities of the consensus form of decentralized min-max optimization problems.
Next, in Section~\ref{Sec: Related Work on Policy Evaluation}, we provide an overview of related work of existing optimization algorithms for solving min-max learning problems and their performance in terms of their sample and communication complexities, thus putting our work in comparative perspectives.
%

\subsection{Preliminaries of Decentralized Min-Max Optimization} \label{Sec: formulation}
{\bf 1) Network Consensus Formulation:}
Consider an undirected connected network $\mathcal{G} = (\mathcal{N},\mathcal{L})$, where
$\mathcal{N}$ and $\mathcal{L}$ are the sets of nodes (agents) and edges, respectively, with
$|\mathcal{N}| =m$.
Each agent has local computation capability and is able to communicate with the set of its neighboring agents defined as $\Nc_i \triangleq \{i' \in \Nc,: (i,i')\in \Lc\}$.
For presentation simplicity, we assume that each agent $i$ has $n$ data samples and thus there are $mn$ data samples in total\footnote{
We note that with more complex notation, all our proofs and results continue to hold in cases with unequal sized local datasets.
}.
In decentralized min-max optimization, the agents in the network distributively and collaboratively solve the following decentralized min-max optimization problem:

\begin{align} \label{Eq: general_problem}
&\min_{{\x} \in \mathcal{X}}\max_{{\y} \in \mathcal{Y}} \Big[  \frac{1}{m}\sum_{i=1}^{m}  F_i({\x},{\y}) +h(\x) \Big],
\end{align}
where $\x \in \mathcal{X}$ and $\y \in \mathcal{Y}$ are parameters to be trained for the outer-min and inner-max problems, respectively,
the sets $\mathcal{X} \subseteq \mathbb{R}^{p_1}$ and $\mathcal{Y} \subseteq \mathbb{R}^{p_2}$ are closed and convex sets,
$F_i({\x},{\y}) \!\triangleq\! \frac{1}{n} \sum_{j=1}^{n} f_{ij}({\x_i},{\y_i}| \bm{\xi}_{ij}) $ denotes the local objective function, and $h(\x_i)$ is a proper convex function (possibly non-differentiable) that usually plays the role of regularization.
Here, $F_i({\x},{\y})$ is only observable to node $i$ and is assumed to be non-convex with respect to $\x$ for a fixed $\y$, and strongly concave with respect to $\y$ for a fixed $\x$.
To solve Problem~\eqref{Eq: general_problem} in a decentralized fashion, a common approach is to rewrite it in the following equivalent form:
\begin{align} \label{Eq: consensus_problem}
 \min_{\!\left\{{\x_i}\!\in \mathcal{X},\forall i\right\}}  \max_{\!\left\{{\y_i}\!\in \mathcal{Y},\forall i\right\}} & \left[ \frac{1}{mn}\sum_{i=1}^{m}\sum_{j=1}^{n} f_{ij}({\x_i},{\y_i}| \bm{\xi}_{ij}) +h(\x_i)\right],
\nonumber\\
 \hspace{.05in}  \text{subject to}   &  \quad {\x_i} = \x_{i'}, {\y_i} = \y_{i'}, ~~  \forall (i,i') \in \mathcal{L}, \!\!\! 
\end{align}
where ${\x_i}$ and ${\y_i}$ are the local copies of the original parameters $\x$ and $\y$ at agent $i$, respectively.
%
%
The equality constraints in \eqref{Eq: consensus_problem} ensure that the local copies at all agents are equal to each other, hence the name ``consensus form.''
Clearly, Problems~\eqref{Eq: general_problem} and \eqref{Eq: consensus_problem} share the same solution.
%
%
In the rest of this paper, we will focus on solving Problem~\eqref{Eq: consensus_problem}, which will be referred to as a decentralized non-convex-strongly-concave (NCX-SCV) consensus min-max optimization problem.
The goal of decentralized consensus min-max optimization is to design an algorithm to attain a collective $\epsilon$-stationary point $\{\x_i,\y_i, \forall i\}$ that satisfies the following condition:

\begin{align*}  
&\underbrace{\frac{1}{m} \sum_{i=1}^{m}\left\|\mathbf{x}_{i}-\overline{\mathbf{x}}\right\|^{2}}_{\substack{\mathrm{Outer \,\, consensus} \\ \mathrm{error}}} + \underbrace{\frac{1}{m} \sum_{i=1}^{m}\left\|\mathbf{y}_{i}-\overline{\mathbf{y}}\right\|^{2}}_{\substack{\mathrm{Inner \,\, consensus} \\ \mathrm{error} } }  \\&+ \underbrace{\Eb\|{\y}^* \!-\! {\bar{{\y}}}\|^2 }_{\substack{\mathrm{Saddle \,\, point} \\ \mathrm{error} }}+  \underbrace{\|\frac{1}{m}\sum_{i=1}^m \nabla_{\x} F_i(\x,\y ) \|^2 }_{\mathrm{Global \,\,gradient \,\,magnitude} } \le \epsilon^2,
\end{align*}
where $\bar{\x} \triangleq \frac{1}{m} \sum_{i=1}^{m} \x_{i}$, $\bar{\y} \triangleq \frac{1}{m} \sum_{i=1}^{m} \y_{i}$, and ${\y}^*$ represents the maximizer point of $F$ over $\y$, where ${\y}^*({\bar{{\x}}}) \in \arg\max_{{\y}\in\mathcal{Y}} F({\bar{{\x}}} ,\y)$,


As mentioned in Section~\ref{Section: introduction}, two of the most important performance metrics in decentralized optimization are the sample and communication complexities.
In this paper, we adopt two definitions of sample and communication complexities that are widely used in the decentralized optimization literature (e.g., \cite{sun2020improving}) to measure the efficiency of our algorithms:
\begin{defn}[Sample Complexity]
	The sample complexity is defined as the total number of incremental first-order oracle (IFO) calls required across all nodes until an algorithm converges to an $\epsilon$-stationary point, where one IFO call evaluates a pair of gradients $(\nabla_{\x} f_{ij}(\x,\y), \nabla_{\y} f_{ij}(\x,\y))$ at node $i$.
\end{defn}

\begin{defn}[Communication Complexity]
Let a round of communications be a time window during which each node sends a vector to its neighboring nodes while receiving a set of vectors from all its neighboring nodes.
Then, the communication complexity is defined as the total number of rounds of communications required until an algorithm converges to an $\epsilon$-stationary point.
\end{defn}

{\bf 2) Motivating Application Examples:}
With the basics of decentralized constrained NCX-SCV min-max optimization, we provide two examples to further motivate its practical relevance:
\begin{list}{\labelitemi}{\leftmargin=1em \itemindent=-0.09em \itemsep=.2em}
	%
%
\item {\em Multi-Agent Fair ML:}
Consider a machine learning task with dataset $\{b_{ij},[\tilde{\bm{\xi}}_{ij}^{\top},\bm{\xi}_{ij}^{*\top}]^{\top} \}$ over a multi-agent network, where $b_{ij}$ is the observed label of the $j$-th sample at the $i$-th  agent, $\tilde{\bm{\xi}}_{ij}\in\mathbb{R}^{d_1}$ denotes the corresponding nonsensitive features and $\bm{\xi}_{ij} \in\mathbb{R}^{d_2}$ represents the sensitive features.
In the problem of Fair ML, fairness is imposed by adding a regularization term that penalizes the statistical correlation between the learning model output $\hat{b}_{ij}$ and the sensitive attributes $\bm{\xi}_{ij}^{*}$.
In binary case, one example is the Renyi correlation \cite{baharlouei2019r} as a regularization to impose fairness, under which the multi-agent fair ML problem can be written as a decentralized NCX-SCV min-max problem \cite{baharlouei2019r}:
$
\min _{\x \in \mathcal{X} } \max _{{\y}\in \mathcal{Y}} \mathbb{E}_i\big[ \mathbb{L}(F_i({\x},  \y| {\bm{\xi}_i}), b_i)-\lambda_l \sum_{j=1}^{c} y_{ij}^{2} {\mathbf{f}}_{ij}({\x_i}, {\bm{\xi}_i})+\lambda_l $

$\cdot \sum_{j=1}^{c} y_{ij} \tilde{S}  {\mathbf{f}}_{ij}({\x}, {\bm{\xi}_i})\big],
$
 where $\tilde{S}=2 S-1$, $S=\{0,1\}$, denotes the sensitive attribute, $\mathbb{L}$ is the loss function, $\lambda_l$ is a positive scalar balancing fairness and goodness-of-fit, $c$ is the class label and $ {\mathbf{f}}_{ij}({\x}, {\bm{\xi}_i})$ represents the vector-valued output of a neural network after soft-max layer.
 \item {\em Data Poisoning Attack:}
Consider a decentralized learning problem with $m$ agents trying to learn a common model. An adversary has the ability to inject noise into the training samples of a subset of agents.
Let $\y_i$ denote the model parameter and let $\x_i$ denote the injected poisoned data parameter. 
In this problem, the adversary tries to maximize the loss function while the other agents aim at minimizing the loss function.
Thus, the data poisoning attack problem has the following NCX-SCV min-max problem:
 $\max _{ \x \in \mathcal{X} } \min _{\y \in \mathcal{Y}  } \sum_{i=1}^{m} \frac{1}{|{\bm{\xi}_i} |} \sum_{\ell \in {\bm{\xi}_i}} \log \big(1+\exp \big(\big(-v_{\ell} \y_{i}^{T}\big(w_{\ell}+\x_{i}\big)\big)\big)$,
where $v_{\ell} \in \mathbb{R}$ and $w_{\ell} \in \mathbb{R}^{d}$ denote the $\ell$-th data point's label and the feature vector, respectively.
%


\end{list}

\subsection{Related Work} \label{Sec: Related Work on Policy Evaluation}

{\bf 1) Centralized NCX-SCV Min-Max Optimization:}
In the literature, the state-of-the-art algorithms for solving NCX-SCV optimization problems in the centralized setting are GDA~\cite{lin2020gradient}, min-max-PPA~\cite{lin2020near}, and  SREDA~\cite{luo2020stochastic}.
Specifically, \citet{lin2020gradient} proposed a gradient-based GDA method to find a first-order Nash equilibrium point.
In each iteration, GDA performs gradient descent over the $\x$-variable and gradient ascent over the $\y$-variable.
GDA has an $\mathcal{O}(1/T)$ convergence rate for NCX-SCV min-max optimization problems, where $T$ is the maximum number of iterations.
Also, it requires a full gradient evaluation in each iteration, which implies an $\mathcal{O}(n\epsilon^{-2})$ sample complexity to achieve an $\epsilon$ convergence error.
%
%
The Minimax-PPA method is proposed in \cite{lin2020near} to solve NCX-NCV problem and achieves an $\tilde{\mathcal{O}}\left( n \varepsilon^{-2}\right)$ sample complexity.
These methods have a high sample complexity in the big-data regime with a large $n$.
To overcome this issue, several variance reduction methods have also been proposed.
%
For example, in \cite{luo2020stochastic}, a variance reduction algorithm named SREDA is proposed, which is further enhanced by \cite{xu2020enhanced} to allow a larger step-size.
SREDA achieves an $\tilde{\mathcal{O}}\left(n +\sqrt{n} \epsilon^{-2}\right)$ sample complexity for large $n$, thus having a lower sample complexity than GDA and minimax-PPA.
However,  SREDA can only handle min-max problems with constraints on $\x$ but not on $\y$. 
We summarize the above comparisons in Table~\ref{tab:comp1}.
While the above algorithms achieve varying degrees of success in solving NCX-SCV min-max problems, they are developed for the centralized setting, which is fundamentally different from our work.

\begin{table*}[htbp]
	\caption{Comparisons among algorithms for NCX-SCV min-max problems ($m$ is the number of agents, $n$ is the size of dataset for each agent, and $\epsilon$ is the convergence error.
	Our proposed algorithms are marked in bold.}
	\label{tab:comp1}
	\begin{center}
	{\small
		\begin{tabular}{c c c c c}
			\toprule
			\multirow{2}{*}{Algorithm$^*$}   &Proximal  & Sample  & Commun.&Decen-\\
			&Operator  & Complex. &  Complex. &tralized\\
			\midrule
			GDA \cite{lin2020gradient} &  $\y$  & \!\!{$\tilde{\mathcal{O}}\left( n\varepsilon^{-2}\right)$} & {-} & \xmark\\
			\midrule
			Minmax-PPA \cite{lin2020near} &$  	 	    \x$ and $	  \y$  &\!\! {$\tilde{\mathcal{O}}\left( n \varepsilon^{-2}\right)$} & {-} & \xmark\\
			\midrule	{ SREDA \cite{luo2020stochastic} } & $\x$ &
		\!\!\!\!\!\!\!	$	\tilde{\mathcal{O}}  \left( n  +   \sqrt{n} \varepsilon^{-2} \right) $& -& \xmark \\
			\midrule
			$\mathsf{\mathbf{PRECISION}}$ & \multirow{2}{*}{$\x $ and $ \y$ }& \multirow{2}{*}{$ \!\!\mathcal{O}(m\sqrt{n}\epsilon^{-2})$} & \multirow{2}{*}{$\mathcal{O}(\epsilon^{-2}) $}& \multirow{2}{*}{\cmark}\\
			{$\mathsf{\mathbf{PRECISION}^+}$}& &  \\
			\bottomrule
		\end{tabular}
	}
	\end{center}
\end{table*}

\begin{table*}[t!]
	\caption{Comparisons among algorithms for decentralized min-max problems.
	}
	\label{tab:comp2}
	\begin{center}
	{\small
		\begin{tabular}{c c c  c c}
			\toprule
			\multirow{2}{*}{Algorithm$^*$}   &Proximal  & Sample  & Commun.&	\multirow{2}{*}{Problem}   \\
			&Operator  & Complex. &  Complex.\\
			\midrule
			{DPOSG \cite{liu2020decentralized} } & - & {$\mathcal{O}(\epsilon^{-12})$} & {$\!\!\!\!\!\!\!\!\!\mathcal{O}(\log (1/\epsilon))$} &\!\!NCX-NCV\\
		\midrule
		{  CSPSG \cite{mateos2015distributed} } & $\x $ and $ \y$  & $\mathcal{O}(\epsilon^{-4}) $ & $\mathcal{O}(\epsilon^{-4}) $ & {CX-CV} \\
		\midrule
		{DPPSP \cite{liu2019decentralized2} } & $\x $ and $ \y$  & $\mathcal{O}(\epsilon^{-4}) $ & $\mathcal{O}(\epsilon^{-4}) $ & \!NCX-NCV \\
			\midrule
			{GT-GDA \cite{tsaknakis2020decentralized}  } & $\x$ or $\y$ & {$\mathcal{O}( {mn}\epsilon^{- 2})$} & {$\mathcal{O}(\epsilon^{-2})$}& \!\!NCX-SCV  \\
			\midrule
			$\mathsf{\mathbf{PRECISION}}$  & \multirow{2}{*}{$\x $ and $ \y$ }& \multirow{2}{*}{$ \mathcal{O}( m \sqrt{n} \epsilon^{ - 2})$} & \multirow{2}{*}{$\mathcal{O}( \epsilon^{-2}) $} &\multirow{2}{*}{\!\!NCX-SCV} \\
			{$\mathsf{\mathbf{PRECISION}^+}$}& &  \\
			\bottomrule
		\end{tabular}
		}
	\end{center}
\end{table*}

\smallskip
{\bf 2) Decentralized Min-Max Optimization:}
%
%
%
%
As mentioned in Section~\ref{Section: introduction}, existing results on decentralized min-max optimization are quite limited.
%
The earliest attempt is the CSPSG method \cite{mateos2015distributed}, which considered the most ideal convex-concave (CX-CV) setting.
Due to its simplistic SGD-type updates, CSPSG has high sample and communication complexities of $\mathcal{O}(\epsilon^{-4})$.
DPOSG \cite{liu2020decentralized} considered unstructured nonconvex-nonconave (NCX-NCV) unconstrained decentralized min-max problems in the context of large-scale GANs, and proposed to leverage the classical DSGD \cite{nedic2009distributed} approach to decentralize the centralized counterpart algorithm called OGDA \cite{mokhtari2020unified}.
Due to the limitations inherent in DSGD, DPOSG suffers from a high sample complexity of $\mathcal{O}(\epsilon^{-12})$.
In contrast, DPPSP \cite{liu2019decentralized2} also studied unstructured NCX-NCV decentralized min-max optimization problems with constraints.
Due to the use of basic proximal SGD-type updates, DPPSP also suffers high sample and communication complexities of $\mathcal{O}(\epsilon^{-4})$.

Compared to the simplistic algorithmic techniques in \cite{liu2019decentralized2,liu2020decentralized}, our  \alg algorithms is a triple hybrid algorithm that integrates proximal operators, variance reductions, and gradient tracking, thus achieving much lower sample and communication complexities.
We note that although our significantly lower sample and communication complexities are achieved under the more structured NCX-SCV setting, we believe our techniques can also be applied to NCX-NCV to improve the sample and communication complexities of existing works.
This will be left in our future work.

%
%
%
The most related work to ours is GT-GDA \cite{tsaknakis2020decentralized}, which also studied constrained decentralized NCX-SCV min-max optimization.
The key difference between GT-GDA and our work is that only one constraint set is imposed on either $\x$ or $\y$, but not on both.
In contrast, we consider the more complex case where both $\x$ and $\y$ are constrained.
GT-GDA also requires several inner updates for $\y$ and then performs one update for $\x$, which is similar to alternating direction method of multipliers \cite{boyd2011distributed} (ADMM) update scheme.
Also, our algorithms achieve a lower sample complexity $\mathcal{O}(m\sqrt{n}\epsilon^{-2})$ than that of $\mathcal{O}(mn\epsilon^{-2})$ in GT-GDA. 
To conclude this section, we summarize the above comparisons in Table~\ref{tab:comp2}.
Another closely related work can be found in \cite{zhang2021taming}, where the authors developed a decentralized optimization method for a multi-agent reinforcement learning policy evaluation problem based on the mean squared projected Bellman error (MSPBE), which can be formulated as a finite-sum minimax problem.
However, our work differs from \cite{zhang2021taming} in the following aspects:
	(i) Unlike \cite{zhang2021taming}, our method can handle  {\em non-smooth} objectives.
	However, the direct proximal extension of the algorithm in \cite{zhang2021taming} may diverge in solving the decentralized problem~\cite{hong2022divergence}. To this end, we propose a specialized proximal operator $\tilde{\mathbf{x}}_{i}\left(\mathbf{x}_{i, t}\right)$ to address this challenge, see detailed discussions in our Remark \ref{rem22};
	(ii) Our approach addresses  {\em general} decentralized min-max optimization problems, while \cite{zhang2021taming} is limited to RL policy evaluation.

%
%
%

\section{Solution Approach}\label{Sec: GT-SRVR}


In this section, we first present our \alg and \algplus algorithms in Sections~\ref{sec:algs} and \ref{sec:alg+}, respectively.
Then, we provide the main theoretical results and the key insights of the \alg and \algplus algorithms in Section~\ref{sec:thms}.
Due to space limitation and for better readability, we relegate some proof details of the theoretical results to our Appendix. 

\subsection{The \alg Algorithm} \label{sec:algs}

To solve the consensus form of decentralized min-max problem in Problem~\eqref{Eq: consensus_problem}, we adopt the network consensus mixing approach in the literature \cite{nedic2009distributed}. 
Toward this end, we let $\M \in \Rb^{m\times m}$ denote the consensus weight matrix and let $[\M]_{ii'}$ denote the element in the $i$-th row and the $i'$-th column in $\M$.
$\M$ satisfies the following properties \cite{nedic2009distributed,wai2018multi}:

\begin{enumerate}[topsep=1pt, itemsep=-.1ex, leftmargin=.25in]
\item[(a)] {\em Doubly stochastic:} $\sum_{i=1}^{m} [\mathbf{M}]_{ii'}=\sum_{i'=1}^{m} [\mathbf{M}]_{ii'}=1$;

\item[(b)] {\em Symmetric:} $[\mathbf{M}]_{ii'} = [\M]_{i'i}$, $\forall i,i' \in \mathcal{N}$;

\item[(c)] {\em Network-Defined Sparsity:} $[\M]_{ii'} > 0$ if $(i,i')\in \mathcal{L};$ otherwise $[\mathbf{M}]_{ii'}=0$, $\forall i,i' \in \mathcal{N}$.
\end{enumerate}
Note that the above properties imply that the eigenvalues of $\M$ are real and can be sorted as $-1 < \lambda_m(\M) \leq \cdots \leq \lambda_2(\M) < \lambda_1(\M) = 1$.
For notational convenience, we define the second-largest eigenvalue in magnitude of $\M$ as $\lambda \triangleq \max\{|\lambda_2(\M)|, .., |\lambda_m(\M)|\}$, which will play an important role in the step-size selection and analysis of the algorithm's convergence rate.
With the above notation, we are now in a position to describe our proposed algorithms.

As mentioned in Section~\ref{Section: introduction}, our \alg algorithm can be viewed as a triple hybrid of proximal, gradient tracking, and variance reduction techniques.
Next, we will see that these techniques can be organized into three key algorithmic steps:

\begin{list}{\labelitemi}{\leftmargin=1em \itemindent=-0.09em \itemsep=.2em}
	\item {\em Step 1 (Local Proximal Operations):} In each iteration $t$, each agent $i$ first performs the following proximal operations to cope with the constraint sets $\mathcal{X}$ and $\mathcal{Y}$ for the outer and inner variables, respectively:
\begin{align}
\label{Eq: DA_updating_x} \tilde{{\x}}_i({\x}_{i,t}) =&{\arg\min}_{{\x}_i \in \mathcal{X}} \langle\p_{i,t}, {\x}_{i}- {\x}_{i,t}\rangle\notag\\& +\frac{\tau}{2} \| {\x}_i - {\x}_{i,t}\|^2 + h({\x}_i), \\
\label{Eq: DA_updating_y} \tilde{{\y}}_{i}({{\y}}_{i,t}) \! =& {\arg\min}_{{\y}_i \in \mathcal{Y}}  \big\| {\y}_i- \big({\y}_{i,t} + \alpha \d_{i,t}\big)\big\|^2,
\end{align}
where $\p_{i,t}$ and $\d_{i,t}$ are two auxiliary vectors for gradient tracking purposes and will be defined shortly, $\tau > 0$ is a constant proximal control parameter,
%
%
and $\alpha > 0$ is a constant parameter to control the magnitude of the updates of ${\y}$. 
%
%
	\item {\em Step~2 (Consensus Update):} Next, each agent $i$ updates the outer and inner model parameters ${\x}_i,{\y}_i$:
\begin{align}\label{Eq: consenus_updating_x} 
{\x}_{i,t+1} \! =& \! \underbrace{ \sum_{i' \in \mathcal{N}_{i}} [\M]_{ii'} {\x}_{i',t} }_{\mathrm{(a)}}+ \underbrace{ \nu \left( \tilde{{\x}}_i({\x}_{i,t}) -{\x}_{i,t}  \right) }_{\mathrm{(b)}}, 
\end{align}

\begin{align}\label{Eq: consenus_updating_y}
 {\y}_{i,t+1} =& \underbrace{ \sum_{i' \in \mathcal{N}_{i}} [\M]_{ii'} {\y}_{i',t} }_{\mathrm{(a)}}+ \underbrace{ \eta (\tilde{{\y}}_{i} ({{\y}}_{i,t})- {\y}_{i,t}) }_{\mathrm{(b)}},
\end{align}
where $\nu$ and $\eta$ are the step-sizes for updating ${\x}$- and ${\y}$-variables, respectively. 
Note that in \eqref{Eq: consenus_updating_x} and \eqref{Eq: consenus_updating_y}, component $(a)$ is a local weighted average at agent $i$, which is also referred to as ``consensus step,'' 
and component $(b)$ performs a local update in the spirit of Frank-Wolfe given the proximal points $\tilde{\x}$ and $\tilde{\y}$, which is different from the conventional decentralized stochastic gradient updates \cite{nedic2018network}.

	\item {\em Step 3 (Local Gradient Estimate):} In the next step, each agent $i$ estimates its local gradients using the following gradient estimators: 
		\begin{subequations}\label{Eq: Updating_uv}
			\begin{eqnarray}
			\v_{i,t} =
			\begin{cases}
			\nabla_{{{\x}}}  F_i({{\x}}_{i,t},{{\y}}_{i,t}), \qquad \text{if~~} \text{mod}(t,q) = 0, \\
			\v_{i,t\!-\!1} \!+\!\frac{1}{|\Sc_{i,t}|}\!\sum_{j \in \Sc_{i,t}}\!\! \big( \nabla_{{{\x}}} f_{ij}({{\x}}_{i,t},{{\y}}_{i,t})
			\!\\\quad \quad\quad-\! \nabla_{{{\x}}}  f_{ij}({{\x}}_{i,t\!-\!1},{{\y}}_{i,t-1})\big), \quad  \text{o.w.}
			\end{cases}\label{Eq: Updateing_v} \!\!\!\!\!\!  \\
			\u_{i,t} =
			\begin{cases}
			\nabla_{{{\y}}}  F_i({{\x}}_{i,t},{{\y}}_{i,t}), \qquad \text{if~~} \text{mod}(t,q) = 0, \\
			\u_{i,t\!-\!1} \!+\!\frac{1}{|\Sc_{i,t}|}\!\sum_{j \in \Sc_{i,t}}\!\! \big( \nabla_{{{\y}}} f_{ij}({{\x}}_{i,t},{{\y}}_{i,t})
			\!\\ \quad \quad\quad-\! \nabla_{{{\y}}}  f_{ij}({{\x}}_{i,t\!-\!1},{{\y}}_{i,t-1})\big), \quad  \text{o.w.} 
			\end{cases}\label{Eq: Updating_u} \!\!\!\!\!\! 
			\end{eqnarray}
		\end{subequations}
Here, $\Sc_{i,t}$ is the sample mini-batch in the $t$-th iteration, and $q$ is a pre-set inner loop iteration number.
%

\item {\em Step 4 (Gradient Tracking): Each agent $i$ updates $\p_i$ and $\d_i$ by averaging over its neighboring tracked gradients:}
	\begin{align}\label{Eq: Tracking_updating}
	\begin{cases}
	\p_{i,t} \!=\! \sum_{i' \in \mathcal{N}_{i}} [\M]_{ii'} \p_{i',t-1}  + \v_{i,t}- \v_{i,t-1}, \\
	\d_{i,t} \!= \!\sum_{i' \in \mathcal{N}_{i}} [\M]_{ii'} \d_{i',t-1}  + \u_{i,t} -\u_{i,t-1}.
	\end{cases}
	\end{align}
\end{list}

Our \alg algorithm can be intuitively understood as follows:
In \algns, each agent conducts both descent and ascent steps, since Problem~\eqref{Eq: consensus_problem} minimizes over ${\x}$ and maximizes over ${\y}$.
Note that $\v_{i,t}$ and $\u_{i,t}$ in \eqref{Eq: Updating_uv} only contain the gradient information of the local objective function $F_{i}({{\x}},{{\y}})$.
Merely updating with directions $\v_{i,t}$ and $\u_{i,t}$ cannot guarantee the convergence of the global objective function $F({\x},{\y})$.
Therefore, we introduce two auxiliary variables $\p_{i,t}$ and $\d_{i,t}$ for global gradient tracking purposes.
As each agent $i$ updates these two variables by performing the local weighted aggregation shown in \eqref{Eq: Tracking_updating}, $\p_{i,t}$ and $\d_{i,t}$ track the directions of the global gradients.

It is insightful to compare \alg with our most related work, the GT-GDA method in \cite{tsaknakis2020decentralized}.
In GT-GDA, agent $i$ computes the local {\em full gradients} in the $t$-th iteration as follows:
\begin{equation}\label{Eq: grad_updating}
\v_{i,t} = \nabla_{{\x}} F_{i}({\x}_{i,t},{\y}_{i,t}), \quad \u_{i,t} =  \nabla_{{\y}}  F_{i}({\x}_{i,t},{\y}_{i,t}).
\end{equation}
Different from GT-GDA \cite{tsaknakis2020decentralized}, \alg estimates the local gradients in Eq.~\eqref{Eq: Updating_uv} at agent $i$.
In Eq.~\eqref{Eq: Updating_uv}, the algorithm evaluates a full gradient $\nabla  F_i({{\x}}_{i,t},{{\y}}_{i,t})$ only every $q$ steps.
For other iterations with $\text{mod}(t,q) \neq 0$, \alg uses local stochastic gradients estimated by a mini-batch $\frac{1}{|\Sc_{i,t}|}\sum_{j \in \Sc_{i,t}}\!\! \nabla_{{{\y}}} f_{ij}({{\x}}_{i,t},{{\y}}_{i,t})$ and a recursive correction term $\u_{i,t\!-\!1} \!-\!\frac{1}{|\Sc_{i,t}|}\!\sum_{j \in \Sc_{i,t}}\!\! \nabla_{{{\y}}}  f_{ij}({{\x}}_{i,t\!-\!1},{{\y}}_{i,t-1})$.
Thanks to the periodic full gradients and recursive correction terms, \alg is able to achieve a convergence rate of $\mathcal{O}(1/T)$.
Moreover, due to the stochastic subsampling of  $\Sc_{i,t}$, \alg has a {\em lower} sample complexity than GT-GDA \cite{tsaknakis2020decentralized}.
The full description of \alg is shown in Algorithm~\ref{Algorithm: Prox-GT-SRVR}.

\begin{algorithm}[t!]
	\caption{\algns/\algplus at Agent $i$.}\label{Algorithm: Prox-GT-SRVR}
	\begin{algorithmic} [1]
		\REQUIRE If { \alg:}$ |\Rc_{i,t}| \!=\! n$;\\
		If {\algplus:} $$ |\Rc_{i,t}| = \min \{c_{\gamma} \sigma^{2}\left(\gamma_{t}\right)^{-1}, c_{\epsilon} \sigma^{2} \epsilon^{-1}, n\}.$$
		\STATE Set prime-dual parameter pair $({{\x}}_{i,0},{{\y}}_{i,0}) = ({{\x}}^0,{{\y}}^0)$.
		\STATE Draw $\Rc_{i,0}$ samples without replacement and calculate local stochastic gradient estimators as $$\p_{i,0} \!=\! \v_{i,0} \!=\! \frac{1}{|\Rc_{i,0}|}\!\! \sum_{j \in \Rc_{i,0}}\!\!\! \nabla_{{{\x}}}  f_{ij}({{\x}}_{i,0},{{\y}}_{i,0});$$
		$$\d_{i,0} \!=\! \u_{i,0}\!=\! \frac{1}{|\Rc_{i,0}|}\!\! \sum_{j \in \Rc_{i,0}} \!\!\!\nabla_{{{\y}}}  f_{ij}({{\x}}_{i,0},{{\y}}_{i,0});$$
		\FOR{$t = 1, \cdots, T$}
		\STATE  Update local parameters $({{\x}}_{i,t+1},{{\y}}_{i,t+1})$ as in Eq.~\eqref{Eq: DA_updating_x}-\eqref{Eq: consenus_updating_y};
		\STATE Compute local estimators $(\v_{i,t+1},\u_{i,t+1})$ as in Eq.~\eqref{Eq: Updating_uv};
		\STATE Track global gradients $(\p_{i,t+1},\d_{i,t+1})$ as in Eq.~\eqref{Eq: Tracking_updating};
		\ENDFOR
	\end{algorithmic}
\end{algorithm}

\subsection{The \algplus Algorithm} \label{sec:alg+}

Note that in \algns, full gradients are required for every $q$ steps, which may still incur high computational costs in some situations.
Also, in the initialization phase of \alg (before the main loop), agents need to evaluate full gradients, which could be time-consuming.
To address these challenges, we enhance the \alg with an adaptive batch size technique, and this enhanced version is called \alg{$^+$}.
Specifically, we modify the gradient estimators in ~\eqref{Eq: Updateing_v} and \eqref{Eq: Updating_u} in iteration $t$ with $\text{mod}(t,q)=0$ as follows :
\begin{align}
&\v_{i,t} =
\frac{1}{|\Rc_{i,t}|}\!\! \sum_{j \in \Rc_{i,t}}\!\!\! \nabla_{{{\x}}}  f_{ij}({{\x}}_{i,t},{{\y}}_{i,t}), \quad\\&
\u_{i,t} =
\frac{1}{|\Rc_{i,t}|}\!\! \sum_{j \in \Rc_{i,t}}\!\!\! \nabla_{{{\y}}}  f_{ij}({{\x}}_{i,t},{{\y}}_{i,t}),
\end{align}
where $\Rc_{i,t}$ is a subsample set (sampling without replacement), whose size
is chosen as
\begin{align} \label{choose_batch}
|\Rc_{i,t}| = \min \{c_{\gamma} \sigma^{2}\left(\gamma_{t}\right)^{-1}, c_{\epsilon} \sigma^{2} \epsilon^{-1}, n\}.
\end{align}
Here, $c_{\gamma}$ and $c_{\epsilon}$ are problem-dependent constants  to be defined later, $\sigma^{2}$ is the variance bound of data heterogeneity across agents (also defined later), and $\gamma_{t+1} \triangleq \frac{1}{q} \sum_{i=\left(n_{t}-1\right) q}^{t}\left\| \tilde{{\x}}_{t}-1 \otimes \bar{{\x}}_{t} \right\|^{2}$, where $\otimes$ represents the Kronecker product operator.

The selection of $|\Rc_{i,t}|$ is motivated by the fact that the periodic full gradient evaluation only plays an important role in the later stage of the convergence process: in the later stage of the convergence process, we need more accurate update direction.  
Later, we will see that under some mild assumptions and parameter settings, \algplus has the same convergence rate as that of \algns.
The full description of the \algplus algorithm is also illustrated in Algorithm~\ref{Algorithm: Prox-GT-SRVR}.

\subsection{ Theoretical Results of the $\mathsf{\mathbf{PRECISION}}$ and $\mathsf{\mathbf{PRECISION}^+}$ Algorithms} \label{sec:thms}

Before presenting the theoretical results of our algorithms, we first state the following assumptions:
\begin{assum}[Global Objective]\label{Assump: obj}
	The functions $F({\x},{\y}) = \frac{1}{m}\sum_{i=1}^{m} [F_i({\x}_i,{\y}_i)]$ and $J({\x}) = \max_{{\y}\in \mathcal{Y}} F({\x},{\y})$ satisfy:
	\begin{enumerate}[topsep=1pt, itemsep=-.1ex, leftmargin=.25in]
		\item[(a)] (Boundness from Below): There exists a finite lower bound $Q^* =Q({\x}^*)  = \inf_{{\x}} (J({\x})+h(\x) ) > -\infty;$
		
		\item[(b)] (Strong Concavity in ${\y}$): Local objective function $F_i({\x},\cdot)$ is $\mu$-strongly concave for fixed ${\x}\in\Rb^{p_1}$, i.e., there exists a positive constant $\mu$ such that $\|\nabla_{{\y}} F_i({\x}, {\y}) \!-\! \nabla_{{\y}} F_i({\x}, {\y}^\prime)\|\!\ge\! \mu\|{\y}\!-\!{\y}^\prime\|, \forall~\bm {x},{\y},{\y}^\prime \!\in\! \mathbb{R}^{p_2}, i \!\in\! [m]$.
%
		
		\item[(c)] (Bounded Gradient at Maximum): The partial gradient at  every $(\x,\nabla_{{\x}} F({x}, {\y}^*({\x})))$ pair is bounded, i.e., $\|\nabla_{{\x}} F({\x}, {\y}^*({\x}))\|<\infty$, $\forall ~{\x}\in\Rb^{p_1}$.
	\end{enumerate}
\end{assum}

Assumptions~\ref{Assump: obj}(a) and \ref{Assump: obj}(b) are standard in the literature.
Assumption~\ref{Assump: obj}(c) guarantees that $\nabla J({\x}) = \nabla_{{\x}} F({\x}, {\y}^*({\x}))$. 

%

\begin{assum}[Lipschitz Smoothness of Local Objectives]\label{Assump: Individual Lipschitz}
	The function $f_{ij}({{\x}},\cdot)$ is $L_f$-Lipschitz smooth, i.e., there exists a constant $L_f>0$, such that $\nabla f_{ij}({{\x}}, {{\y}}) \!=\! [\nabla_{{{\x}}} f_{ij}({{\x}}, {{\y}}]^\top, \nabla_{{{\y}}} f_{ij}({{\x}}, {{\y}})^\top)^\top$ satisfies $\|\nabla f_{ij}({{\x}}, {{\y}})\!-\! \nabla f_{ij}({{\x}}^\prime, {{\y}}^\prime)\|^2\! \le \!L_f^2\|{{\x}}\!-\!{{\x}}^\prime\|^2 \!+\! L_f^2\|{{\y}}\!-\!{{\y}}^\prime\|^2$, $\forall~ {{\x}},{{\x}}^\prime \in\mathcal{X},  {{\y}},{{\y}}^\prime \in \mathcal{Y}, i \in [m], j\in[n]$.
\end{assum}

Further, we have the following assumption only for the algorithm \algplusns:
\begin{assum}[Bounded Variance]\label{Assump: Bounded Variance}
	There exists a constant $\sigma^2>0$, such that $\Eb\|\nabla f_{ij}({{\x}}, {{\y}}) - \nabla F_{i}({{\x}}, {{\y}})\|^2\le\sigma^2 $, $\forall~ {{\x}}, {{\y}},\in \mathbb{R}^p, i \in [m], j\in[n]$.
\end{assum}
%

To address the challenges in characterizing the convergence rate for NCX-SCV decentralized constrained min-max problems, we propose the following {\em new} metric, which is the key to the success of establishing all convergence results in this paper:
\begin{align}\label{Eq: metric1}
\mathfrak{M}_t \triangleq& \Eb[
 \left\|\tilde{{\x}}_{t} - 1 \otimes \bar{{\x}}_t\right\|^{2} +\left\|{\x}_{t} -1 \otimes \bar{{\x}}_{t}\right\|^{2}\notag\\&+\left\|{\y}_{t} -1 \otimes \bar{{\y}}_{t}\right\|^{2} +\|{\y}_t^* -  {\bar{{\y}}}_t\|^2],
\end{align}
where  ${\y}_t^*$ denotes ${\y}^*({\bar{{\x}}}_t) = \arg\max_{{y}\in\Rb^p} F({\bar{{\x}}}_t,{\y})$.
The first two terms in~\eqref{Eq: metric1} are inspired by the metric in SONATA \cite{scutari2019distributed}, which measures the converging progress of non-convex decentralized {\em minimization} problems (not min-max). 
 %
 The third term in~\eqref{Eq: metric1} measures the consensus error of local copies on $\y$.
The fourth term in~\eqref{Eq: metric1} quantifies ${\bar{{\y}}}_t$'s convergence to the point ${\y}_t^*$ for $F({\bar{{\x}}}_t,\cdot)$.
Thus, as $\mathfrak{M}_t \rightarrow 0$, we have that the algorithm reaches a consensus on a first-order stationary point (FOSP) of the original decentralized constrained min-max optimization problem.


With the metric in~\eqref{Eq: metric1}, the convergence rates of algorithms \alg/\algplus can be characterized as follows:
\begin{thm}[Convergence of \algns]\label{Thm: Prox-GT-SRVR}
Under Assumption~\ref{Assump: obj} (a)-(d) and Assumption~\ref{Assump: Individual Lipschitz}, suppose that $\beta\leq \min \Big\{ \frac{\tau}{12}  ,\frac{1}{3} \Big\},$

$\alpha\leq \frac{1}{4L_f}, q = |\Sc_{i,t}| = \lceil  \sqrt{n}  \rceil $ hold and let $c_1= \frac{1-\lambda^2}{1+\lambda^2}$, if the step-sizes satisfy: 
	$
%
\eta\!\leq  \min \Big\{ \frac{1}{8}, \frac{c_1 m \mu}{ 375\alpha L_f^2},  \frac{15L_f^2}{\beta \mu \alpha^2c_1} ,\frac{3c_1^2 m}{10(1+c_1)\mu \alpha} \Big\},
%
\nu\leq\min \Big\{ \frac{c_1 m\beta}{40L_f^2} , $
$\frac{2c_1 m\beta}{5\tau} ,\frac{2c_1\beta\mu^2 m}{375L_f^4}, \frac{5 \tau}{3m c_1} ,
%
\frac{\tau}{6 m (1+1/c_1)}, \frac{3\mu\eta \alpha\tau}{17L_f^2} ,\frac{\tau}{3(L_f+ L_f^2/\mu) } \Big\} ,
$
then the following convergence result for the \alg algorithm holds:
\begin{align}
\frac{1}{T}\!\sum_{t\!=\!0}^{T-1} \Eb[\mathfrak{M}_t]
\!\le\!
\frac{\Eb [\mathfrak{p}_{0} - Q^*] }{\min\{C_1,C_2,C_3,\nu L_f^2/2\}(T+1)},\notag
\end{align}
where $Q^*=Q({{\x}}^*)$ and $\mathfrak{p}_t$ is a potential function defined as:
\begin{align} \label{eqn_Thm1_pt}
 \mathfrak{p}_{t}&  \triangleq Q({\bar{{\x}}}_{t})  + \frac{4\nu L_f^2}{ \beta \mu\eta^2 } \|{\bar{{\y}}}_{t}  -  {\y}_{t}^*\|^2   \notag\\&+  \frac{1}{m} \sum_{i = 1}^{m} [\|{\x}_{i,t}  - {\bar{{\x}}}_t\|^2  + \|{\y}_{i,t}  - {\bar{{\y}}}_t\|^2 ],
\end{align}
and $C_1,C_2, C_3\geq 0$ are constants. 
Due to space limitation, detailed definition of these constants are relegated to our Appendix.
Also, in (\ref{eqn_Thm1_pt}), $Q({{\x}}_t) \triangleq \max_{{{\y}}} F({{\x}}_t,{{\y}}) + h (  {\x}_t)$,  
and ${{\y}}_t^* = \arg\max_{{{\y}}} F({{\x}}_t,{{\y}})$.

\end{thm}

\begin{figure*}[htbp!]
	\centering
	\subfigure[Loss function value vs. sample complexity.]{
		\includegraphics[width=0.225\textwidth]{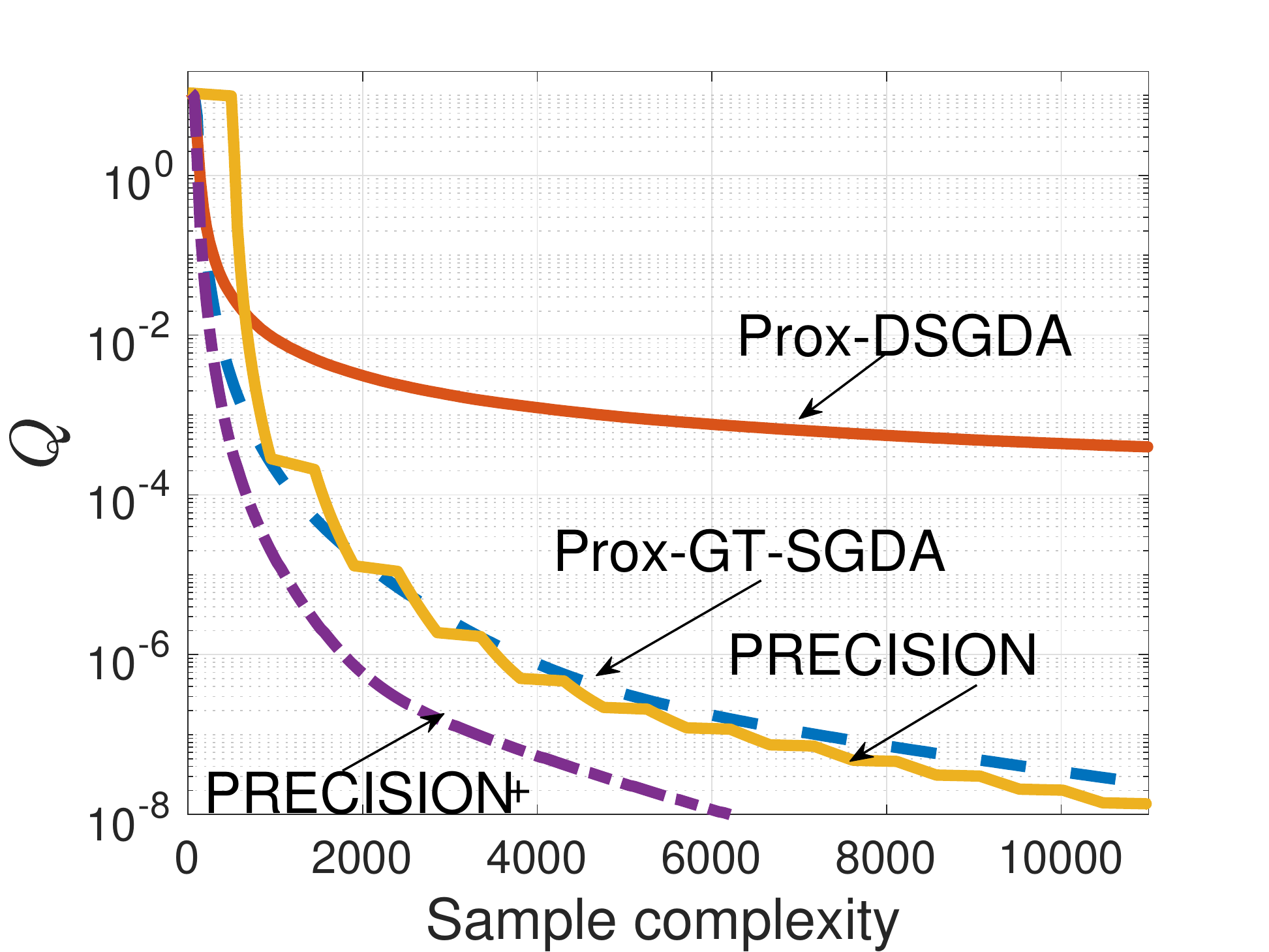}
		\label{fig:sample_MSPBE}
	}
	\hspace{0.001\textwidth}
	\subfigure[Convergence metric vs. sample complexity.]{
		\includegraphics[width=0.225\textwidth]{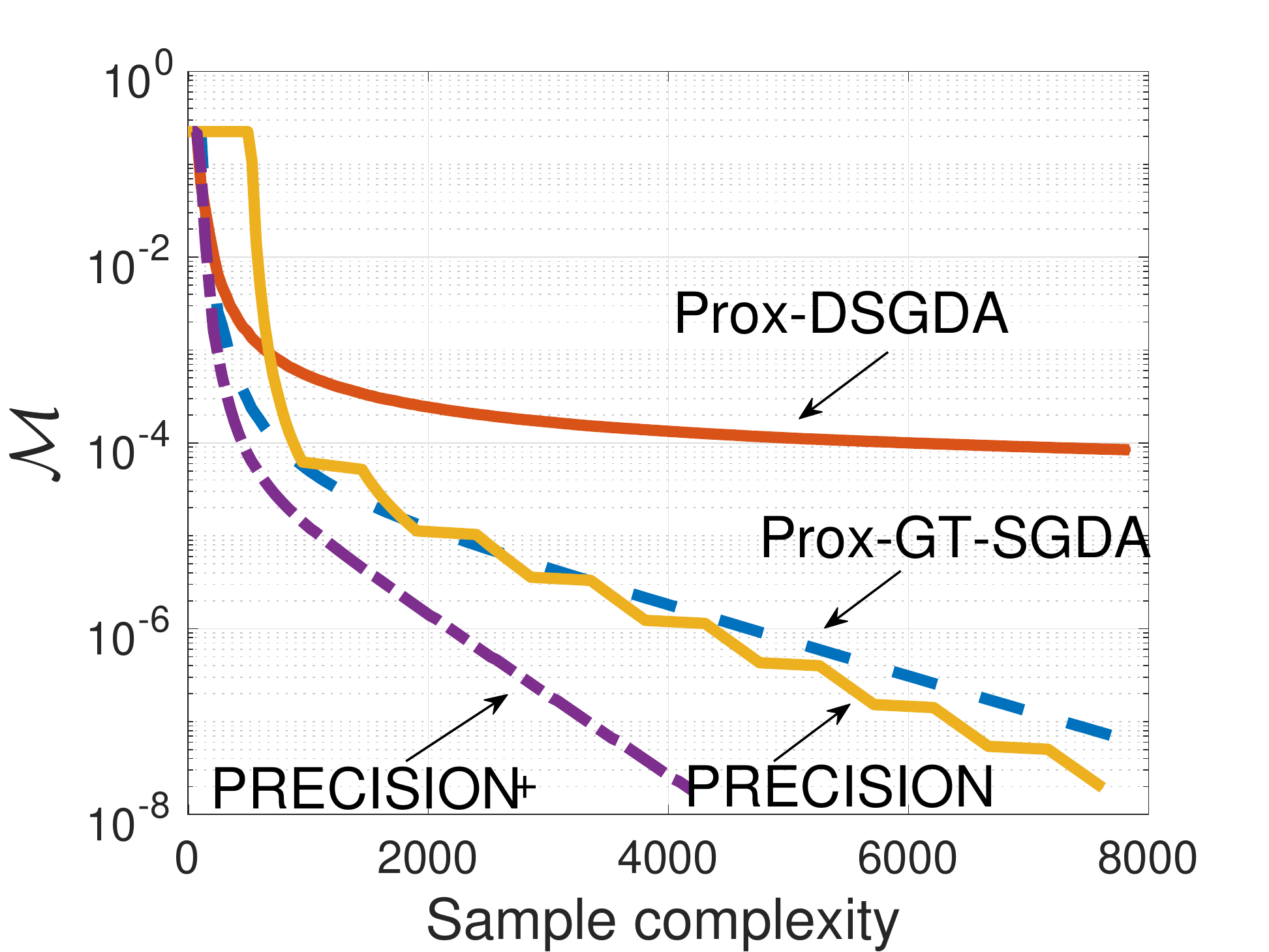}
		\label{fig:sample_Metric}
	}
	\hspace{0.001\textwidth}
	\subfigure[Loss function value vs. communication complexity.]{
		\includegraphics[width=0.225\textwidth]{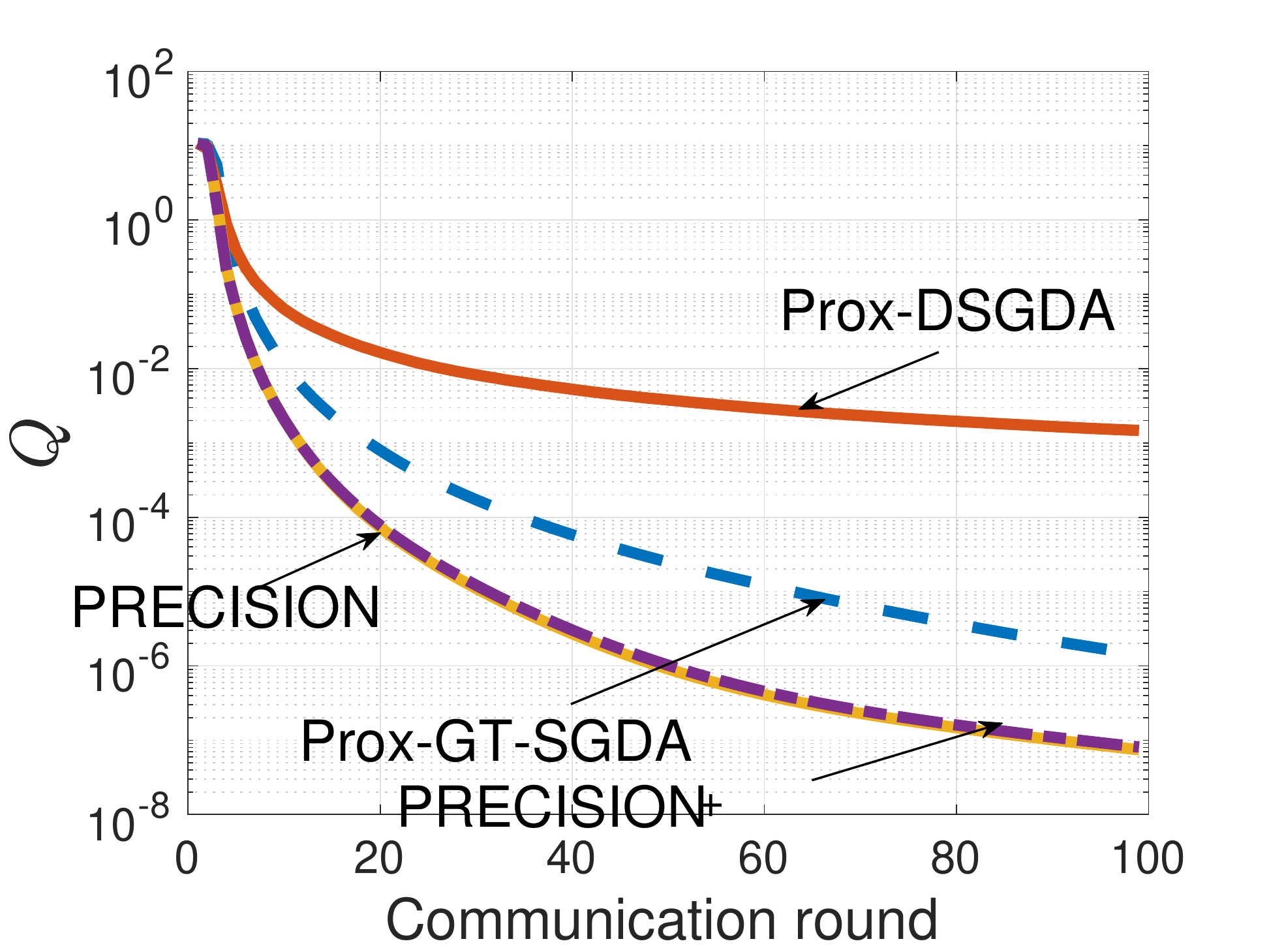}
		\label{fig:comms_MSPBE}
	}
	\hspace{0.001\textwidth}
	\subfigure[Convergence metric  vs. communication \!\!  complexity.]{
		\includegraphics[width=0.225\textwidth]{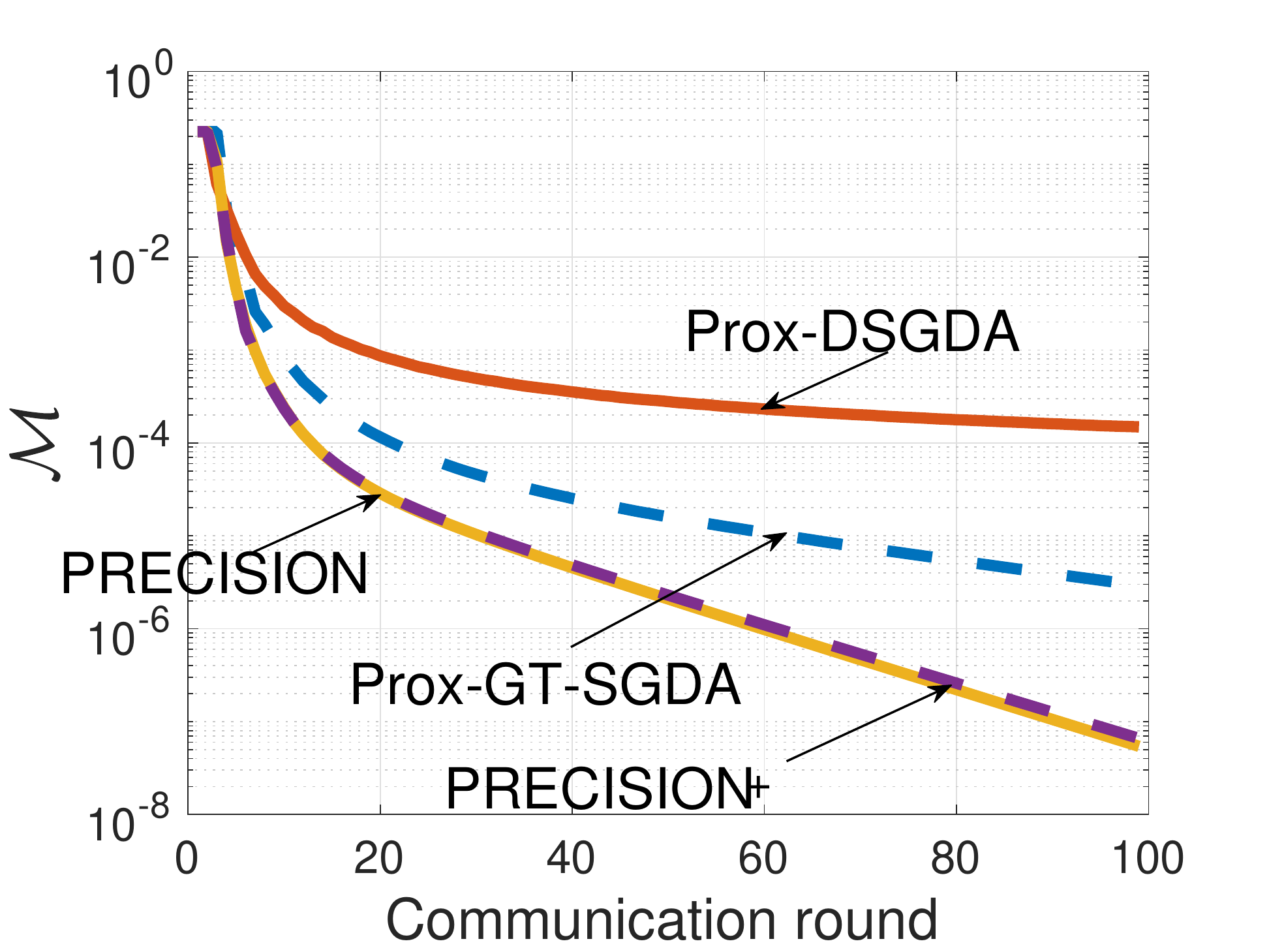}
		\label{fig:comms_Metric}
	}
	\caption{Comparisons of algorithms for decentralized NCX-SCV min-max optimization problems.}
	\label{fig_compare}
\end{figure*}

\begin{thm}[Convergence of \algplusns]\label{Thm: Prox-GT-SRVR+}
Under Assumption~\ref{Assump: obj} (a)-(d), Assumptions~\ref{Assump: Individual Lipschitz}-\ref{Assump: Bounded Variance}, and the same parameter settings as in Theorem \ref{Thm: Prox-GT-SRVR}, with additional parameters $c_{\gamma}$ and $c_{\epsilon}$ satisfying the conditions:
\begin{align} &
c_{\gamma} \geq  (\frac{75 {\eta\alpha}}{8 \mu}\frac{1}{m} +\frac{\nu}{ \beta } \frac{1}{m} ) \frac{\nu\tau}{12}, \quad c_{\epsilon}>0, 
\end{align}and the potential function as stated in Theorem \ref{Thm: Prox-GT-SRVR}, the following convergence result for \algplus holds:
\begin{align} \label{thm2_eq}
\frac{1}{T}\!\sum_{t\!=\!0}^{T-1} \Eb[\mathfrak{M}_t]
\!&\le
\frac{\Eb [\mathfrak{p}_{0} -Q^*] }{(T+1)\min\{ C_1,C_2',C_3, \nu L_f^2/2\}} \notag\\
&+ \bigg( \frac{75 {\eta\alpha}}{16 \mu}\frac{2}{m} +\frac{\nu}{2 \beta } \frac{2}{m}  \bigg) \frac{ \epsilon}{c_{\epsilon} },
\end{align}
where the constant $C_2'\geq$ and the definition of $C_2'$ is relegated to our Appendix.
\end{thm}

\begin{rem}\label{rem22}{\em
Compared to existing works on decentralized min-max optimization\cite{tsaknakis2020decentralized, zhang2021taming}, it is worth noting that the main difficulty in establishing convergence results in Theorem~\ref{Thm: Prox-GT-SRVR} and Theorem~\ref{Thm: Prox-GT-SRVR+} arises from the proximal operator in the outer-level subproblem. 
This operator precludes the use of conventional descent lemmas for convergence analysis, as outlined in Lemma~\ref{Desceding J} in the Appendix.
 Furthermore, unlike in single-agent constrained bilevel optimization, the direct proximal extension of the algorithm in ~\cite{hong2022divergence} $(\widetilde{{\x}}_{i,t} ={\arg\min}_{{\x} \in \mathcal{X}} \| {\x}- ({\x}_{i,t} - \tau\p_{i,t}) \|^2$) will diverge for the decentralized constrained min-max problem in this paper.
  To address this challenge, we employ a special proximal update rule in~\eqref{Eq: DA_updating_x}. 
  The proximal operator $\tilde{{\x}} _{i,t}$ in~\eqref{Eq: DA_updating_x}, consensus updating~\eqref{Eq: consenus_updating_x}, and the corresponding local update~\eqref{Eq: consenus_updating_x} are the key in addressing the {\em non-smooth} objective challenge encountered in decentralized learning.
  }
\end{rem}

\begin{rem} 
	{\em
		In  Theorems~\ref{Thm: Prox-GT-SRVR} and \ref{Thm: Prox-GT-SRVR+}, the step-sizes and convergence rates depend on the network topology.
		For a sparse network, $\lambda$ is close to (but not exactly) one (recall that $\lambda = \max\{|\lambda_2|,|\lambda_m|\} <1$),
      the step-size needs to be smaller as $\lambda$ gets close to one, which leads to a slower convergence.
Additionally, the convergence performance of \algplus is affected by constant $( \frac{75 {\eta\alpha}}{16 \mu}\frac{2}{m} +\frac{\nu}{2 \beta } \frac{2}{m}  ) \frac{ \epsilon}{c_{\epsilon} }$, which depends on the inexact gradient estimation at the $t$-th iteration with $\text{mod}(t,q) = 0$.
Intuitively, a larger value of $c_\epsilon$ allows us to use a larger batch size as shwon in \eqref{choose_batch}, which in turn leads to faster convergence. Theoretically, we can observe that a larger value of $c_\epsilon$ results in a smaller constant $( \frac{75 {\eta\alpha}}{16 \mu}\frac{2}{m} +\frac{\nu}{2 \beta } \frac{2}{m}  ) \frac{ \epsilon}{c_{\epsilon} }$ in \eqref{thm2_eq}, thereby yielding a more accurate estimation.}
\end{rem}

Following from Theorems~\ref{Thm: Prox-GT-SRVR} and \ref{Thm: Prox-GT-SRVR+}, we immediately have the sample and communication complexity results for the \alg and \algplus algorithms:
\begin{cor}\label{Cor: complexity-Prox-GT-SRVR}
Under the conditions in Theorems~\ref{Thm: Prox-GT-SRVR} and ~\ref{Thm: Prox-GT-SRVR+}, and with $q=\sqrt{n}$,
 to achieve an $\epsilon$-stationary solution, 
the following results for the \alg and \algplus algorithms hold:
 \begin{itemize}[topsep=1pt, itemsep=-.1ex, leftmargin=.15in]
 	\item Communication Complexity: the numbers of total communication rounds are upper bounded by $\mathcal{O}(\epsilon^{-2})$
	\item Sample Complexity: The total samples evaluated across the network are upper bounded by $\mathcal{O}(m\sqrt{n}\epsilon^{-2}))$.
\end{itemize}
\end{cor}

\begin{rem}{\em
The \algns/\algplus algorithms have the same communication complexity as GT-GDA \cite{tsaknakis2020decentralized}, but the sample complexity is a $\sqrt{n}$-factor lower than that of GT-GDA \cite{tsaknakis2020decentralized}.
This is particularly advantageous in ``big data'' scenarios, where $n$ is large (i.e., the size of local datasets is large).
Although the theoretical complexity bounds for \algplus is the same as \algns, the fact that \algplus does not need full gradient evaluations implies that \algplus uses significantly fewer samples than \alg in practice.
Our numerical results in the next section will also empirically confirm this.
}
\end{rem}

\section{Experimental Results}\label{Section: experiment}


In this section, we conduct numerical experiments to demonstrate the performance of our proposed \alg and \algplus algorithms using a decentralized NCX-SCV regression problem on ``a9a" dataset from LIBSVM repository, which is publicly available in \cite{chang2011libsvm}.
In the supplementary material, we also provide additional experiments for environments of AUC maximization problem on dataset ``a9a"\cite{chang2011libsvm} and `MNIST"\cite{lecun2010mnist}.
Due to the lack of existing algorithms for decentralized NCX-SCV with simultaneous outer and inner constraint sets (cf. Section~\ref{Sec: Related Work on Policy Evaluation} for details),
we compare our algorithms with two stochastic algorithms as  the baselines in our experiments.
These baselines can be viewed as ``stripped-down'' versions of \alg/\algplus by removing gradient tracking or variance reduction techniques. Due to the space limitation, detailed experimental settings are relegated to our Appendix.

\smallskip
\textbf{1) Logistic Regression Model and Datasets:}
We use the following decentralized NCX-SCV min-max regression problem with datasets
$\left\{\left(\mathbf{a}_{ij}, b_{ij}\right)\right\}_{j=1}^{n}$, where $\mathbf{a}_{ij} \in \mathbb{R}^{d}$ is the feature of the $j$-th sample of agent $i$ and $b_{ij} \!\in\! \{1,-1\}$ is the associated label:
\begin{align}
\min _{{\x}_i \in \mathcal{X}}\max _{{\y}_i \in \mathcal{Y}} \frac{1}{m}\sum_{i=1}^m F_i({\x}_i, {\y}_i),
\end{align}
where $F_i({\x}_i, {\y}_i)$ is defined as:
\begin{align} \label{eqn_Fi}
	F_i({\x}_i, {\y}_i) \triangleq &\frac{1}{n}\sum_{j=1}^{n}\left({y}_{ij} l_{ij}({\x}_{i})-V({\y}_{i})+g({\x}_{i})\right). 
\end{align}
In \eqref{eqn_Fi}, the loss function is $l_{ij}({\x}_{i}) \triangleq \log \left(1+\exp \left(-b_{ij} \mathbf{a}_{ij}^{\top} {\x}_{i} \right)\right)$ and $g({\x}_{i})$ is a non-convex regularizer defined as:
$
g({\x}_{i}) \triangleq \lambda_{2} \sum_{k=1}^{d} \frac{\alpha {x}_{ik}^{2}}{1+\alpha {x}_{ik}^{2}},
$
where $V(\bm{\y_i})=\frac{1}{2} \lambda_{1}\|n {\y}_i-\mathbf{1}\|_{2}^{2}$ and we set the constraints $ \mathcal{X} = [0,10]^d, \mathcal{Y} = [0,10]^n$.
We choose constants $\lambda_{1}=1 / n^{2}$, $ \lambda_{2}=10^{-3}$ and $\alpha=10$. 
We test the convergence performance of our algorithms using the ``a9a" dataset from LIBSVM repository, which is publicly available at \cite{chang2011libsvm}.

\smallskip
\textbf{2) Algorithms comparision:}
Due to the very limited results of decentralized constrained min-max optimization in the literature, in our experiments, we adopt the following algorithms as our baselines for performance comparisons:
\begin{list}{\labelitemi}{\leftmargin=1em \itemindent=-0.09em \itemsep=.2em}
	\item {\em Prox-DSGDA} (proximal decentralized stochastic gradient descent ascent): This algorithm is motivated by DSGD \cite{nedic2009distributed,jiang2017collaborative}.
	Each agent updates its local parameters as $\vtheta_{i,t+1}  =  \sum_{j \in \mathcal{N}_{i}} [\M]_{ij}\vtheta_{j,t}  - \gamma  \frac{1}{|\Sc_{i,t}|} \sum_{j \in \Sc_{i,t}}   \nabla_{\vtheta} f_{ij}(\vtheta_{i,t},\vomega_{i,t})$ and $\vomega_{i,t+1}  =  \sum_{j \in \mathcal{N}_{i}} [\M]_{ij}\vomega_{j,t}  - \eta \frac{1}{|\Sc_{i,t}|} \sum_{j \in \Sc_{i,t}}  \nabla_{\vomega} f_{ij}(\vtheta_{i,t},\vomega_{i,t})$. 
	\item {\em Prox-GT-SGDA} (proximal gradient-tracking-based stochastic gradient descent ascent): This algorithm is motivated by the GT-SGD algorithm \cite{xin2020improved,lu2019gnsd}. GT-SGDA has the same structure as that of GT-GDA, but it updates $\v_{i,t}$ and $\u_{i,t}$ using stochastic gradients as follows: $\v_{i,t} = \frac{1}{|\Sc_{i,t}|} \sum_{j \in \Sc_{i,t}}   \nabla_{\vtheta} f_{ij}(\vtheta_{i,t},\vomega_{i,t})$ and  $\u_{i,t} =  \frac{1}{|\Sc_{i,t}|} \sum_{j \in \Sc_{i,t}}  \nabla_{\vomega} f_{ij}(\vtheta_{i,t},\vomega_{i,t})$.
\end{list}

\textbf{3) Results:}
From Fig.~\ref{fig:sample_MSPBE} and~\ref{fig:sample_Metric}, we can see that our proposed \algplus algorithm  converges much faster than other algorithms (\algns, Prox-GT-SGDA and Prox-DSGDA) in terms of the total number of first-order oracle evaluations.
We can also observe that both \alg and \algplus have lower sample complexities than those of the other two algorithms. 
As shown in Figs.~\ref{fig:comms_MSPBE} and~\ref{fig:comms_Metric}, \alg and \algplus have much lower communication costs than those of Prox-DSGDA and Prox-GT-SGDA.
Our experimental results thus verify our theoretical analysis that \alg/\algplus have both low sample and communication complexities in decentralized constrained min-max optimization problems.

\section{Conclusion}\label{Section: conclusion}

In this paper, we studied the decentralized constrained non-convex-strongly-concave (NCX-SCV) min-max optimization and developed two algorithms called \alg and \algplusns.
We showed that, to achieve an $\epsilon$-stationary point of a decentralized constrained NCX-SCV min-max problem, \alg and \algplus achieve the communication complexity of $\mathcal{O}(\epsilon^{-2})$ and sample complexity of $\mathcal{O}(m\sqrt{n}\epsilon^{-2})$,
where $m$ is the number of agents and $n$ is the size of dataset for each agent.
Our numerical studies also verified the theoretical performance of our proposed algorithms.
We note that decentralized constrained min-max learning remains an under-explored area, and our work opens up several interesting directions for future research. 
 For example, the agents need to send outer and inner model parameter pairs to their neighbors in our algorithm, both of which could be high dimensional. 
In our future work, it would be interesting to adopt communication-efficient mechanisms (e.g., compression techniques) to further reduce the communication cost, especially for large-scale deep learning models.

\bibliographystyle{abbrvnat}
\bibliography{./reference,./reference_mobihoc21}

\appendix

\section{Proof Sketch of Main Results}
Due to space limitation,  we outline the key steps of the proofs of Theorems~\ref{Thm: Prox-GT-SRVR} and \ref{Thm: Prox-GT-SRVR+}. The complete version of our proofs is available in our Appendix.
Before diving in our theoretical analysis, we first provide the following notations:
\begin{itemize}
	\item $\xb_t =\frac{1}{m}\sum_{i=1}^{m} \x_{i,t}$ and $\x_t = [\x_{1,t}^{\top},\cdots, \x_{m,t}^{\top}]^{\top}$ for any vector $\x$;
	\item $\nabla_{{{\x}}} F_t\! =\! [\nabla_{{{\x}}} F({{\x}}_{1,t}, {{\y}}_{1,t})^\top,\cdots,\nabla_{{{\x}}} F({{\x}}_{m,t}, {{\y}}_{m,t})^\top]^{\top} $;
	\item $\nabla_{{{\y}}} F_t \!= \![\nabla_{{{\y}}} F({{\x}}_{1,t}, {{\y}}_{1,t})^\top,\cdots,\nabla_{{{\y}}} F({{\x}}_{m,t}, {{\y}}_{m,t})^\top]^{\top} $;
	\item $\Ec(\x_t) \!= \!\frac{1}{m}\sum_{i=1}^{m} \|\x_{i,t} - \xb_t\|^2$ for any vector $\x$.
\end{itemize}

Also, the result below is useful for our subsequent analysis.
\begin{lem}\label{Lemma: Lip_J}
	Under Assumption~\ref{Assump: obj}, the funciton $J({{\x}}) = F({{\x}}, \y^*({{\x}}))$ w.r.t ${{\x}}$ is Lipschitz smooth, i.e., there exists a positive constant $L_{J}$, such that
	\begin{align}
	\|\nabla J({{\x}}) - \nabla J({{\x}}^\prime)\| \le L_{{J}} \|{{\x}} -{{\x}}^\prime\|,~~\forall {{\x}}, {{\x}}^\prime \in \Rb^d,
	\end{align}
	where the Lipschitz constant is $L_{{J}} = L_f + L_f^2/\mu$ for Algorithm \ref{Algorithm: Prox-GT-SRVR}. This lemma follows immediately from Lemma 4.3 in \cite{lin2020gradient}.
\end{lem}
\begin{lem}\label{Lemma: Lip_w2}
	Under Assumption~\ref{Assump: obj}, ${{\y}}^*({{\x}}) = \arg\max_{{{\y}}} F({{\x}},{{\y}})$ is Lipschitz continuous, i.e., there exists a positive constant $L_{{{y}}}$, such that
	\begin{align}
	\|{{\y}}^*({{\x}}) - {{\y}}^*({{\x}}^\prime)\| \le L_{{{y}}} \|{{\x}} -{{\x}}^\prime\|,~~\forall {{\x}}, {{\x}}^\prime \in \Rb^d,
	\end{align}
	where the Lipschitz constant is $L_{{{y}}} = L_f/\mu$.
\end{lem}
\subsection{ Important Lemmas for Proving Main Theorems}\label{proof}
		We first show the following descent property of \alg algorithm on the function $Q(\cdot)$, which is stated in the following lemma:
		
		\begin{lem}[Descent Inequality on $Q({{\x}})$]\label{Desceding J}
			Under Assumption~\ref{Assump: obj}, the following descent inequality holds:
			\begin{align}\label{Eq: descending J}
			&Q ({\bar{{\x}}}_{t+1}) - Q({\bar{{\x}}}_{t})\le \frac{\nu L_F^2}{2 \beta } \left\| {\y}_t^*- \bar{{\y}_t} \right\|^{2}\notag\\&+\frac{\nu}{2 \beta }  \left\|\nabla_{{\x}} F (\bar{\bm{{\x}}}_t,\bar{\bm{{\y}}}_t)-\bar{\p}_{t}\right\|^{2}
			+\!\frac{\nu \tau}{2 \beta m} \left\|{\x}_{t}\!-\! 1 \otimes \bar{{\x}}_{t}\right\|^{2}
			\!\notag\\&-\! \left(\frac{\nu \tau}{m}-\frac{\nu^{2} L_J}{2 m}- \frac{\nu \beta}{m}\! -\!\frac{\nu \tau \beta}{2 m}\right)\left\|\tilde{{\x}}_{t} -1 \otimes \bar{{\x}}_{t}\right\|^{2}.
			\end{align}
			where $Q({{\x}}_t) = \max_{{{\y}}} F({{\x}}_t,{{\y}}) + h (  {\x}_t)$ and ${{\y}}_t^* = \arg\max_{{{\y}}} F({\bar{\x}}_t,{{\y}})$.
		\end{lem}
		\begin{proof}[Proof Sketch of Lemma~\ref{Desceding J}]
			Let $J({{\x}}_t) = \max_{{{\y}}} F({{\x}}_t,{{\y}})$. According to the algorithm update, Lipschitz continuous gradients of $J$ and optimal conditions of $h({\x})$, we have:
			\begin{align}
			&J({\bar{{\x}}}_{t+1}) - J({\bar{{\x}}}_{t})
		\notag\\
			&	\leq \frac{{\nu}}{m} \sum_{i}\left\langle\nabla  J\left(\bar{{\x}}_{t}\right)-\p_{i,t}, \tilde{{\x}}_{i,t}-\bar{{\x}}_{t}\right\rangle \notag\\
			&+\frac{{\nu} \tau}{m} \sum_{i}\left\langle {\x}_{i,t}-\bar{{\x}}_{t}, \tilde{{\x}}_{i,t}-\bar{{\x}}_{t}\right\rangle
			+\frac{{\nu}^{2} L_J}{2 m}\left\|\tilde{{\x}}_{t}-1 \otimes \bar{{\x}}_{t}\right\|^{2} \notag\\
			&-\frac{{\nu} \tau}{m}\left\|\tilde{{\x}}_{t}-1 \otimes \bar{{\x}}_{t}\right\|^{2}-h\left(\bar{{\x}}_{t+1}\right)+h\left(\bar{{\x}}_{t}\right).
			\end{align}
			From triangle inequality and the definition of $Q({\x})$, we have:
			%
			%
		From triangle inequality and the definition of $Q({\x})$, we arrive at the result stated in Lemma \ref{Desceding J}.
		\end{proof}

	Next, consider the error bound $ \left\|\bar{{\y}}_{t}-{{\y}}_{t}^{*}\right\|^{2}$ in Lemma \ref{Desceding J},
		we have	the following Lemma:
		
		\begin{lem}[Error Bound on ${{\y}}^*({{\x}})$]\label{Error Bound on omega}
			Under Assumption \ref{Assump: obj}, the following inequality holds for \algns/\algplus:
			\begin{align}\label{eqn_step2}
	&	\left\|\bar{{\y}}_{t+1}\!-\!{{\y}}_{t+1}^{*}\right\|^{2} \!\leq \!\left(1\!-\!\frac{\mu {\eta\alpha}}{4}\right)\left\|\bar{{\y}}_{t}-{{\y}}_{t}^{*}\right\|^{2}\!\notag \\&-\!\frac{3 {\eta}}{4}\left\|\tilde{{\y}}_{t}-1 \otimes \bar{{\y}}_{t}\right\|^{2}
		+\frac{75 {\eta\alpha}}{16 \mu}\left\|{\bd_t}-\nabla_{{\y}} F\left(\bar{{\x}}_t, \bar{{\y}}_{t}\right)\right\|^{2}\notag \\&+\frac{17 L_{{y}}^{2}\nu^2 }{2\mu {\eta\alpha}m}\left\|\tilde{{\x}}_{t}-1 \otimes \bar{{\x}}_{t}\right\|^{2}.
			\end{align}
		\end{lem}
		
		\begin{proof}[Proof Sketch of Lemma~\ref{Error Bound on omega}]
			
			Similar to \cite{qiu2020single}, Lemma B.2,  B.3 and due to the optimality condition for the constrained optimization on ${\y}$ and the $\mu$-strongly concavity, we have
			\begin{align}\label{2_1}
			&\left\|\bar{{\y}}_{t+1}-{{\y}}_{t}^{*}\right\|^{2}
			  \leq \frac{4 {\eta}^2}{\mu}\left\|\nabla_{{\y}} F\left(\bar{{\x}}_t, \bar{{\y}}_{t}\right)-{\bd_t}\right\|^{2} \notag\\&+\left(1-\frac{{\eta}^2 \mu}{2}\right)\left\|\bar{{\y}}_{t}-{{\y}}_{t}^{*}\right\|^{2}-\frac{3 {\eta}}{4}\left\|\tilde{{\y}}_{t}-1 \otimes \bar{{\y}}_{t}\right\|^{2}.
			\end{align}
			
			Furthermore, we have
			\begin{align}\label{2_2}
			&\left\|\bar{{\y}}_{t+1}-{{\y}}_{t+1}^{*}\right\|^{2} \leq\left(1+\frac{\mu {\eta}^2}{4}\right)\left\|\bar{{\y}}_{t+1}-{{\y}}_{t}^{*}\right\|^{2}\notag\\&-\left(1+\frac{4}{\mu {\eta}^2}\right) L_{{y}}^{2} \nu^2 \left\|\frac{1}{m} \sum_i \tilde{{\x}}_{i,t}-\bar{{\x}}_t\right\|^{2}.
			\end{align}

			From triangle inequality and the definition of $Q({\x})$, we have:
			\begin{align}
			&Q\left(\bar{{\x}}_{t+1}\right)
			\leq Q\left(\bar{{\x}}_{t}\right)+\frac{\nu}{m} \sum_{i} \frac{1}{2 \beta}\left\|\nabla J\left(\bar{{\x}}_{t}\right)-\nabla_{{\x}} F (\bar{\bm{{\x}}_t},\bar{\bm{{\y}}_t})\right\|^{2} \notag\\&+\frac{\nu}{m} \sum_{i} \frac{\beta}{2}\left\|\tilde{{\x}}_{i,t}-\bar{{\x}}_{t}\right\|^{2}
			+{\nu}\frac{1}{2 \beta}\left\|\nabla_{{\x}} F (\bar{\bm{{\x}}_t},\bar{\bm{{\y}}_t})-\bar{\p}_{t}\right\|^{2} \notag\\&+\frac{\nu}{m} \sum_{i} \frac{\beta}{2}\left\|\tilde{{\x}}_{i,t}-\bar{{\x}}_{t}\right\|^{2}
			+\frac{\nu \tau}{m} \frac{1}{2 \beta} \sum_{i}\left\|\bar{{\x}}_{t}-{\x}_{i,t}\right\|^{2} \notag\\&+\frac{\nu \tau}{m} \sum_{i} \frac{\beta}{2}\left\|\tilde{{\x}}_{i,t}-\bar{{\x}}_{t}\right\|^{2}-\left(\frac{\nu \tau}{m}-\frac{\nu^{2} L_J}{2 m}\right)\left\|\tilde{{\x}}_{t}-1 \otimes \bar{{\x}}_{t}\right\|^{2}.\notag
			\end{align}
			After some rearrangements of the above inequality, we arrive at the result stated in Lemma \ref{Error Bound on omega}.

		\end{proof}
			
	By telescoping the combined results of previous lemmas from $0$ to $T+1$ and after some rearrangements, we arrive at the following results:
		
		\begin{lem}\label{Lem: Descending_Q}
			Under Assumption \ref{Assump: obj} and condition $\eta\le 1/2L_f$, the following inequality holds for \algns/\algplus:
			\begin{align}\label{eqn_step3}
		&
		Q({\bar{{\x}}}_{T\!+\!1}) \!-\! Q({\bar{{\x}}}_{0}) \!+\! \frac{4\nu L_F^2}{ \beta \mu\eta^2 }\big[\|{\bar{{\y}}}_{T\!+\!1} \!- \!{{\y}}_{T\!+\!1}^*\|^2  \!- \!\|{{\y}}_0^* \!-\! {\bar{{\y}}}_0\|^2\big] \notag\\
		\!\le\!
		& \frac{4 \nu L_F^2}{ \beta \mu{\eta\alpha} } \big\{  -\frac{3 {\eta}}{4}\left\|\tilde{{\y}}_{t}-1 \otimes \bar{{\y}}_{t}\right\|^{2}  +\frac{17L_{{y}}^{2}\nu^2 }{2 \mu m {\eta\alpha}}\left\|\tilde{{\x}}_{t}-1 \otimes \bar{{\x}}_{t}\right\|^{2}  \big\} \notag\\& + \frac{75 {\eta\alpha}}{16 \mu}\frac{2}{m}\|\nabla_{{{\y}}} F({{\x}}_{t},{{\y}}_{t}) - \bd_{t}\|^2 \notag\\&+\frac{\nu}{2 \beta } \frac{2}{m}\|\nabla_{{{\x}}} F({{\x}}_{t},{{\y}}_{t}) - \bp_{t}\|^2   \notag\\&
		+\frac{\nu \tau}{2 \beta m} \left\|{\x}_t-1 \otimes \bar{{\x}}_{t}\right\|^{2}
		-\big(\frac{\nu \tau}{m}-\frac{\nu^{2} L_J}{2 m}-\frac{\nu \beta}{m}-\frac{\nu \tau \beta}{2 m}\big)\notag\\& \cdot \left\|\tilde{{\x}}_{t}\!-\!1\! \otimes \bar{{\x}}_{t}\right\|^{2}
		\!+\!\big[   \frac{\nu}{ \beta }  \frac{ L_F^2}{m}\!+\! \frac{4\nu L_F^2}{\beta \mu{\eta\alpha}} \frac{75 {\eta\alpha}}{16 \mu}\frac{2L_F^2}{m}  \big]\notag\\&  \cdot \sum_{i=1}^{m}[\|{\bar{{\x}}}_t - {{\x}}_{i,t}\|^2 + \|{\bar{{\y}}}_t - {{\y}}_{i,t}\|^2 ]\notag\\& -\frac{\nu L_F^2}{2\beta}\left\|\bar{{\y}}_{t}-{{\y}}_{t}^{*}\right\|^{2}.
			\end{align}
		\end{lem}
	
		Next, we bound the iterates contraction of $\|{{\x}}_t-1 \otimes {\bar{{\x}}}_t\|^2 $ and $\|{{\y}}_t-1 \otimes {\bar{{\y}}}_t\|^2 $ in (\ref{eqn_step3}). 
		\begin{lem}[Iterates Contraction]\label{Lem: Contraction}
			The following contraction properties of the iterates hold:
			\begin{align} \label{4_1}
			\|{{\x}}_t-1 \otimes {\bar{{\x}}}_t\|^2 & \le (1+c_1)\lambda^2\|{{\x}}_{t-1} -1 \otimes {\bar{{\x}}}_{t-1} \|^2 \notag\\
			&+ (1+\frac{1}{c_1}) \nu^2\|\tilde{{\x}}_{t-1}- {{\x}}_{t-1}\|^2, \notag\\
			\|{{\y}}_t-1 \otimes {\bar{{\y}}}_t\|^2 & \le (1+c_2)\lambda^2\|{{\y}}_{t-1} -1 \otimes {\bar{{\y}}}_{t-1} \|^2 \notag\\
			&+ (1+\frac{1}{c_2}) {\eta}^2\|\tilde{{\y}}_{t-1}- {{\y}}_{t-1}\|^2,
			\end{align}
			where $c_1$ and $c_2$ are arbitrary positive constants.
			Additionally, we have
			\begin{align} \label{4_2}
			\|{{\x}}_t-{{\x}}_{t-1}\|^2
			& \le
			8\Ec({{\x}}_{t-1}) + 2\nu^2 \|\tilde{{\x}}_{t-1}-{\x}_{t-1}\|^2, \notag\\
			\|{{\y}}_t-{{\y}}_{t-1}\|^2
			& \le
			8\Ec({{\y}}_{t-1}) + 2\eta^2 \|\tilde{{\y}}_{t-1}-{\y}_{t-1}\|^2 .
			\end{align}
			
		\end{lem}
%
%
%
	%
		Next, we bound the gradient tracking errors $\sum_{t = 0}^{T}\|\bd_{t} - \nabla_{{{\x}}} F_{t}\|^2 $ and $ 	\sum_{t = 0}^{T}\|\bp_{t} - \nabla_{{{\y}}} F_{t}\|^2 $ in (\ref{eqn_step3}).
		
		\begin{lem}[Error of Gradient Estimator]\label{Lem: SRVR err}
			Under Assumption \ref{Assump: Individual Lipschitz},  we have the following error bounds for the gradient trackers:
			\begin{align}
			&\sum_{t = 0}^{T}\|\bd_{t} - \nabla_{{{\x}}} F_{t}\|^2  \notag\\ \le &
			\sum_{t = 1}^{T} \Eb\|\bd_{(n_t-1)q}-\nabla_{{{\x}}} F({{\x}}_{(n_t-1)q}, 
			{{\y}}_{(n_t-1)q} )\|^2
			\notag\\&+ L_f^2\big(\|{{\x}}_{t} - {{\x}}_{t-1}\|^2 + \|{{\y}}_{t} - {{\y}}_{t-1}\|^2\big), \\
			&\sum_{t = 0}^{T}\|\bp_{t} - \nabla_{{{\y}}} F_{t}\|^2 \notag\\\le&
			\sum_{t = 1}^{T}\Eb\|\bp_{(n_t-1)q}-\nabla_{{{\y}}} F({{\x}}_{(n_t-1)q} ,
			{{\y}}_{(n_t-1)q} )\|^2
			\notag\\&+ L_f^2\big(\|{{\x}}_{t} - {{\x}}_{t-1}\|^2 + \|{{\y}}_{t} - {{\y}}_{t-1}\|^2\big),
			\end{align}
			where $n_t$ is the largest positive integer satisfing $(n_t-1)q\le t$.
		\end{lem}
		
		\begin{proof}[Proof Sketch of Lemma~\ref{Lem: SRVR err}]
		Define
			\begin{align}
			{A_{i,t}}=&{\bd_{i,t} \!-\! \nabla_{{{\x}}} F_{i,t}};\
			B_{i,t}= \frac{1}{|\Sc_{i,t}|}\!\!\sum_{j \in \Sc_{i,t}}\!\! \! \nabla_{{{\x}}} f_{i,t}({{\x}}_{i,t},{{\y}}_{i,t})\!\notag\\&-\! \nabla_{{{\x}}}  f_{i,t}({{\x}}_{i,t\!-\!1},{{\y}}_{i,t\!-\!1})
			\!+ \!\nabla_{{{\x}}} F_{i,t\!-\!1}  \!-\! \nabla_{{{\x}}} F_{i,t}.
			\end{align}
			Note that $\Eb_t[B_{i,t}] = 0$, where the expectation is taken over the randomness of data sampling at the $t$-th iteration. Thus,
			\begin{align}
			\Eb_t\|A_{i,t}\|^2
			=
			\|A_{i,t-1}\|^2 + \Eb_t\|B_{i,t}\|^2.
			\end{align}

			Also, with $|\Sc_{i,t}| = q$, we have
			\begin{align}\label{4_3}
			&
			\Eb_t\|B_{i,t}\|^2
			\le
			\frac{L_f^2}{q}\big(\|{{\x}}_{i,t} \!-\! {{\x}}_{i,t-1}\|^2 + \|{{\y}}_{i,t} \!-\! {{\y}}_{i,t-1}\|^2\big).
			\end{align}
			
			Taking full expectation and telescoping (\ref{4_3}) over $t$ from $(n_t-1)q + 1$ to $t$, where $t\le n_t q - 1$, we have
			$	\Eb\|A_{t}\|^2
			\le
			\Eb\|A_{(n_t-1)q}\|^2 +
			\sum_{r = (n_t-1)q + 1}^{t}
			\frac{L_f^2}{q} \Eb\big(\|{{\x}}_{r} - {{\x}}_{r-1}\|^2 + \|{{\y}}_{r} - {{\y}}_{r-1}\|^2\big).\notag
			$
			Thus, 
			$
			\sum_{k= 0}^{t}
			\Eb\|A_{k}\|^2
			\leq
			\sum_{r = 0}^{t}\|A_{(n_r-1)q}\|^2\!+
			\sum_{r = 1}^{t}
			L_f^2\big(\|{{\x}}_{r} - {{\x}}_{r-1}\|^2 + \|{{\y}}_{r} - {{\y}}_{r-1}\|^2\big).\notag
			$
			We have similar result while $ {A_{i,t}}={\bp_{i,t} \!-\! \nabla_{{{\y}}} F_{i,t}}$.
			This completes the proof of of Lemma. \ref{Lem: SRVR err}.
		\end{proof}

\subsection{ Proof Sketch of Theorem~\ref{Thm: Prox-GT-SRVR}} 
\begin{proof}
Following the defined potential function $\mathfrak{p}$ and the result of Lemma \ref{Desceding J}-\ref{Lem: SRVR err}, we have
		\begin{align}
		&\Eb \mathfrak{p}_{T+1} \!\!-\! \mathfrak{p}_{0}
		\!\le\!{\nu {L_f}^2}{2} \!\sum_{t\!=\!0}^{T}\! \left\|\bar{{\y}}_{t}-{{\y}}_{t}^{*}\right\|^{2} \notag\\
		&-C_1 \!\sum_{t\!=\!0}^{T}\! \sum_{i=1}^{m}\|{\bar{{\x}}}_t - {{\x}}_{i,t}\|^2-
		\!\!C_2\!\sum_{t\!=\!0}^{T}\! \left\|\tilde{{\x}}_{t} -1 \otimes \bar{{\x}}_{t}\right\|^{2} \notag\\
		&- C_3 \!\sum_{t\!=\!0}^{T}\! \sum_{i=1}^{m}[ \|{\bar{{\y}}}_t \!-\!{{\y}}_{i,t}\|^2 ]
		\!-\!C_4 \!\sum_{t\!=\!0}^{T}\! \left\|\widetilde{{\y}}_{t} -1 \otimes \bar{{\y}}_{t}\right\|^{2}  ,
		\end{align}
		where 
		\begin{align}
		C_1=&\big[ 1-8L_f^2\big( \frac{75 {\eta\alpha}}{16 \mu}\frac{2}{m} + \frac{\nu}{2 \beta } \frac{2}{m} \big)-\frac{\nu \tau}{2 \beta m}\notag\\&- (1+c_1)\lambda^2-  \frac{\nu}{ \beta }  \frac{ L_f^2}{m}   -  \frac{4 \nu L_f^2}{ \beta \mu{\eta\alpha} } \frac{75 {\eta\alpha}}{16 \mu}\frac{2L_f^2}{m}   \big], \notag\\
		C_2=&\big(-2\nu^2L_f^2\big( \frac{75 {\eta\alpha}}{16 \mu}\frac{2}{m} + \frac{\nu}{2 \beta } \frac{2}{m} \big)- (1+\frac{1}{c_1})\nu^2 \notag\\& -\frac{17L_{{\y}}^{2}\nu^2 }{2 \mu m {\eta\alpha}}+\frac{\nu \tau}{m} -\frac{\nu^{2} L_J}{2 m}-\frac{\nu \beta}{m}-\frac{\nu \tau \beta}{2 m}\big),\notag\\
		C_3=&\big[ 1\!-\!8L_f^2\big( \frac{75 {\eta\alpha}}{16 \mu}\frac{2}{m} \!+\! \frac{\nu}{2 \beta } \frac{2}{m} \big)\!-\! (1\!+\!c_2)\lambda^2 \notag\\&-\!  \frac{\nu}{ \beta }  \frac{ L_f^2}{m}\!-\!  \frac{4 \nu L_f^2}{ \beta \mu{\eta\alpha} } \frac{75 {\eta\alpha}}{16 \mu}\frac{2L_f^2}{m}   \big] ,\notag\\
		C_4=& \frac{4 \nu L_f^2}{ \beta \mu{\eta\alpha} }\frac{3 {\eta}}{4} \notag\\& - (1+\frac{1}{c_2})\eta^2  -2{\eta\alpha}L_f^2\big( \frac{75 {\eta\alpha}}{16 \mu}\frac{2}{m} + \frac{\nu}{2 \beta } \frac{2}{m} \big).\notag
		\end{align}	

  Suppose that $\beta\leq \min \Big\{ \frac{\tau}{12}  ,\frac{1}{3} \Big\}$,

$
\!\!\!\!\!	\alpha\leq \frac{1}{4L_f}$ hold and let $c_1= \frac{1-\lambda^2}{1+\lambda^2}$, if step-sizes satisfy Thm.~\ref{Thm: Prox-GT-SRVR}
to ensure $C_1,C_2,C_3,C_4\geq0$. We can conclude that
		\begin{align}
		&\frac{1}{T+1}	\sum_{t=0}^{T}  \mathfrak{M}_t	\le
		\frac{\Eb [\mathfrak{p}_{0} - Q^*] }{\min\{C_1,C_2,\nu L_f^2/2\}(T+1)}.
		\end{align}
	This completes the proof Theorem~\ref{Thm: Prox-GT-SRVR}.
\end{proof}

\subsection{ Proof Sketch of Theorem~\ref{Thm: Prox-GT-SRVR+}} 

\begin{proof}
	For \algplus, we have 
	\begin{align}
	&\Eb\|\bd_{(n_t-1)q}- \nabla_{{{\x}}} F_{(n_t-1)q}\|^2 \notag\\&= \Eb\|\bp_{(n_t-1)q}- \nabla_{{{\y}}} F_{(n_t-1)q}\|^2=\frac{I_{\left(\mathcal{N}_{s}<M\right)}}{\mathcal{N}_{s}} \sigma^{2}.
	\end{align}
	  Recall that $\mathcal{N}_s =\min\{ c_{\gamma} \sigma^2(\gamma^{(k)})^{-1}, c_{\epsilon} \sigma^2\epsilon^{-1} ,M\}  $, we have
	  \begin{align} \label{34}
	  \frac{I_{( \mathcal{N}_s <M)}}{ \mathcal{N}_s}  &  \leq \max\{ \frac{{\gamma^{(k)}}}{c_{\gamma}\sigma^2}, \frac{\epsilon}{c_{\epsilon} \sigma^2}  \}  \leq \frac{{\gamma^{(k)}}}{  c_{\gamma}\sigma^2}+\frac{\epsilon}{ c_{\epsilon}\sigma^2}.
	  \end{align}
	Since $\gamma_{t+1}=\frac{1}{q} \sum_{i=\left(n_{t}-1\right) q}^{t}\left\| \tilde{{\x}}_{t}-1 \otimes \bar{{\x}}_{t} \right\|^{2}$.
	Plugging \eqref{34} to Lemma \ref{Lem: Descending_Q}, we have the following result,
	with additional parameter  setting
$$
	c_{\gamma} \geq  (\frac{75 {\eta\alpha}}{8 \mu}\frac{1}{m} +\frac{\nu}{ \beta } \frac{1}{m} ) \frac{\nu\tau}{12}.$$
For \algplusns, following the defined potential function $\mathfrak{p}$ and the result of Lemma \ref{Desceding J}-\ref{Lem: SRVR err}, with $\mathfrak{p}_{T+1} \ge Q^*$, we reach the conclusion:
	\begin{align}
&\frac{1}{(T+1)}\sum_{t=0}^{T} \mathfrak{M}_t 
\le  
 ( \frac{75 {\eta\alpha}}{16 \mu}\frac{2}{m} +\frac{\nu}{2 \beta } \frac{2}{m}  ) \frac{ \epsilon}{c_{\epsilon} }\notag\\
&+\frac{\Eb [\mathfrak{p}_{0} - \mathfrak{p}_{T+1}] }{(T+1)\min\{ C_1,C_2',\nu L_f^2/2\}},
\end{align}
where $C_1, C_2'\geq 0$. This completes the proof Theorem~\ref{Thm: Prox-GT-SRVR+}.
 
\end{proof}

\newpage
\onecolumn

\section{ Further experiments and additional results} \label{exp}
In the followings, we provide the detailed settings for our experiments:

\smallskip

\textbf{1) AUC Maximization Model and Datasets:}

We apply the following AUC maximization problem with a given dataset $\left\{\mathbf{a}_{ij}, b_{ij}\right\}_{j=1}^{n}$ where $\mathbf{a}_{ij}$ denotes a feature vector and $b_{ij} \in\{-1,+1\}$ indicates the corresponding label. With function $h_{\boldsymbol{\x}}$ of a classification model parameterized by $\boldsymbol{\x_i} \in \mathcal{X}$, the AUC is defined as
\begin{align} \label{eqn_1}
\max _{\boldsymbol{\x_i} \in \mathcal{X}} \frac{1}{e^{+} e^{-}} \sum_{b_{ij}=+1, b_{ij}=-1} \mathbb{I}_{\left\{h_{\boldsymbol{\x}_i}\left(\mathbf{a}_{ij}\right) \geq h_{\boldsymbol{\x}_i}\left(\mathbf{a}_{jj}\right)\right\}},
\end{align}
where $e^{+}\left(e^{-}\right)$ indicates the number of positive (negative) samples and $\mathbb{I}$ denotes the indicator function. 
The above optimization problem has the following equivalent minimax formulation:
$$
\begin{aligned}
	&\min _{\boldsymbol{\x}_i, c_1, c_2} \max _{y_i}  \frac{1}{m} \sum_{i=1}^{m}F_i(\boldsymbol{\x}_i, c_1, c_2, \lambda)\\\notag:=& \frac{1}{m} \sum_{i=1}^{m}\left\{(1-\tau)\left(h_{\boldsymbol{\x}_i}\left(\mathbf{a}_{ij}\right)-c_1\right)^{2} \mathbb{I}_{\left\{b_{ij}=1\right\}}-\tau(1-\tau) y_i^{2}+\tau\left(h_{\boldsymbol{\x}_i}\left(\mathbf{a}_{ij}\right)-c_2\right)^{2} \mathbb{I}_{\left\{b_{ij}=-1\right\}}\right.\\
	&\left.+2(1+y_i) \tau h_{\boldsymbol{\x_i}}\left(\mathbf{a}_{ij}\right) \mathbb{I}_{\left\{b_{ij}=-1\right\}}-2(1+y_i)(1-\tau) h_{\boldsymbol{\x}_i}\left(\mathbf{a}_{ij}\right) \mathbb{I}_{\left\{b_{ij}=1\right\}}\right\},
\end{aligned}
$$
where $\tau:=e^{+} /\left(e^{+}+e^{-}\right)$is the ratio of positive data. 

We test the convergence performance of our algorithms using the ``a9a" dataset from LIBSVM repository, which is publicly available at \cite{chang2011libsvm} and `MNIST"\cite{lecun2010mnist}.

%

%
%
%

\smallskip

\textbf{2) Decentralizednetworks:}
We use a five-node multi-agent system, with the communication graph $\mathcal{G}$ being generated by the Erd$\ddot{\text{o}}$s-R$\grave{\text{e}}$nyi graph, where the edge connectivity probability is $p_c=0.6.$ 
The network consensus matrix is chosen as $\W = \I - \frac{2}{3\lambda_{\text{max}}(\mathbf{L})} \mathbf{L},$ where $\mathbf{L}$ is the Laplacian matrix of $\mathcal{G}$, and $\lambda_{\text{max}}(\mathbf{L})$ denotes the largest eigenvalue of $\mathbf{L}$.
The generated topology is shown in Figure~\ref{fig: topos}.




\begin{figure}[h]
	\centering
	\includegraphics[width=0.3\linewidth]{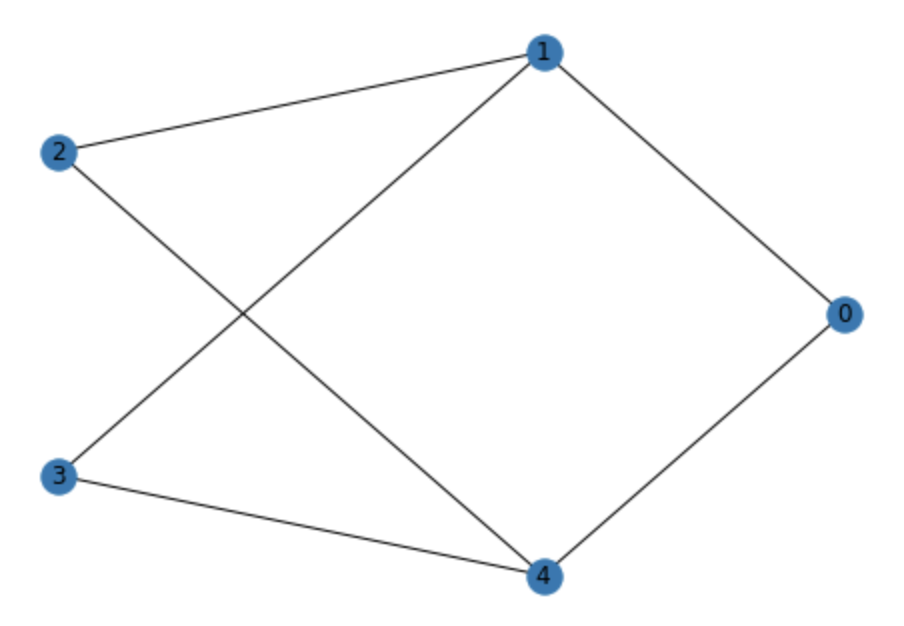}
	\caption{Network topology}
	\label{fig: topos}
\end{figure}

\subsection{Algorithms comparison}
In this subsection, we provide an additional experiment on the algorithms' comparison.
We run all algorithms for solving optimization problem over AUC maximization problem under a9a dataset and mnist dataset.
In this experiment, we initialized the parameters from the normal distribution for all the algorithms and fixed learning rates as $\gamma=10^{-1}, \eta=10^{-1}$.
From Figure~\ref{img_auc}, we observe our proposed algorithms  \algns/\algplus enjoy low sample and communication complexities on solving AUC maximization problem under both ``a9a'' dataset and ``MNIST'' dataset.
%


\begin{figure}[htbp]
	\centering
	\subfigure[Algorithms comparison on ``a9a'' dataset.]{
		\begin{minipage}[t]{0.22\linewidth}
			\centering
			\includegraphics[width=1\textwidth]{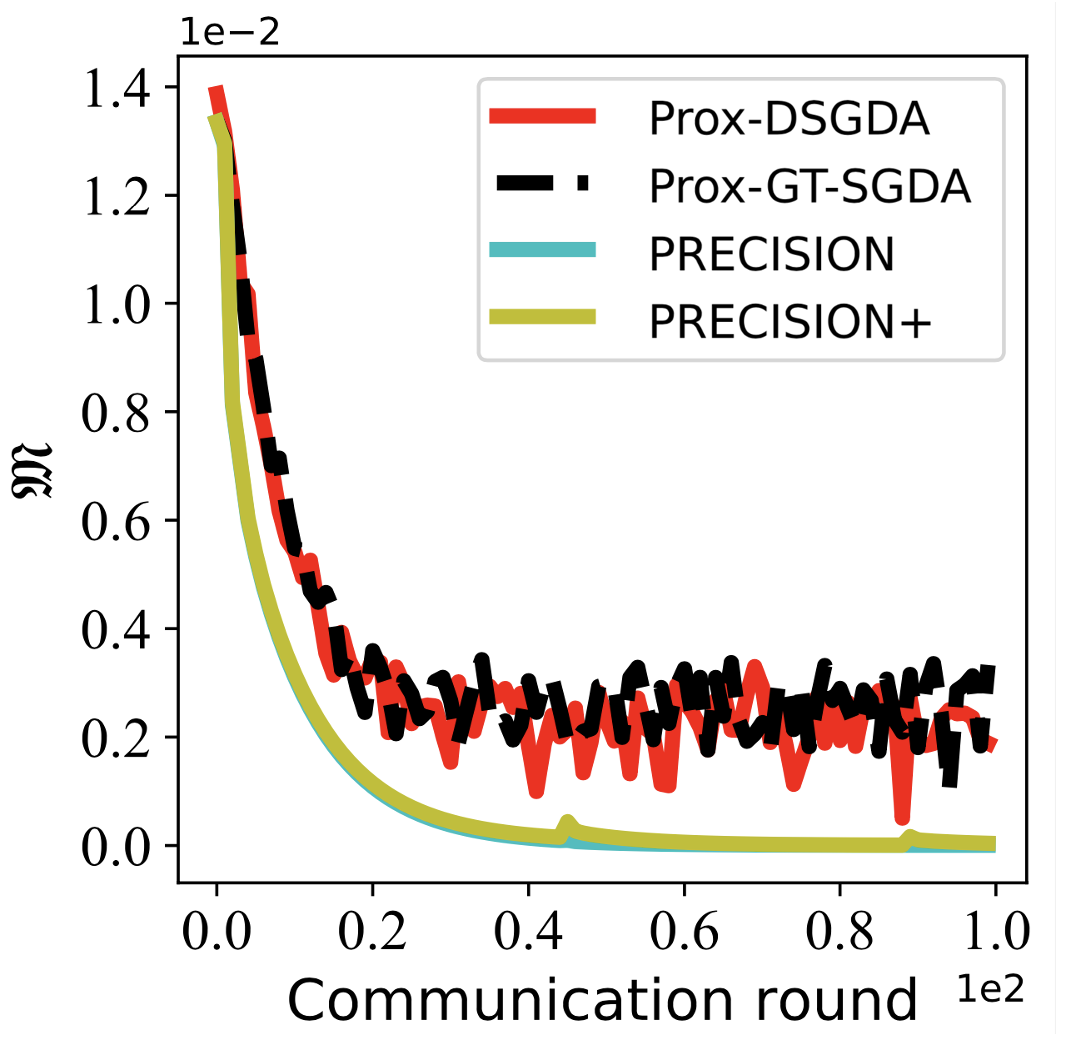} 
		\end{minipage}
		\begin{minipage}[t]{0.22\linewidth}
			\centering
			\includegraphics[width=1\textwidth]{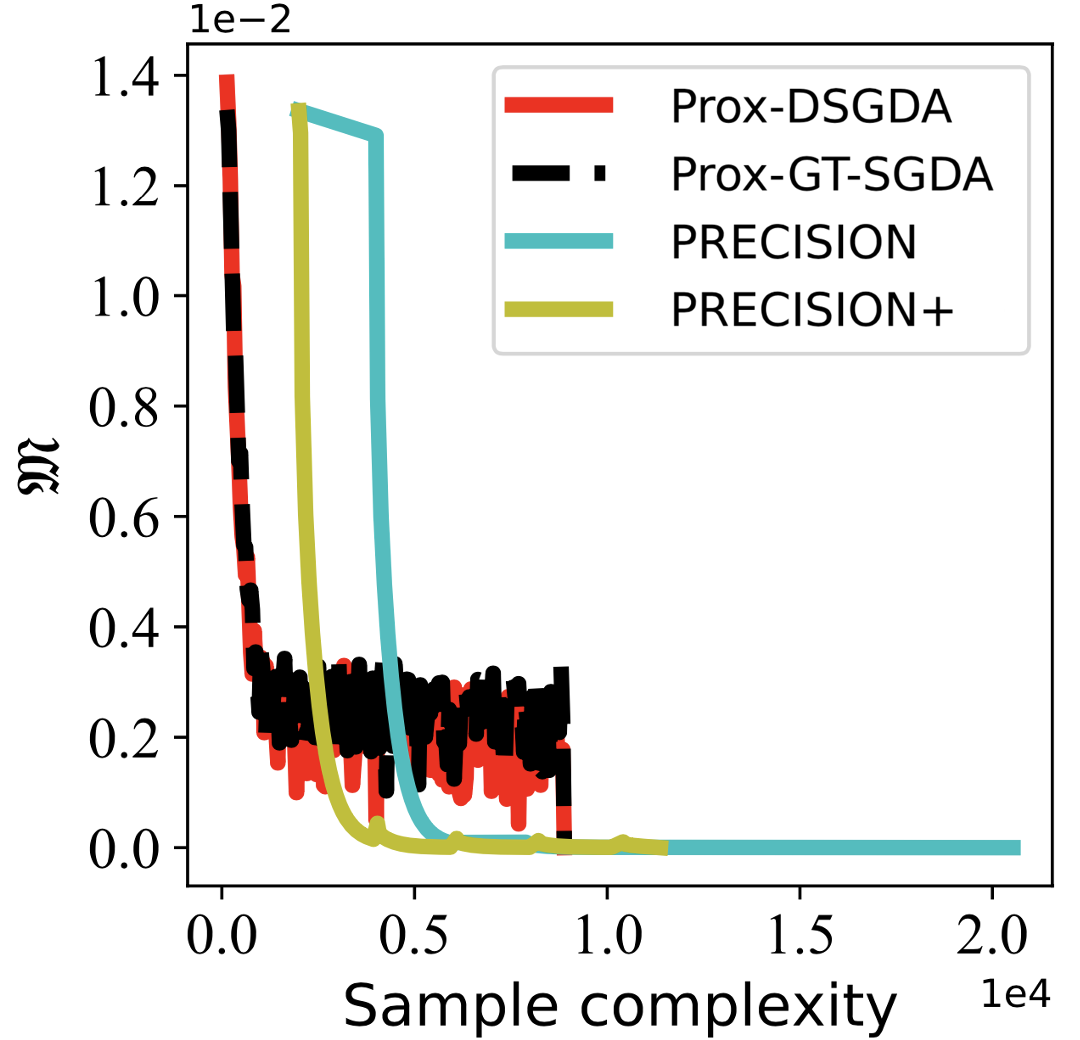} 
		\end{minipage}
		\label{imgauc_regression}
	}	     
	\centering
	\subfigure[Algorithms comparison on ``MNIST'' dataset.]{
		\begin{minipage}[t]{0.22\linewidth}
			\centering
			\includegraphics[width=1\textwidth]{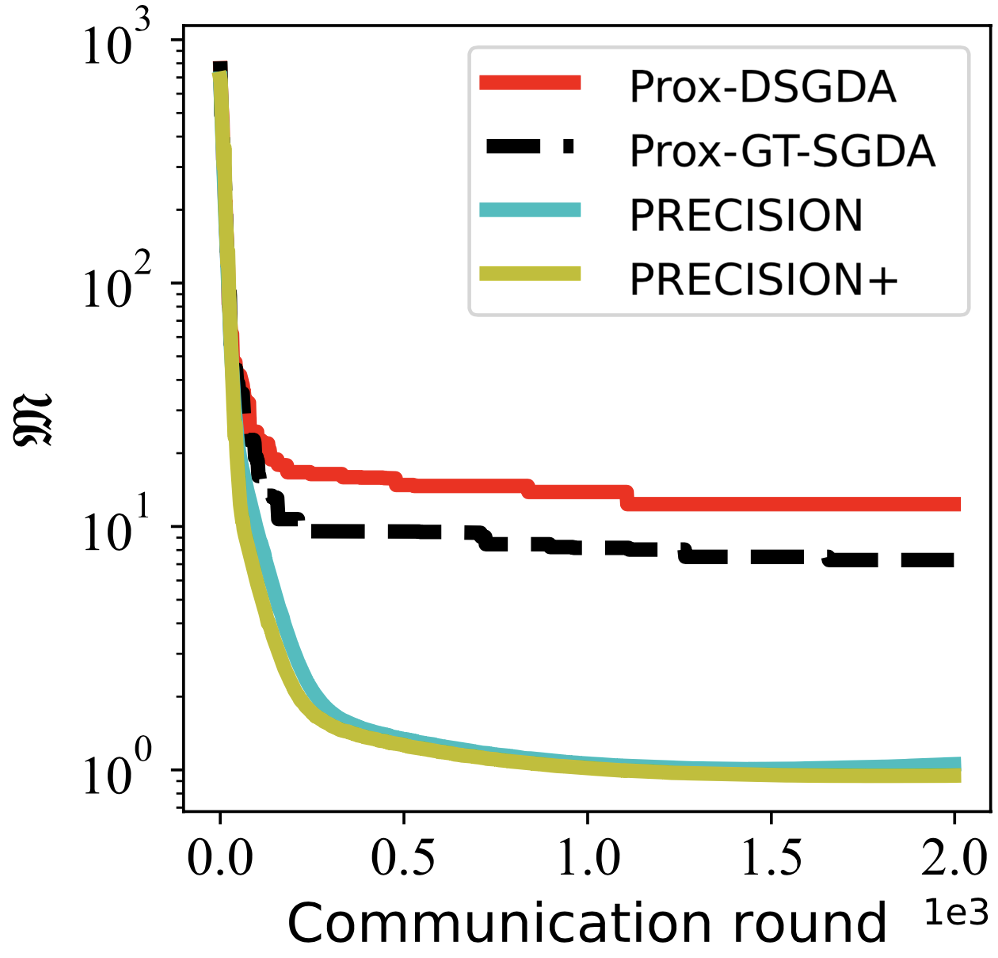}
		\end{minipage}
		\begin{minipage}[t]{0.225\linewidth}
			\centering
			\includegraphics[width=1\textwidth]{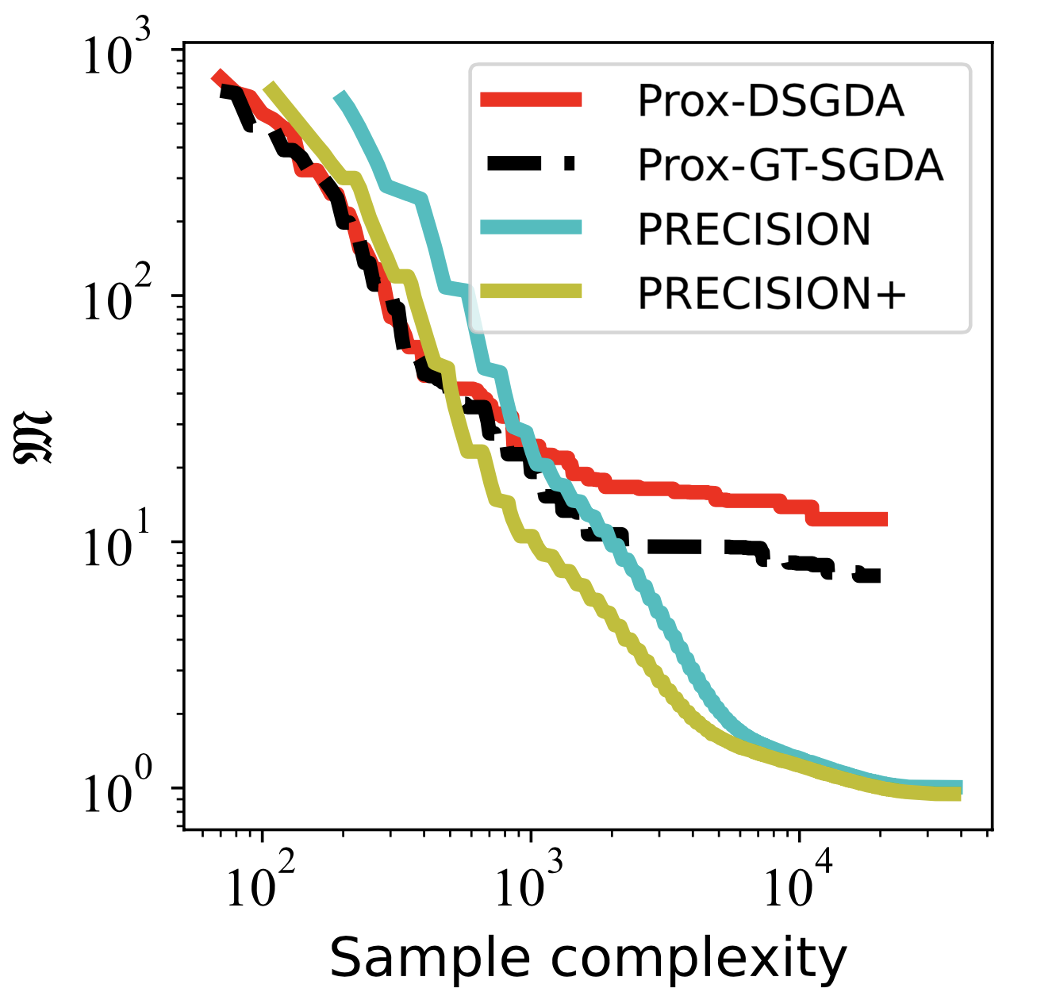}
		\end{minipage}
		\centering
		\label{imgauc_mnist}
	}	\caption{ Algorithms Comparision on AUC maximization problem .}
	\label{img_auc}
\end{figure}

\subsection{Learning rate setting}
We use a 5-node multi-agent system with a generated topology as shown in Figure~\ref{fig: topos}.
In this experiment, we choose the datasize $n=2000$, mini-batch size $q=\lceil\sqrt{n}\rceil$.  
Figs.~\ref {img23} illustrate the convergence metric $\mathfrak{M}$ performance of \alg with different learning rates $\gamma$ and $\eta$. We fix a relatively small learning rate $\gamma=10^{-1}$ while comparing $\eta$; and set $\eta=10^{-1}$ while comparing $\gamma$.  
In this experiment, we observe that methods with a smaller learning rate have a smaller slope in the figure, which leads to a slower convergence.

\begin{figure}[htbp]
	\centering
	\subfigure[Step-size comparison on Regression.]{
		\begin{minipage}[t]{0.215\linewidth}
			\centering
			\includegraphics[width=1\textwidth]{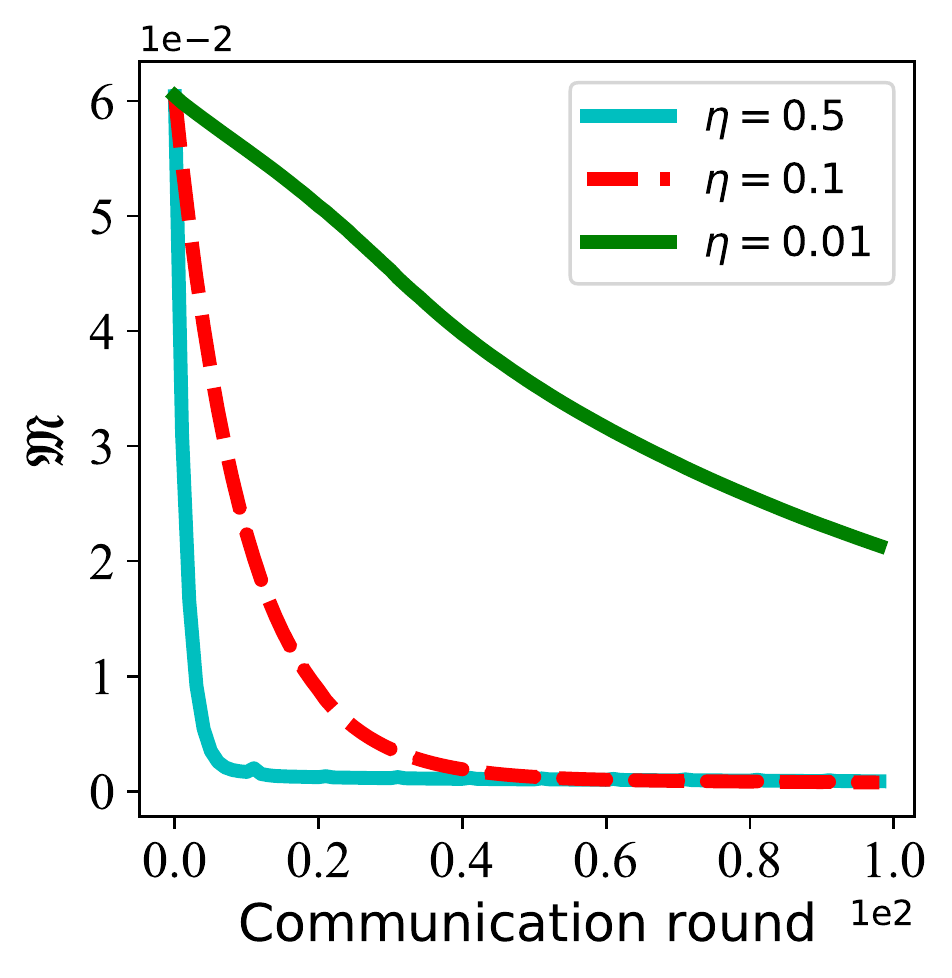} 
		\end{minipage}
		\begin{minipage}[t]{0.23\linewidth}
			\centering
			\includegraphics[width=1\textwidth]{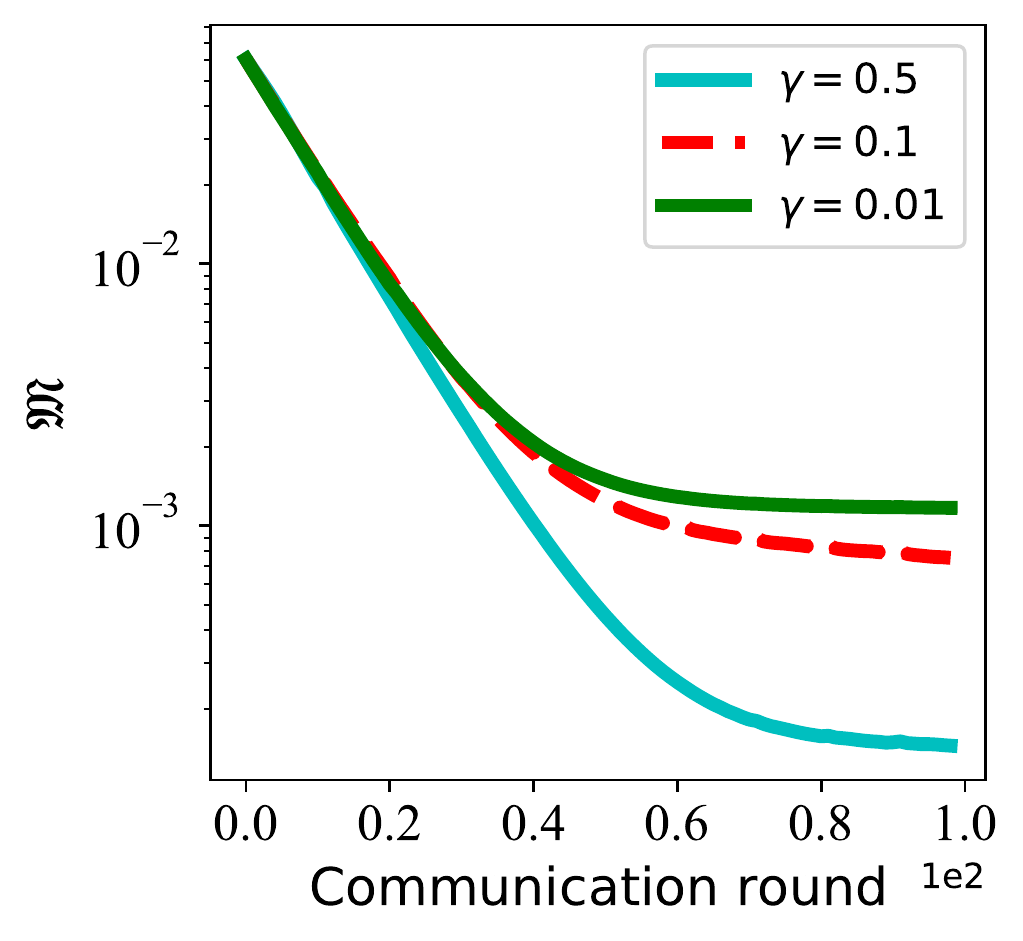} 
		\end{minipage}%
		\label{img23a}
	}	     
	\centering
	\subfigure[Step-size comparison on AUC maximization.]{
		\begin{minipage}[t]{0.22\linewidth}
			\centering
			\includegraphics[width=1\textwidth]{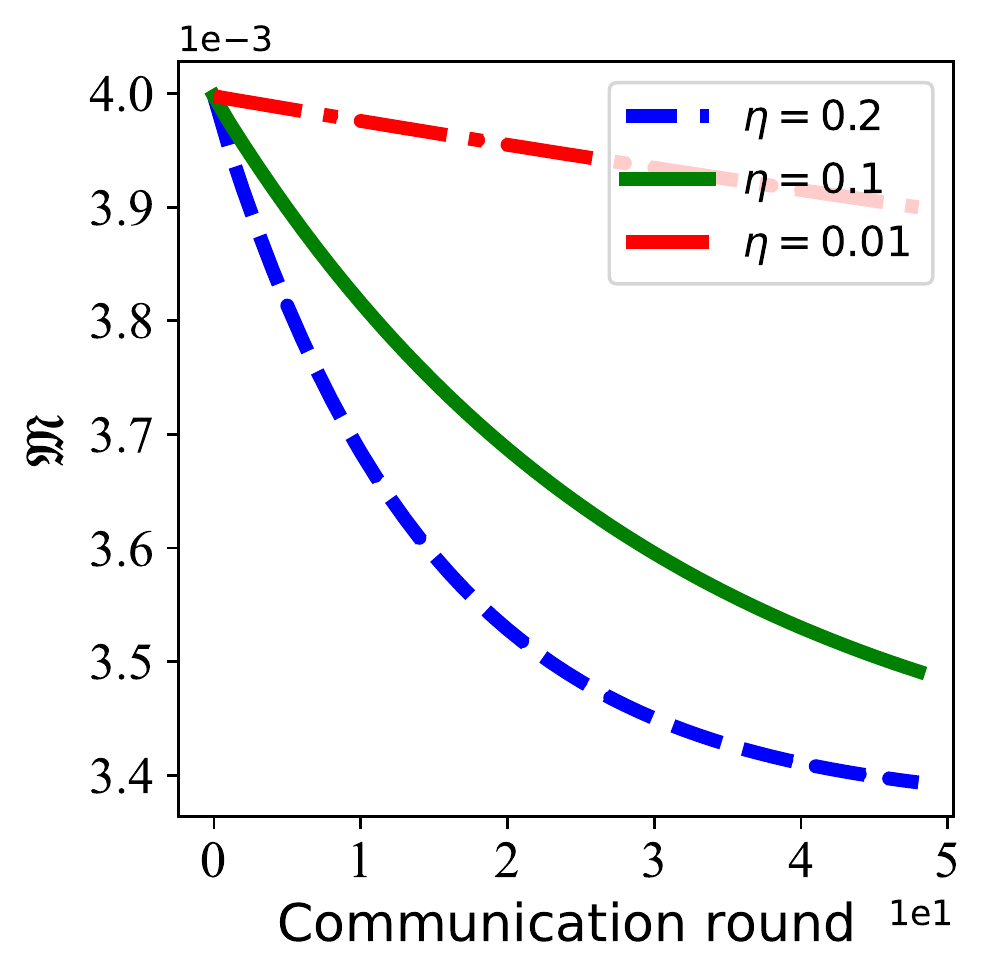}
		\end{minipage}
		\begin{minipage}[t]{0.23\linewidth}
			\centering
			\includegraphics[width=1\textwidth]{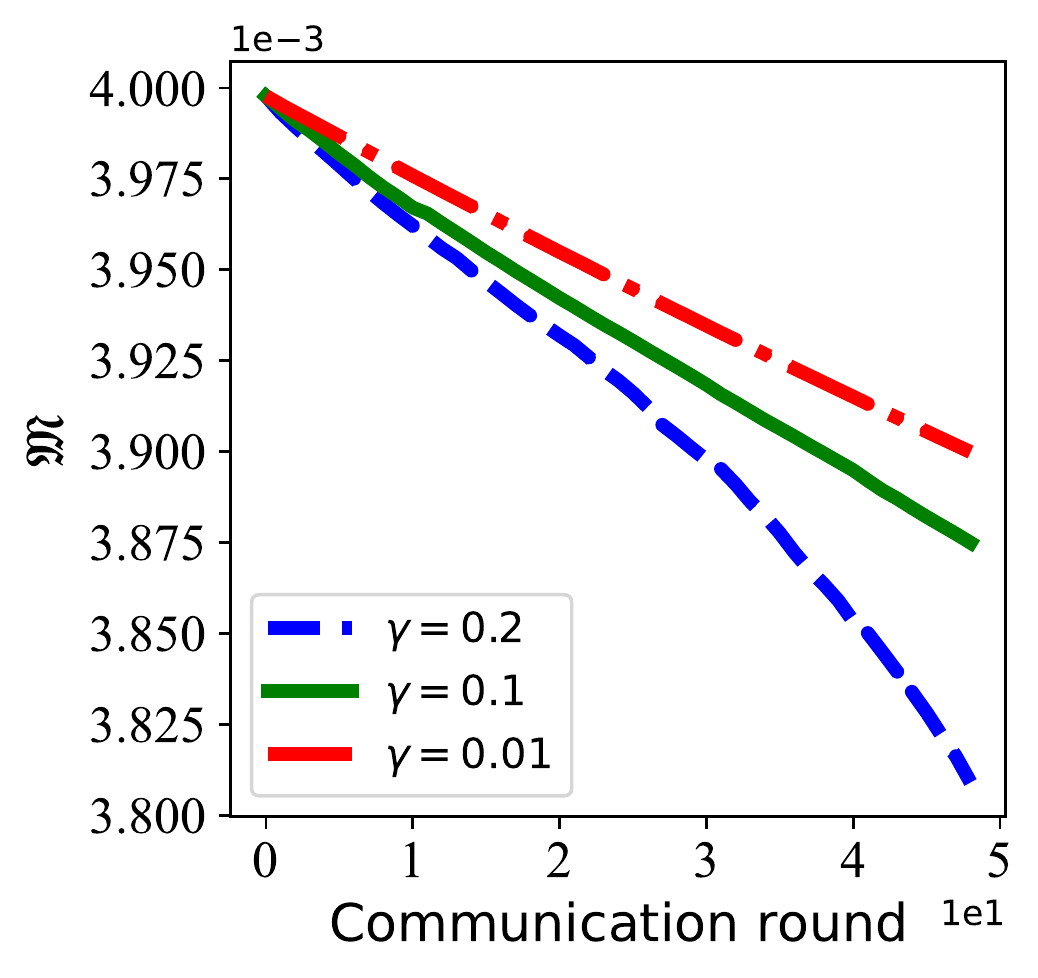}
		\end{minipage}
		\centering
		\label{img23b}
	}	\caption{ Algorithm(\alg) performance  with different step-size.}
\label{img23}
\end{figure}

\subsection{Topology setting}

\begin{figure}[htbp]
	\centering
	\subfigure[Topology sparsity $p_c=0.3$.]{
		\begin{minipage}[t]{0.2\linewidth}
			\centering
			\includegraphics[width=1\textwidth]{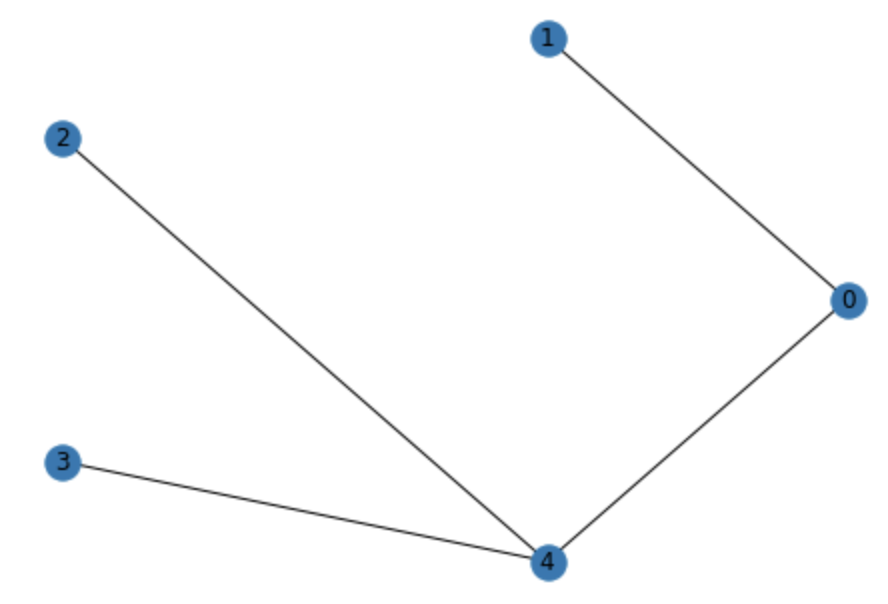} 
		\end{minipage}
	}	     
	\centering
	\subfigure[Topology sparsity $p_c=0.6$.]{
		\begin{minipage}[t]{0.2\linewidth}
	\centering
	\includegraphics[width=1\textwidth]{Figure/p_c_6e1.png}
\end{minipage}
\centering
	}
\subfigure[Topology sparsity $p_c=0.9$.]{
	\begin{minipage}[t]{0.2\linewidth}
		\centering
		\includegraphics[width=1\textwidth]{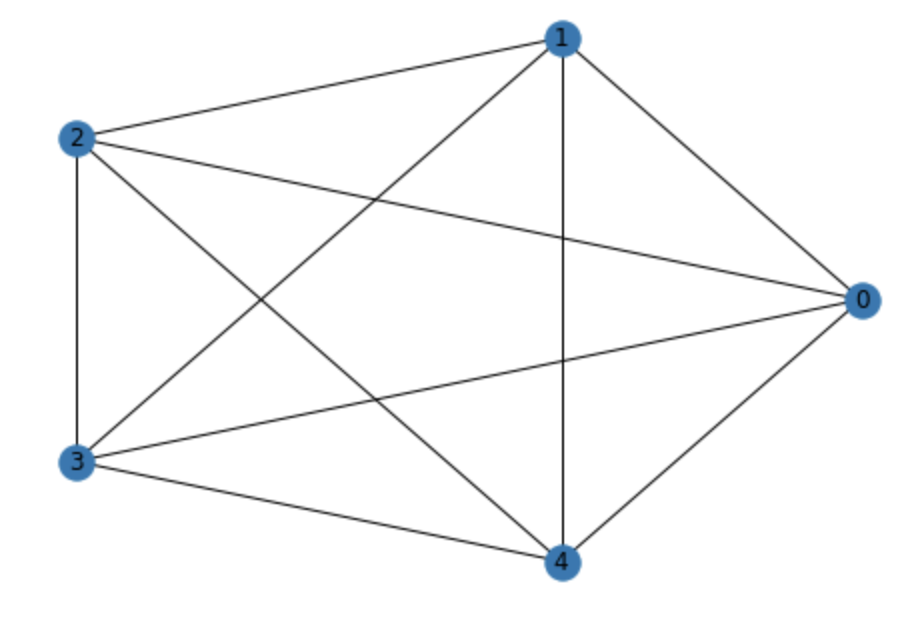}
	\end{minipage}
	\centering
}	
\subfigure[Topology sparsity $p_c=0.5$ with 20 nodes.]{
		\begin{minipage}[t]{0.3\linewidth}
			\centering
			\includegraphics[width=.6\textwidth]{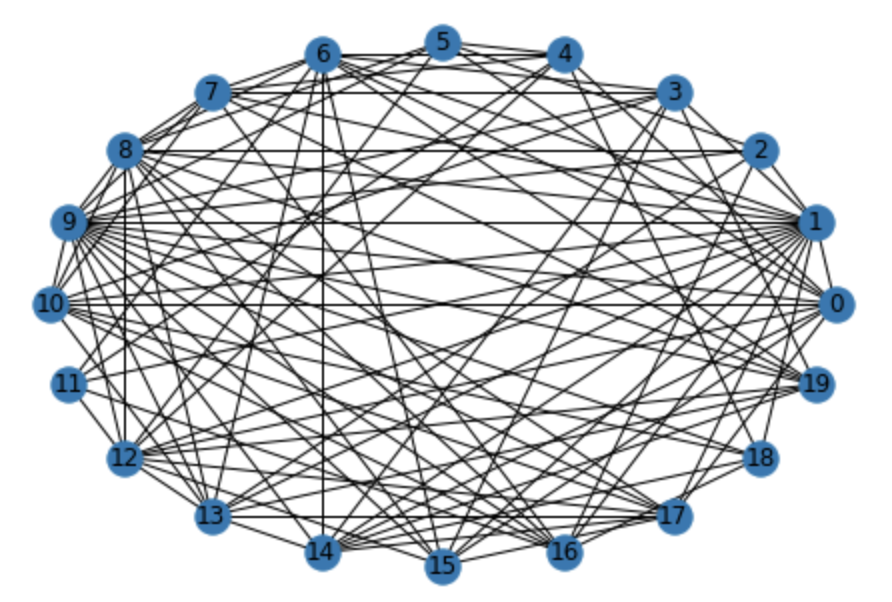}
		\end{minipage}
		\centering
		\label{img7}
	}
\caption{Topology.}
\label{img5}
\end{figure}

We use a 5-node multi-agent system and experiment on three different topologies. 
The generated topology with different sparsity are shown in Fig.~\ref{img5}. 
The datasize for each agent is $n=100$ and we set the constant learning rate $\gamma=0.1$, $\eta=0.1$ and mini-batch size $q=\lceil\sqrt{n}\rceil$.
We observe that the convergence metric $\mathfrak{M}$ is insensitive to the network topology. 
The subplot in Fig.~\ref{img4a} and Fig.~\ref{img4b} show that $\mathfrak{M}$  slightly increase as $p_c$ decreases.

\subsection{Node setting}

We test the following experiments on different multi-agent systems. 
The generated topology with a 20-node system are shown in Figs.~\ref{img7}. 
The constant learning rate $\gamma=0.1$, $\eta=0.1$ and mini-batch size $q=\lceil\sqrt{n}\rceil$.
We compare our proposed algorithm  \algns/\algplus with two baseline algorithms Prox-GT-SGDA and Prox-DSGDA in terms of  the convergence metric in (\ref{Eq: metric1}).
We observe similar results as shown in Section~\ref*{Section: experiment}. 
Thus, we can conclude that our proposed algorithms  \algns/\algplus enjoy low sample and communication complexities in general.

\begin{figure}[htbp]
	\centering
		\subfigure[Topology sparsity comparison on Regression]{
		\begin{minipage}[t]{0.23\linewidth}
			\centering
			\includegraphics[width=1\textwidth]{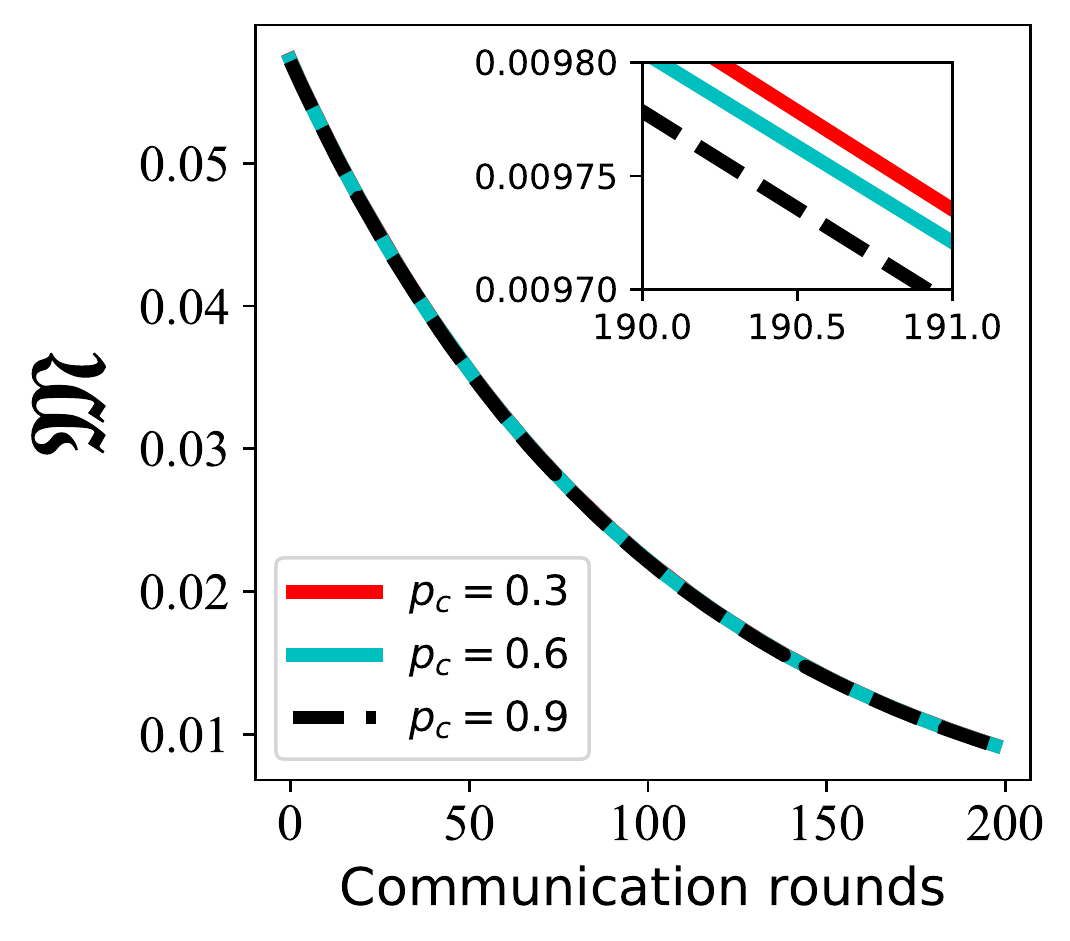} 
		\end{minipage}
		\label{img4a}
	}	     
	\centering
	\subfigure[Topology sparsity comparison on AUC maximization]{
		\begin{minipage}[t]{0.25\linewidth}
			\centering
			\includegraphics[width=1\textwidth]{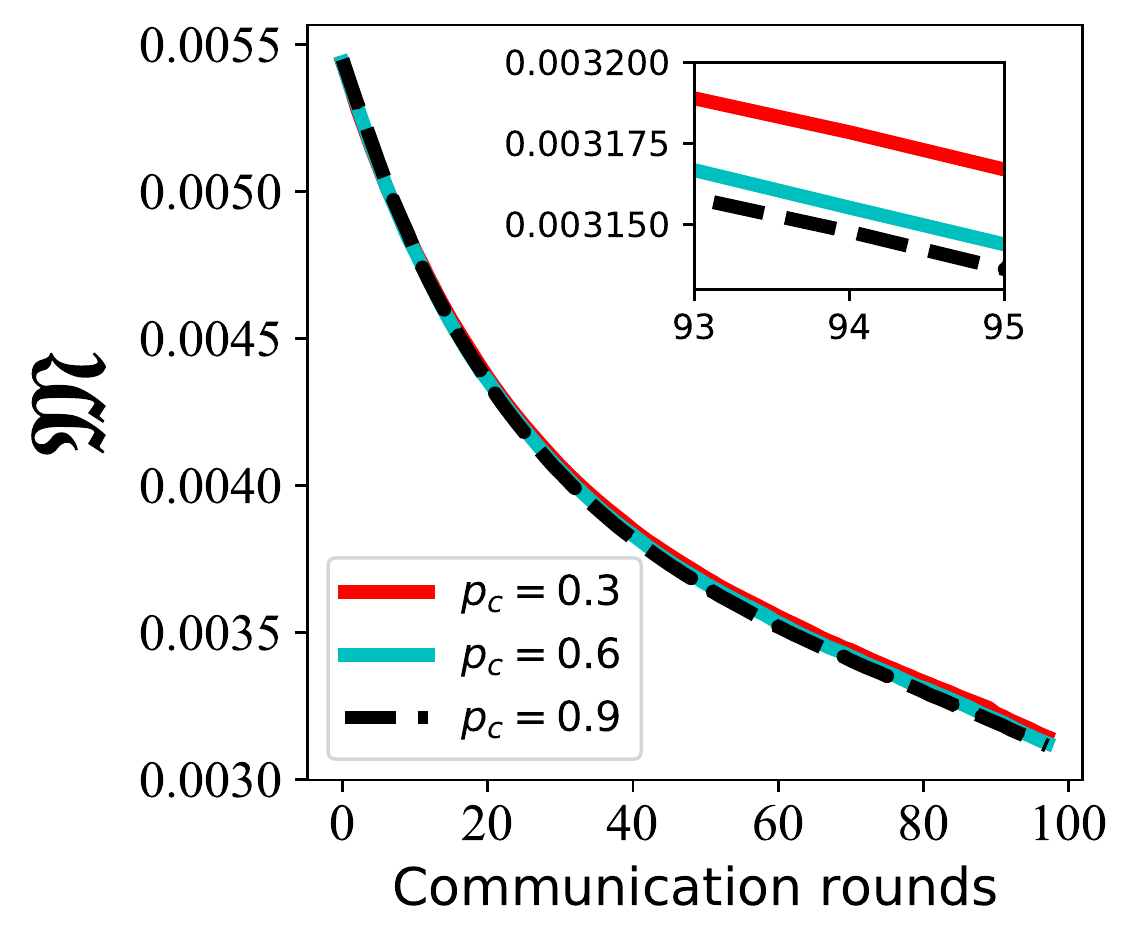}
		\end{minipage}
		\centering
		\label{img4b}
	}	 
	\centering
	\subfigure[Algorithms Comparision with 20nodes.]{
		\begin{minipage}[t]{0.22\linewidth}
			\centering
			\includegraphics[width=1\textwidth]{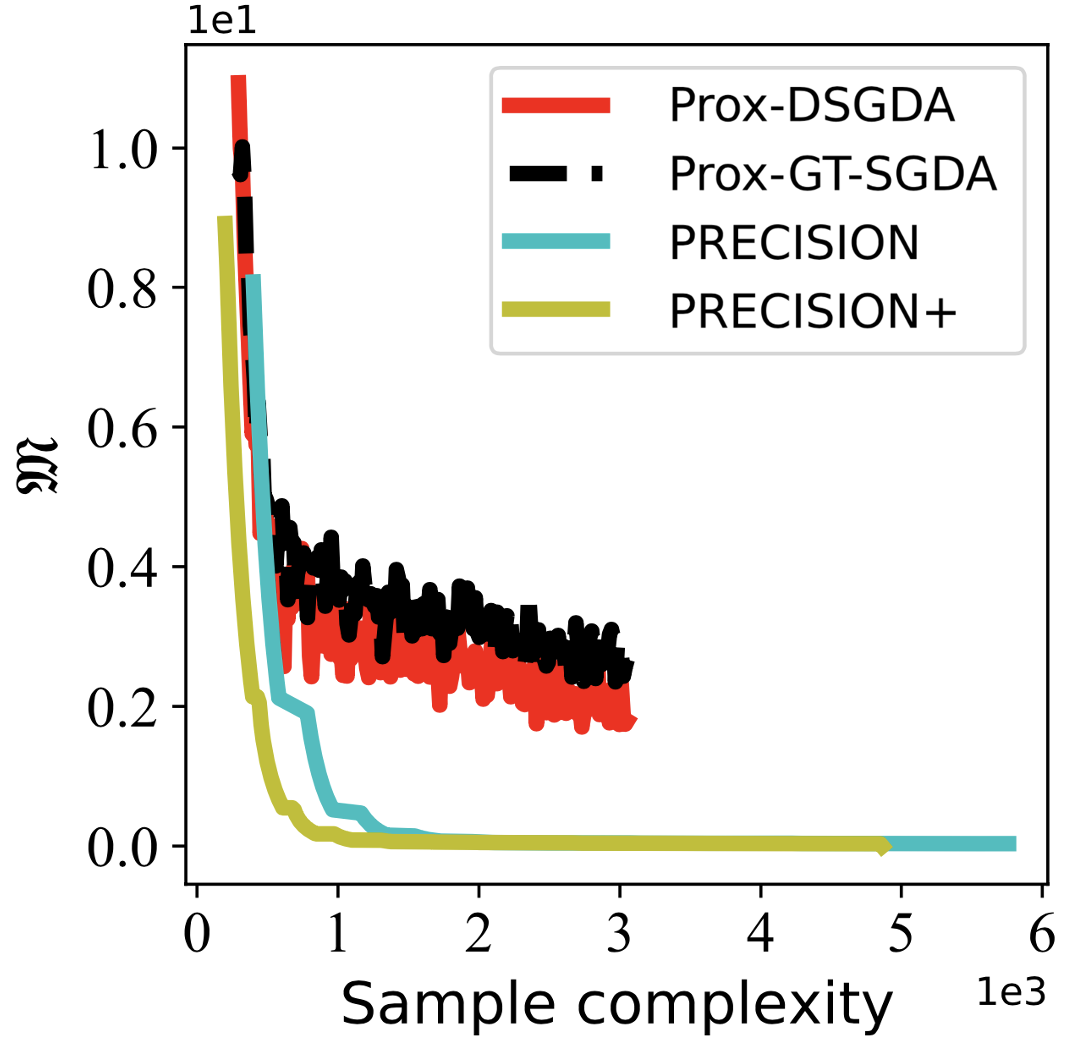}
		\end{minipage}
		\begin{minipage}[t]{0.212\linewidth}
			\centering
			\includegraphics[width=1\textwidth]{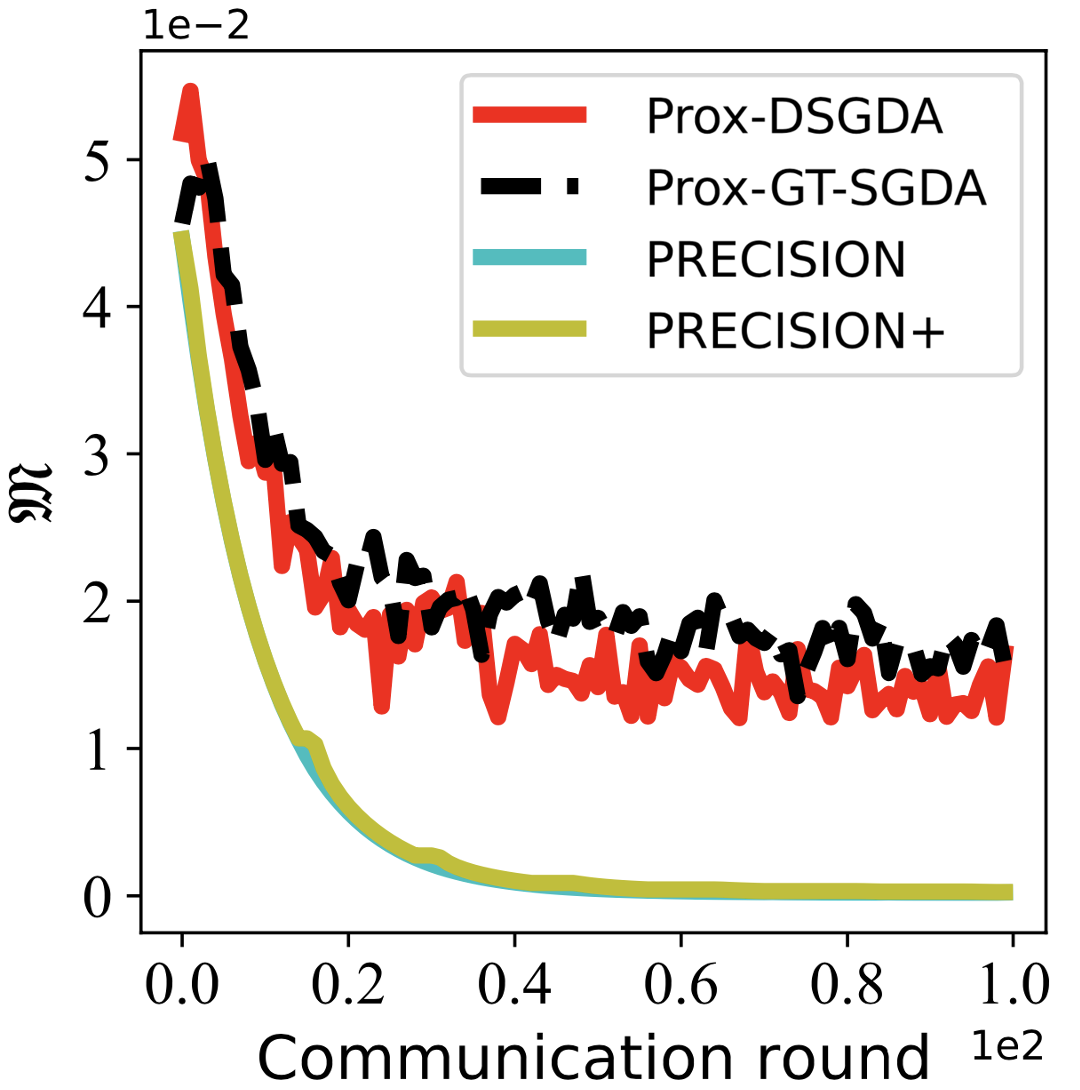}
		\end{minipage}
		\centering
		\label{imgauc_20nodes}
	}	\caption{ Algorithms Comparision.}
	\label{img_auc_nodes}
\end{figure}

\section{Proof of Lemmas}

Before diving in our theoretical analysis, we first define the following notations:
\begin{itemize}
	\item $\xb_t = \frac{1}{m}\sum_{i=1}^{m} \x_{i,t}$ and $\x_t = [\x_{1,t}^{\top},\cdots, \x_{m,t}^{\top}]^{\top}$ for any vector $\x$;
	\item $\nabla_{{{\x}}} F_t = [\nabla_{{{\x}}} F({{\x}}_{1,t}, {{\y}}_{1,t})^\top,\cdots,\nabla_{{{\x}}} F({{\x}}_{m,t}, {{\y}}_{m,t})^\top]^{\top} $;
	\item $\nabla_{{{\y}}} F_t = [\nabla_{{{\y}}} F({{\x}}_{1,t}, {{\y}}_{1,t})^\top,\cdots,\nabla_{{{\y}}} F({{\x}}_{m,t}, {{\y}}_{m,t})^\top]^{\top} $;
	\item $\Ec(\x_t) = \frac{1}{m}\sum_{i=1}^{m} \|\x_{i,t} - \xb_t\|^2$ for any vector $\x$.
\end{itemize}

\subsection{ Proof of Lemma \ref{Desceding J}:}
Our first step is to show the descent property of \alg algorithm on the function $Q(\cdot)$ as shown in Lemma \ref{Desceding J}.


\begin{proof}
Let $J({{\x}}_t) = \max_{{{\y}}} F({{\x}}_t,{{\y}})$. According to the algorithm update, we have:
\begin{align}
J({\bar{{\x}}}_{t+1}) - J({\bar{{\x}}}_{t}) 
&\stackrel{(a)}{\le}
\langle \nabla J({\bar{{\x}}}_t), {\bar{{\x}}}_{t+1} - {\bar{{\x}}}_{t} \rangle + \frac{L_J}{2}\|{\bar{{\x}}}_{t+1} - {\bar{{\x}}}_{t}\|^2 \notag\\
&
\leq\left\langle\nabla  J\left(\bar{{\x}}_{t}\right), {\nu} \left(\frac{1}{m} \sum_{i \in M} \tilde{{\x}}_{i,t}-\bar{{\x}}_{t}\right)\right\rangle+\frac{{\nu}^{2} L_J}{2}\left\|\frac{1}{m} \sum_{i} \tilde{{\x}}_{i,t}-\bar{{\x}}_{t}\right\|^{2} \notag\\
&
\leq {\nu} \frac{1}{m} \sum_{i}\left\langle\nabla  J\left(\bar{{\x}}_{t}\right), \tilde{{\x}}_{i,t}-\bar{{\x}}_{t}\right\rangle+\frac{{\nu}^{2} L_J}{2} \frac{1}{m}\left\|\tilde{{\x}}_{t}-1 \bar{{\x}}_{t}\right\|^{2} \notag\\
&
\leq {\nu} \frac{1}{m} \sum_{i}\left\langle\nabla  J\left(\bar{{\x}}_{t}\right)-\p_{i,t}-\tau\left(\bar{{\x}}_{t}-{\x}_{i,t}\right), \tilde{{\x}}_{i,t}-\bar{{\x}}_{t}\right\rangle+\frac{{\nu}^{2} L_J}{2} \frac{1}{m}\left\|\tilde{{\x}}_{t}-1 \bar{{\x}}_{t}\right\|^{2} \notag\\
&
-\frac{{\nu} \tau}{m}\left\|\tilde{{\x}}_{t}-1 \bar{{\x}}_{t}\right\|^{2}-h\left(\bar{{\x}}_{t+1}\right)+h\left(\bar{{\x}}_{t}\right) \notag\\
&
\leq \frac{{\nu}}{m} \sum_{i}\left\langle\nabla  J\left(\bar{{\x}}_{t}\right)-\p_{i,t}, \tilde{{\x}}_{i,t}-\bar{{\x}}_{t}\right\rangle+\frac{{\nu} \tau}{m} \sum_{i}\left\langle {\x}_{i,t}-\bar{{\x}}_{t}, \tilde{{\x}}_{i,t}-\bar{{\x}}_{t}\right\rangle \notag\\
&
+\frac{{\nu}^{2} L_J}{2 m}\left\|\tilde{{\x}}_{t}-1 \bar{{\x}}_{t}\right\|^{2}-\frac{{\nu} \tau}{m}\left\|\tilde{{\x}}_{t}-1 \bar{{\x}}_{t}\right\|^{2}-h\left(\bar{{\x}}_{t+1}\right)+h\left(\bar{{\x}}_{t}\right),
\end{align}
where (a) is because of Lipschitz continuous gradients of $J$.

\end{proof}

\begin{align}
Q\left(\bar{{\x}}_{t+1}\right) & \leq Q\left(\bar{{\x}}_{t}\right)+\frac{\nu}{m} \sum_{i}\left\langle\nabla J\left(\bar{{\x}}_{t}\right)-\p_{i,t}, \tilde{{\x}}_{i,t}-\bar{{\x}}_{t}\right\rangle+\frac{\nu \tau}{m} \sum_{i}\left\langle {\x}_{i,t}-\bar{{\x}}_{t}, \tilde{{\x}}_{i,t}-\bar{{\x}}_{t}\right\rangle \notag \\
&+\frac{\nu^{2} L_J}{2 m}\left\|\tilde{{\x}}_{t}-1 \bar{{\x}}_{t}\right\|^{2}-\frac{\nu \tau}{m}\left\|\tilde{{\x}}_{t}-1 \bar{{\x}}_{t}\right\|^{2} \notag \\
& \leq Q\left(\bar{{\x}}_{t}\right)+\frac{\nu}{m} \sum_{i}\left\langle\nabla J\left(\bar{{\x}}_{t}\right)-\nabla_{{\x}} F (\bar{\bm{{\x}}_t},\bar{\bm{{\y}}_t}), \tilde{{\x}}_{i,t}-\bar{{\x}}_{t}\right\rangle \notag \\
&+\frac{\nu}{m} \sum_{i}\left\langle\nabla_{{\x}} F (\bar{\bm{{\x}}_t},\bar{\bm{{\y}}_t})-\bar{\p}_{t}, \tilde{{\x}}_{i,t}-\bar{{\x}}_{t}\right\rangle\notag  \\
&+\frac{\nu \tau}{m} \sum_{i}\left\langle {\x}_{i,t}-\bar{{\x}}_{t}, \tilde{{\x}}_{i,t}-\bar{{\x}}_{t}\right\rangle+\frac{\nu^{2} L_J}{2 m}\left\|\tilde{{\x}}_{t}-1 \bar{{\x}}_{t}\right\|^{2}-\frac{\nu \tau}{m}\left\|\tilde{{\x}}_{t}-1 \bar{{\x}}_{t}\right\|^{2}
\notag\\&
\leq Q\left(\bar{{\x}}_{t}\right)+\frac{\nu}{m} \sum_{i} \frac{1}{2 \beta}\left\|\nabla J\left(\bar{{\x}}_{t}\right)-\nabla_{{\x}} F (\bar{\bm{{\x}}_t},\bar{\bm{{\y}}_t})\right\|^{2}+\frac{\nu}{m} \sum_{i} \frac{\beta}{2}\left\|\tilde{{\x}}_{i,t}-\bar{{\x}}_{t}\right\|^{2} \notag\\&
+{\nu}\frac{1}{2 \beta}\left\|\nabla_{{\x}} F (\bar{\bm{{\x}}_t},\bar{\bm{{\y}}_t})-\bar{\p}_{t}\right\|^{2}+\frac{\nu}{m} \sum_{i} \frac{\beta}{2}\left\|\tilde{{\x}}_{i,t}-\bar{{\x}}_{t}\right\|^{2} \notag\\&
+\frac{\nu \tau}{m} \frac{1}{2 \beta} \sum_{i}\left\|\bar{{\x}}_{t}-{\x}_{i,t}\right\|^{2}+\frac{\nu \tau}{m} \sum_{i} \frac{\beta}{2}\left\|\tilde{{\x}}_{i,t}-\bar{{\x}}_{t}\right\|^{2}-\left(\frac{\nu \tau}{m}-\frac{\nu^{2} L_J}{2 m}\right)\left\|\tilde{{\x}}_{t}-1 \bar{{\x}}_{t}\right\|^{2} \notag\\&
\leq Q\left(\bar{{\x}}_{t}\right)+\frac{\nu L_F^2}{2 \beta m} \sum_{i}\left\| {\y}_t^*- \bar{{\y}_t} \right\|^{2}+\frac{\nu}{2 \beta }  \left\|\nabla_{{\x}} F (\bar{\bm{{\x}}_t},\bar{\bm{{\y}}_t})-\bar{\p}_{t}\right\|^{2}\notag\\&
+\frac{\nu \tau}{2 \beta m}\left\|{\x}_t-1 \bar{{\x}}_{t}\right\|^{2}+\left(\frac{\beta \nu}{m}+\frac{\nu \tau \beta}{2 m}\right)\left\|\tilde{{\x}}_{t}-1 \bar{{\x}}_{t}\right\|^{2}-\left(\frac{\nu \tau}{m}-\frac{\nu^{2} L_J}{2 m}\right)\left\|\tilde{{\x}}_{t}-1 \bar{{\x}}_{t}\right\|^{2} \notag\\&
\leq Q\left(\bar{{\x}}_{t}\right)+\frac{\nu L_F^2}{2 \beta } \left\| {\y}_t^*- \bar{{\y}_t} \right\|^{2}+\frac{\nu}{2 \beta } \sum_{i} \left\|\nabla_{{\x}} F (\bar{\bm{{\x}}_t},\bar{\bm{{\y}}_t})-\bar{\p}_{t}\right\|^{2}\notag\\&
+\frac{\nu \tau}{2 \beta m} \left\|{\x}_t-1 \bar{{\x}}_{t}\right\|^{2}
-\left(\frac{\nu \tau}{m}-\frac{\nu^{2} L_J}{2 m}-\frac{\nu \beta}{m}-\frac{\nu \tau \beta}{2 m}\right)\left\|\tilde{{\x}}_{t}-1 \bar{{\x}}_{t}\right\|^{2},
 \end{align}

\subsection{Proof of Lemma \ref{Error Bound on omega}}

Note that in the RHS of Lemma \ref{Desceding J}, there is an error term $\|{{\y}}_t^* - {\bar{{\y}}}_t\|^2$. Here, Lemma \ref{Error Bound on omega} states the contraction property of this error term.
\begin{proof}
Recall that ${{\y}}^*_t = {{\y}}^*({\bar{{\x}}}_t) = \arg \max_{{{\y}}} F({\bar{{\x}}}_t,{{\y}})$. We have:
\begin{align}\label{22}
	\left\|\bar{{\y}}_{t+1}-{{\y}}_{t}^{*}\right\|^{2} &=\left\|\bar{{\y}}_{t}+{\eta}\left(\frac{1}{m}\sum_{i\in M} \widetilde{{\y}}_{i}-\bar{{\y}}_{t}\right)-{{\y}}_{t}^{*}\right\|^{2} \notag\\
	&=\left\|\bar{{\y}}_{t}-{{\y}}_{t}^{*}\right\|^{2}+2 {\eta}\left\langle\bar{{\y}}_{t}-{{\y}}_{t}^{*}, \frac{1}{m}\sum_{i\in M} \widetilde{{\y}}_{i}-\bar{{\y}}_{t}\right\rangle+{\eta}^{2}\left\|\frac{1}{m}\sum_{i\in M} \widetilde{{\y}}_{i}-\bar{{\y}}_{t}\right\|^{2}\notag\\&
	\leq\left\|\bar{{\y}}_{t}-{{\y}}_{t}^{*}\right\|^{2}+2 {\eta}\left\langle\bar{{\y}}_{t}-{{\y}}_{t}^{*}, \frac{1}{m}\sum_{i\in M} \widetilde{{\y}}_{i}-\bar{{\y}}_{t}\right\rangle+{\eta}^{2}\left\|\tilde{{\y}}_{t}-1 \bar{{\y}}_{t}\right\|^{2}.
\end{align}
From the projection operation, we have 
\begin{align}  
\tilde{{\y}}_{i}({{\y}}_{i,t}) \! =& {arg\,min}_{{{\y}}_i \in \mathcal{Y}}  \big\| {\y}_i- \big({\y}_{i,t} + {\alpha} \d_{i,t}\big)\big\|^2.
\end{align}

Due to the optimality condition for the constrained convex optimization, we have 
\begin{align}  
\left\langle\widetilde{{\y}}_{i}-\left({{\y}}_{i,t}+{\alpha}\d_{i,t}\right),{\y}-\widetilde{{\y}}_{i}\right\rangle \geq 0, \quad \forall{\y} \in \mathcal{Y}, i\in M.
\end{align}
Thus, we have
\begin{align} 
\langle -\d_{i,t}+{\alpha}^{-1}\left( \widetilde{{\y}}_{i}-{{\y}}_{i,t}\right),{\y}- \widetilde{{\y}}_{i} \rangle \geq 0, \forall{\y} \in \mathcal{Y}, i\in M. 
\end{align}

Moreover, we have
\begin{align}
F\left(\bar{{\x}}_t,{\y}\right)-F\left(\bar{{\x}}_t, \bar{{\y}}_{t}\right)-\left\langle\nabla_{{\y}} F\left(\bar{{\x}}_t, \bar{{\y}}_{t}\right),{\y}-\bar{{\y}}_{t}\right\rangle \leq-\frac{\mu}{2}\left\|{\y}-\bar{{\y}}_{t}\right\|^{2}
\end{align}
Rearranging the terms in the above inequality, we have
\begin{align}
F\left(\bar{{\x}}_t,{\y}\right)+\frac{\mu}{2}\left\|{\y}-\bar{{\y}}_{t}\right\|^{2} \leq & F\left(\bar{{\x}}_t, \bar{{\y}}_{t}\right)+\left\langle\nabla_{{\y}} F\left(\bar{{\x}}_t, \bar{{\y}}_{t}\right),{\y}-\bar{{\y}}_{t}\right\rangle
\notag\\
 \leq & F\left(\bar{{\x}}_t, \bar{{\y}}_{t}\right)+\frac{1}{{\alpha}}\left\langle\tilde{{\y}}_{t}-1 \bar{{\y}}_{t},{\y}-\frac{1}{m}\sum_{i\in M} \widetilde{{\y}}_{i}\right\rangle+\left\langle\nabla_{{\y}} F\left(\bar{{\x}}_t, \bar{{\y}}_{t}\right)-{\bd_t},{\y}-\frac{1}{m}\sum_{i\in M} \widetilde{{\y}}_{i}\right\rangle \notag \\
&+\left\langle\nabla_{{\y}} F\left(\bar{{\x}}_t, \bar{{\y}}_{t}\right), \tilde{{\y}}_{t}-1 \bar{{\y}}_{t}\right\rangle-\frac{1}{2 {\alpha}}\left\|\tilde{{\y}}_{t}-1 \bar{{\y}}_{t}\right\|^{2}+\frac{1}{2 {\alpha}}\left\|\tilde{{\y}}_{t}-1 \bar{{\y}}_{t}\right\|^{2}.
\end{align}

Since $F({\x},{\y})$ is gradient Lipschitz and due to the condition in this lemma
$$
{\alpha} \leq \frac{1}{2 L_{F}} \leq \frac{1}{L_{F}},
$$
we have

\begin{align}
-\frac{1}{2 {\alpha}}\left\|\tilde{{\y}}_{t}-1 \bar{{\y}}_{t}\right\|^{2} & \leq-\frac{L_{F}}{2}\left\|\tilde{{\y}}_{t}-1 \bar{{\y}}_{t}\right\|^{2}  \notag\\
& \leq F\left({\x}_t, \tilde{{\y}}_{t}\right)-F\left(\bar{{\x}}_t, \bar{{\y}}_{t}\right)-\left\langle\nabla_{{\y}} F\left(\bar{{\x}}_t, \bar{{\y}}_{t}\right), \tilde{{\y}}_{t}-1 \bar{{\y}}_{t}\right\rangle.
\end{align}

\begin{align}
F\left(\bar{{\x}}_t,{\y}\right)+\frac{\mu}{2}\left\|{\y}-\bar{{\y}}_{t}\right\|^{2} \leq & F\left({\x}_t, \tilde{{\y}}_{t}\right)+\frac{1}{{\alpha}}\left\langle\tilde{{\y}}_{t}-1 \bar{{\y}}_{t},{\y}-\tilde{{\y}}_{t}\right\rangle \notag \\
&+\left\langle\nabla_{{\y}} F\left(\bar{{\x}}_t, \bar{{\y}}_{t}\right)-{\bd_t},{\y}-\tilde{{\y}}_{t}\right\rangle+\frac{1}{2 {\alpha}}\left\|\tilde{{\y}}_{t}-1 \bar{{\y}}_{t}\right\|^{2}.
\end{align}
Note that in the last inequality, we have
\begin{align}
&\frac{1}{{\alpha}}\left\langle\tilde{{\y}}_{t}-1 \bar{{\y}}_{t},{\y}-\tilde{{\y}}_{t}\right\rangle+\frac{1}{2 {\alpha}}\left\|\tilde{{\y}}_{t}-1 \bar{{\y}}_{t}\right\|^{2} \notag \\
&=\frac{1}{{\alpha}}\left\langle\tilde{{\y}}_{t}-1 \bar{{\y}}_{t}, \bar{{\y}}_{t}-\tilde{{\y}}_{t}\right\rangle+\frac{1}{{\alpha}}\left\langle\tilde{{\y}}_{t}-1 \bar{{\y}}_{t},{\y}-\bar{{\y}}_{t}\right\rangle+\frac{1}{2 {\alpha}}\left\|\tilde{{\y}}_{t}-1 \bar{{\y}}_{t}\right\|^{2}  \notag\\
\quad&=\frac{1}{{\alpha}}\left\langle\tilde{{\y}}_{t}-1 \bar{{\y}}_{t},{\y}-\bar{{\y}}_{t}\right\rangle-\frac{1}{2 {\alpha}}\left\|\tilde{{\y}}_{t}-1 \bar{{\y}}_{t}\right\|^{2},
\end{align}
which thus leads to
\begin{align}
F\left(\bar{{\x}}_t,{\y}\right)+\frac{\mu}{2}\left\|{\y}-\bar{{\y}}_{t}\right\|^{2} &\leq F\left(\bar{{\x}}_t, \tilde{{\y}}_{t}\right)+\frac{1}{{\alpha}}\left\langle\tilde{{\y}}_{t}-1 \bar{{\y}}_{t},{\y}-\bar{{\y}}_{t}\right\rangle  \notag\\&
+\left\langle\nabla_{{\y}} F\left(\bar{{\x}}_t, \bar{{\y}}_{t}\right)-{\bd_t},{\y}-\tilde{{\y}}_{t}\right\rangle-\frac{1}{2 {\alpha}}\left\|\tilde{{\y}}_{t}-1 \bar{{\y}}_{t}\right\|^{2}.
\end{align}
We let ${\y}={{\y}}_{t}^{*}$ and obtain
\begin{align}
F\left(\bar{{\x}}_t,{\y}_t^*\right)+\frac{\mu}{2}\left\|{{\y}}_{t}^{*}-\bar{{\y}}_{t}\right\|^{2} \leq & F\left(\bar{{\x}}_t, \tilde{{\y}}_{t}\right)+\frac{1}{{\alpha}}\left\langle\tilde{{\y}}_{t}-1 \bar{{\y}}_{t}, {{\y}}_{t}^{*}-\bar{{\y}}_{t}\right\rangle \notag \\
&+\left\langle\nabla_{{\y}} F\left(\bar{{\x}}_t, \bar{{\y}}_{t}\right)-{\bd_t}, {{\y}}_{t}^{*}-\tilde{{\y}}_{t}\right\rangle-\frac{1}{2 {\alpha}}\left\|\tilde{{\y}}_{t}-1 \bar{{\y}}_{t}\right\|^{2},
\end{align}
which further yields

\begin{align}
\frac{\mu}{2} \| {{\y}}_{t}^{*} &-\bar{{\y}}_{t}\left\|^{2}+\frac{1}{2 {\alpha}}\right\| \tilde{{\y}}_{t}-1 \bar{{\y}}_{t} \|^{2} \notag \\
& \leq \frac{1}{{\alpha}}\left\langle\tilde{{\y}}_{t}-1 \bar{{\y}}_{t}, {{\y}}_{t}^{*}-\bar{{\y}}_{t}\right\rangle+\left\langle\nabla_{{\y}} F\left(\bar{{\x}}_t, \bar{{\y}}_{t}\right)-{\bd_t}, {{\y}}_{t}^{*}-\tilde{{\y}}_{t}\right\rangle.
\end{align}

$F\left(\bar{{\x}}_t, {{\y}}_{t}^{*}\right) \geq F\left(\bar{{\x}}_t, \tilde{{\y}}_{t}\right)$ is due to strong concavity and ${{\y}}_{t}^{*}=\operatorname{argmax}_{{\y} \in \mathcal{Y}} F\left(\bar{{\x}}_t,{\y}\right) .$ In addition,
for the last term of the above inequality, we further bound it as follows
\begin{align}
&\left\langle\nabla_{{\y}} F\left(\bar{{\x}}_t, \bar{{\y}}_{t}\right)-{\bd_t}, {{\y}}_{t}^{*}-\tilde{{\y}}_{t}\right\rangle  \notag\\
&\leq\frac{2}{\mu}\left\|\nabla_{{\y}} F\left(\bar{{\x}}_t, \bar{{\y}}_{t}\right)-{\bd_t}\right\|^{2}+\frac{\mu}{4}\left\|{{\y}}_{t}^{*}-\bar{{\y}}_{t}\right\|^{2}+\frac{\mu}{4}\left\|\tilde{{\y}}_{t}-1 \bar{{\y}}_{t}\right\|^{2}.
\end{align}

Then, we have
\begin{align} 
&2 {\eta}\left\langle\tilde{{\y}}_{t}-1 \bar{{\y}}_{t}, \bar{{\y}}_{t}-{{\y}}_{t}^{*}\right\rangle& \notag\\
&\quad \leq-\frac{{\eta}\alpha \mu}{2}\left\|\bar{{\y}}_{t}-{{\y}}_{t}^{*}\right\|^{2}-\frac{2 {\alpha}-{\eta}\alpha \mu}{2}\left\|\tilde{{\y}}_{t}-1 \bar{{\y}}_{t}\right\|^{2}+\frac{4 {\eta}\alpha}{\mu}\left\|\nabla_{{\y}} F\left(\bar{{\x}}_t, \bar{{\y}}_{t}\right)-{\bd_t}\right\|^{2},
 \end{align}

which gives the upper bound of the second term on the right-hand side of \ref{22}. Then, we have
\begin{align}
\left\|\bar{{\y}}_{t+1}-{{\y}}_{t}^{*}\right\|^{2} \leq \frac{2-{\eta}\alpha \mu}{2}\left\|\bar{{\y}}_{t}-{{\y}}_{t}^{*}\right\|^{2}-\frac{2 {\eta}-{\eta}\alpha \mu-2 {\eta}^{2}}{2}\left\|\tilde{{\y}}_{t}-1 \bar{{\y}}_{t}\right\|^{2}+\frac{4 {\eta}\alpha}{\mu}\left\|\nabla_{{\y}} F\left(\bar{{\x}}_t, \bar{{\y}}_{t}\right)-{\bd_t}\right\|^{2}..
 \end{align}
Thus, according to the condition of this lemma that ${\eta} \leq 1 / 8$ and ${\alpha} \leq\left(4 L_{F}\right)^{-1} \leq(4 \mu)^{-1}$ by the fact $L_{F} \geq \mu>0$, we have
\begin{align}
-\frac{2 {\eta}-{\eta}\alpha \mu-2 {\eta}^2}{2} \leq-\frac{3 {\eta}}{4},
 \end{align}
which eventually leads to
\begin{align}
\left\|\bar{{\y}}_{t+1}-{{\y}}_{t}^{*}\right\|^{2} \leq\left(1-\frac{{\eta}\alpha \mu}{2}\right)\left\|\bar{{\y}}_{t}-{{\y}}_{t}^{*}\right\|^{2}-\frac{3 {\eta}}{4}\left\|\tilde{{\y}}_{t}-1 \bar{{\y}}_{t}\right\|^{2}+\frac{4 {\eta}\alpha}{\mu}\left\|\nabla_{{\y}} F\left(\bar{{\x}}_t, \bar{{\y}}_{t}\right)-{\bd_t}\right\|^{2}.
 \end{align}

Denoting ${\y}^{*}\left({\x}_t\right)$ and ${\y}^{*}\left({\x}_{t+1}\right)$ as ${{\y}}_{t}^{*}$ and ${{\y}}_{t+1}^{*}$ for abbreviation, we start the proof by decomposing the term $\left\|\bar{{\y}}_{t+1}-{{\y}}_{t+1}^{*}\right\|^{2}$ as follows
 \begin{align}
 \left\|\bar{{\y}}_{t+1}-{{\y}}_{t+1}^{*}\right\|^{2} &=\left\|\bar{{\y}}_{t+1}-{{\y}}_{t}^{*}+{{\y}}_{t}^{*}-{{\y}}_{t+1}^{*}\right\|^{2} \notag \\
 & \leq\left(1+\frac{\mu {\eta\alpha}}{4}\right)\left\|\bar{{\y}}_{t+1}-{{\y}}_{t}^{*}\right\|^{2}+\left(1+\frac{4}{\mu {\eta\alpha}}\right)\left\|{{\y}}_{t}^{*}-{{\y}}_{t+1}^{*}\right\|^{2} \notag \\
 & \leq\left(1+\frac{\mu {\eta\alpha}}{4}\right)\left\|\bar{{\y}}_{t+1}-{{\y}}_{t}^{*}\right\|^{2}+\left(1+\frac{4}{\mu {\eta\alpha}}\right) L_{{\y}}^{2}\left\|\bar{{\x}}_{t+1}-\bar{{\x}}_t\right\|^{2}.
 \end{align}
 
Next, plugging the updating rule $\bar{{\x}}_{t+1}=\bar{{\x}}_t+\nu \left(\frac{1}{m} \sum_i \tilde{{\x}}_{i,t}-\bar{{\x}}_t\right)$ into the above inequality, we obtain
  \begin{align}
 \left\|\bar{{\y}}_{t+1}-{{\y}}_{t+1}^{*}\right\|^{2} \leq\left(1+\frac{\mu {\eta\alpha}}{4}\right)\left\|\bar{{\y}}_{t+1}-{{\y}}_{t}^{*}\right\|^{2}+\left(1+\frac{4}{\mu {\eta\alpha}}\right) L_{{\y}}^{2} \nu^2 \left\|\frac{1}{m} \sum_i \tilde{{\x}}_{i,t}-\bar{{\x}}_t\right\|^{2}.
  \end{align}
 Furthermore, we have
 \begin{align}
& \left\|\bar{{\y}}_{t+1}-{{\y}}_{t}^{*}\right\|^{2} \notag\\&
 \quad \leq\left(1-\frac{{\eta\alpha} \mu}{2}\right)\left\|\bar{{\y}}_{t}-{{\y}}_{t}^{*}\right\|^{2}-\frac{3 {\eta}}{4}\left\|\tilde{{\y}}_{t}-1 \bar{{\y}}_{t}\right\|^{2}+\frac{4 {\eta\alpha}}{\mu}\left\|\nabla_{{\y}} F\left(\bar{{\x}}_t, \bar{{\y}}_{t}\right)-{\bd_t}\right\|^{2}.
 \end{align}
 According to the conditions $0<{\alpha} \leq\left(4 L_{F}\right)^{-1}, 0<{\eta} \leq 1 / 8$ and due to $L_{F} \geq \mu>0$, we have
  \begin{align}
 {\alpha} \leq \frac{1}{4 L_{F}} \leq \frac{1}{4 \mu}, \quad \text { and } \quad {\eta}\alpha \leq \frac{1}{32 \mu},
  \end{align}
 
which yield
\begin{align}
\left(1+\frac{\mu {\eta\alpha}}{4}\right)\left(1-\frac{\mu {\eta\alpha}}{2}\right)=1-\frac{\mu {\eta\alpha}}{2}+\frac{\mu {\eta\alpha}}{4}-\frac{\mu^{2} {\eta}^{2} {\alpha}^{2}}{4} \leq 1-\frac{\mu {\eta\alpha}}{4} \\
-\left(1+\frac{\mu {\eta\alpha}}{4}\right) \frac{3 {\eta}}{4} \leq-\frac{3 {\eta}}{4}, \quad \frac{4 {\eta\alpha}}{\mu}\left(1+\frac{\mu {\eta\alpha}}{4}\right)=\frac{4 {\eta\alpha}}{\mu}+{\eta}^{2} {\alpha}^{2}<\frac{75 {\eta}\alpha}{16 \mu} \\
\text { and }\left(1+\frac{4}{\mu {\eta\alpha}}\right) L_{{\y}}^{2} \nu^{2} \leq \frac{129}{32} \frac{L_{{\y}}^{2} \nu }{\mu {\eta\alpha}}<\frac{17L_{{\y}}^{2} \nu^2 }{2\mu {\eta\alpha}}
\end{align}
We eventually obtain

\begin{align}
\left\|\bar{{\y}}_{t+1}-{{\y}}_{t+1}^{*}\right\|^{2} \leq &\left(1-\frac{\mu {\eta\alpha}}{4}\right)\left\|\bar{{\y}}_{t}-{{\y}}_{t}^{*}\right\|^{2}-\frac{3 {\eta}}{4}\left\|\tilde{{\y}}_{t}-1 \bar{{\y}}_{t}\right\|^{2}\notag \\
&+\frac{75 {\eta\alpha}}{16 \mu}\left\|{\bd_t}-\nabla_{{\y}} F\left(\bar{{\x}}_t, \bar{{\y}}_{t}\right)\right\|^{2}+\frac{17 L_{{\y}}^{2}\nu^2 }{2\mu {\eta\alpha}m}\left\|\tilde{{\x}}_{t}-1 \bar{{\x}}_{t}\right\|^{2}.
\end{align}

which completes the proof.
\end{proof}

\subsection{ Proof of Lemma \ref{Lem: Descending_Q}}
Next, by combining the results from Lemmas~\ref{Desceding J}-\ref{Error Bound on omega}, we have the descent result shown in Lemma \ref{Lem: Descending_Q}.

\begin{proof}
From Lemmas~\ref{Desceding J}-\ref{Error Bound on omega}, we have
\begin{align}
&
Q({\bar{{\x}}}_{t+1})  - Q({\bar{{\x}}}_{t}) + \frac{4 \nu L_F^2}{ \beta \mu{\eta\alpha} }\big[\|{\bar{{\y}}}_{t+1} - {{\y}}_{t+1}^*\|^2  - \|{{\y}}_t^* - {\bar{{\y}}}_t\|^2\big] \notag\\
\le
&
 \frac{4 \nu L_F^2}{ \beta \mu{\eta\alpha} } \big[  \left(-\frac{\mu {\eta\alpha}}{4}\right)\left\|\bar{{\y}}_{t}-{{\y}}_{t}^{*}\right\|^{2}-\frac{3 {\eta}}{4}\left\|\tilde{{\y}}_{t}-1 \bar{{\y}}_{t}\right\|^{2}+\frac{75 {\eta\alpha}}{16 \mu}\left\|{\bd_t}-\nabla_{{\y}} F\left(\bar{{\x}}_t, \bar{{\y}}_{t}\right)\right\|^{2}\notag \\
 &+\frac{17L_{{\y}}^{2}\nu^2 }{2 \mu m {\eta\alpha}}\left\|\tilde{{\x}}_{t}-1 \bar{{\x}}_{t}\right\|^{2}  \big] +\frac{\nu L_F^2}{2\beta}\left\|\bar{{\y}}_{t}-{{\y}}_{t}^{*}\right\|^{2}+\frac{\nu}{2 \beta }  \left\|\nabla_{{\x}} F (\bar{\bm{{\x}}_t},\bar{\bm{{\y}}_t})-\bar{\p}_{t}\right\|^{2}\notag\\&
+\frac{\nu \tau}{2 \beta m} \left\|{\x}_t-1 \bar{{\x}}_{t}\right\|^{2}
-\left(\frac{\nu \tau}{m}-\frac{\nu^{2} L_J}{2 m}-\frac{\nu \beta}{m}-\frac{\nu \tau \beta}{2 m}\right)\left\|\tilde{{\x}}_{t}-1 \bar{{\x}}_{t}\right\|^{2}\notag\\
=
&
\frac{4 \nu L_F^2}{ \beta \mu{\eta\alpha} } \big[  -\frac{3 {\eta}}{4}\left\|\tilde{{\y}}_{t}-1 \bar{{\y}}_{t}\right\|^{2}+\frac{75 {\eta\alpha}}{16 \mu}\left\|{\bd_t}-\nabla_{{\y}} F\left(\bar{{\x}}_t, \bar{{\y}}_{t}\right)\right\|^{2}\notag \\
&+\frac{17L_{{\y}}^{2}\nu^2 }{2 \mu m {\eta\alpha}}\left\|\tilde{{\x}}_{t}-1 \bar{{\x}}_{t}\right\|^{2}  \big] -\frac{\nu L_F^2}{2\beta}\left\|\bar{{\y}}_{t}-{{\y}}_{t}^{*}\right\|^{2}+\frac{\nu}{2 \beta }  \left\|\nabla_{{\x}} F (\bar{\bm{{\x}}_t},\bar{\bm{{\y}}_t})-\bar{\p}_{t}\right\|^{2}\notag\\&
+\frac{\nu \tau}{2 \beta m} \left\|{\x}_t-1 \bar{{\x}}_{t}\right\|^{2}
-\left(\frac{\nu \tau}{m}-\frac{\nu^{2} L_J}{2 m}-\frac{\nu \beta}{m}-\frac{\nu \tau \beta}{2 m}\right)\left\|\tilde{{\x}}_{t}-1 \bar{{\x}}_{t}\right\|^{2}.
\end{align}

Note that
\begin{align}
&
\|\nabla_{{{\x}}} F({\bar{{\x}}}_t,{\bar{{\y}}}_t) - \bp_{t}\|^2 \notag\\
=
&
\|\nabla_{{{\x}}} F({\bar{{\x}}}_t,{\bar{{\y}}}_t) - \frac{1}{m}\sum_{i=1}^{m}\nabla_{{{\x}}} F_i({{\x}}_{i,t},{{\y}}_{i,t}) + \frac{1}{m}\sum_{i=1}^{m}\nabla_{{{\x}}} F_i({{\x}}_{i,t},{{\y}}_{i,t}) - \bp_{t}\|^2\notag\\
\le
&
2\|\nabla_{{{\x}}} F({\bar{{\x}}}_t,{\bar{{\y}}}_t) - \frac{1}{m}\sum_{i=1}^{m}\nabla_{{{\x}}} F_i({{\x}}_{i,t},{{\y}}_{i,t})\|^2 + 2\|\frac{1}{m}\sum_{i=1}^{m}\nabla_{{{\x}}} F_i({{\x}}_{i,t},{{\y}}_{i,t}) - \bp_{t}\|^2\notag\\
\le
&
\frac{2}{m}\sum_{i=1}^{m}\|\nabla_{{{\x}}} F({\bar{{\x}}}_t,{\bar{{\y}}}_t) - \nabla_{{{\x}}} F_i({{\x}}_{i,t},{{\y}}_{i,t})\|^2 + 2\|\frac{1}{m}\sum_{i=1}^{m}\nabla_{{{\x}}} F_i({{\x}}_{i,t},{{\y}}_{i,t}) - \bp_{t}\|^2\notag\\
\le
&
\frac{2L_F^2}{m}\sum_{i=1}^{m}[\|{\bar{{\x}}}_t - {{\x}}_{i,t}\|^2 + \|{\bar{{\y}}}_t - {{\y}}_{i,t}\|^2 ] + 2\|\frac{1}{m}\sum_{i=1}^{m}\nabla_{{{\x}}} F_i({{\x}}_{i,t},{{\y}}_{i,t}) - \bp_{t}\|^2.
\end{align}
Similarly, we have:
\begin{align}
\|\nabla_{{{\y}}} F({\bar{{\x}}}_t,{\bar{{\y}}}_t) - \bd_{t}\|^2
\le
\frac{2L_F^2}{m}\sum_{i=1}^{m}[\|{\bar{{\x}}}_t - {{\x}}_{i,t}\|^2 + &\|{\bar{{\y}}}_t - {{\y}}_{i,t}\|^2 ] \notag\\
&+ 2\|\frac{1}{m}\sum_{i=1}^{m}\nabla_{{{\y}}} F_i({{\x}}_{i,t},{{\y}}_{i,t}) - \bd_{t}\|^2.
\end{align}

Thus, we have 
\begin{align}
&Q({\bar{{\x}}}_{t+1})  - Q({\bar{{\x}}}_{t}) +  \frac{4 \nu L_F^2}{ \beta \mu{\eta\alpha} }\big[\|{\bar{{\y}}}_{t+1} - {{\y}}_{t+1}^*\|^2  - \|{{\y}}_t^* - {\bar{{\y}}}_t\|^2\big] \notag\\
\le
&
 \frac{4 \nu L_F^2}{ \beta \mu{\eta\alpha} } \big\{  -\frac{3 {\eta}}{4}\left\|\tilde{{\y}}_{t}-1 \bar{{\y}}_{t}\right\|^{2}+\frac{75 {\eta\alpha}}{16 \mu}\big[ \frac{2L_F^2}{m}\sum_{i=1}^{m}(\|{\bar{{\x}}}_t - {{\x}}_{i,t}\|^2 + \|{\bar{{\y}}}_t - {{\y}}_{i,t}\|^2 )+ 2\|\frac{1}{m}\sum_{i=1}^{m}\nabla_{{{\y}}} F_i({{\x}}_{i,t},{{\y}}_{i,t}) - \bd_{t}\|^2  \big]  \notag\\&+\frac{17L_{{\y}}^{2}\nu^2 }{2 \mu m {\eta\alpha}}\left\|\tilde{{\x}}_{t}-1 \bar{{\x}}_{t}\right\|^{2}  \big\}+\frac{\nu}{2 \beta }  \big\{ \frac{2L_F^2}{m}\sum_{i=1}^{m}[\|{\bar{{\x}}}_t - {{\x}}_{i,t}\|^2 + \|{\bar{{\y}}}_t - {{\y}}_{i,t}\|^2 ] + 2\|\frac{1}{m}\sum_{i=1}^{m}\nabla_{{{\x}}} F_i({{\x}}_{i,t},{{\y}}_{i,t}) - \bp_{t}\|^2   \big\}\notag\\&
+\frac{\nu \tau}{2 \beta m} \left\|{\x}_t-1 \bar{{\x}}_{t}\right\|^{2}
-\left(\frac{\nu \tau}{m}-\frac{\nu^{2} L_J}{2 m}-\frac{\nu \beta}{m}-\frac{\nu \tau \beta}{2 m}\right)\left\|\tilde{{\x}}_{t}-1 \bar{{\x}}_{t}\right\|^{2}  -\frac{\nu L_F^2}{2\beta}\left\|\bar{{\y}}_{t}-{{\y}}_{t}^{*}\right\|^{2}\notag\\
\stackrel{(a)}{\le}
&
 \frac{4 \nu L_F^2}{ \beta \mu{\eta\alpha} } \big\{  -\frac{3 {\eta}}{4}\left\|\tilde{{\y}}_{t}-1 \bar{{\y}}_{t}\right\|^{2}+\frac{75 {\eta\alpha}}{16 \mu}\big[ \frac{2L_F^2}{m}\sum_{i=1}^{m}(\|{\bar{{\x}}}_t - {{\x}}_{i,t}\|^2 + \|{\bar{{\y}}}_t - {{\y}}_{i,t}\|^2 )+ \frac{2}{m}\|\nabla_{{{\y}}} F({{\x}}_{t},{{\y}}_{t}) - \bd_{t}\|^2  \big]  \notag\\&+\frac{17L_{{\y}}^{2}\nu^2 }{2 \mu m {\eta\alpha}}\left\|\tilde{{\x}}_{t}-1 \bar{{\x}}_{t}\right\|^{2}  \big\}+\frac{\nu}{2 \beta }  \big\{ \frac{2L_F^2}{m}\sum_{i=1}^{m}[\|{\bar{{\x}}}_t - {{\x}}_{i,t}\|^2 + \|{\bar{{\y}}}_t - {{\y}}_{i,t}\|^2 ] + \frac{2}{m}\|\nabla_{{{\x}}} F({{\x}}_{t},{{\y}}_{t}) - \bp_{t}\|^2   \big\}\notag\\&
+\frac{\nu \tau}{2 \beta m} \left\|{\x}_t-1 \bar{{\x}}_{t}\right\|^{2}
-\left(\frac{\nu \tau}{m}-\frac{\nu^{2} L_J}{2 m}-\frac{\nu \beta}{m}-\frac{\nu \tau \beta}{2 m}\right)\left\|\tilde{{\x}}_{t}-1 \bar{{\x}}_{t}\right\|^{2} -\frac{\nu L_F^2}{2\beta}\left\|\bar{{\y}}_{t}-{{\y}}_{t}^{*}\right\|^{2}\notag\\
=
&
 \frac{4 \nu L_F^2}{ \beta \mu{\eta\alpha} } \big\{  -\frac{3 {\eta}}{4}\left\|\tilde{{\y}}_{t}-1 \bar{{\y}}_{t}\right\|^{2} + \frac{75 {\eta\alpha}}{16 \mu}\frac{2}{m}\|\nabla_{{{\y}}} F({{\x}}_{t},{{\y}}_{t}) - \bd_{t}\|^2   \notag\\&+\frac{17L_{{\y}}^{2}\nu^2 }{2 \mu m {\eta\alpha}}\left\|\tilde{{\x}}_{t}-1 \bar{{\x}}_{t}\right\|^{2}  \big\} +\frac{\nu}{2 \beta } \frac{2}{m}\|\nabla_{{{\x}}} F({{\x}}_{t},{{\y}}_{t}) - \bp_{t}\|^2   \notag\\&
+\frac{\nu \tau}{2 \beta m} \left\|{\x}_t-1 \bar{{\x}}_{t}\right\|^{2}
-\left(\frac{\nu \tau}{m}-\frac{\nu^{2} L_J}{2 m}-\frac{\nu \beta}{m}-\frac{\nu \tau \beta}{2 m}\right)\left\|\tilde{{\x}}_{t}-1 \bar{{\x}}_{t}\right\|^{2}\notag\\
&+\big[   \frac{\nu}{ \beta }  \frac{ L_F^2}{m}+ \frac{4\nu L_F^2}{\beta \mu{\eta\alpha}} \frac{75 {\eta\alpha}}{16 \mu}\frac{2L_F^2}{m}   \big]  \sum_{i=1}^{m}[\|{\bar{{\x}}}_t - {{\x}}_{i,t}\|^2 + \|{\bar{{\y}}}_t - {{\y}}_{i,t}\|^2 ] -\frac{\nu L_F^2}{2\beta}\left\|\bar{{\y}}_{t}-{{\y}}_{t}^{*}\right\|^{2},
\end{align}
where (a) due to $\|\frac{1}{m} \sum_{i=1}^{m} \x_{i,t} -\bx_t\|^2 \le \frac{1}{m} \sum_{i=1}^{m}\| \x_{i,t} -\bx_t\|^2$.

Telescoping the above inequality, we have the stated result.
\end{proof}

\subsection{Proof of Lemma \ref{Lem: Contraction}}
Next, we prove the contraction of iterations in the following lemma, which is useful in analyzing the decentralized gradient tracking algorithms.

\begin{proof}
First for the iterates ${{\x}}_t$, we have the following contraction:
\begin{align}
\|\Mt{{\x}}_{t} -1 {\bar{{\x}}}_{t} \|^2 = \|\Mt({{\x}}_{t} -1 {\bar{{\x}}}_{t}) \|^2 \le \lambda^2\|{{\x}}_{t} -1 {\bar{{\x}}}_{t}\|^2.
\end{align}
This is because ${{\x}}_{t} -1 {{\x}}_{t}$ is orthogonal to $\1,$ which is the eigenvector corresponding to the largest eigenvalue of $\Mt,$ and $\lambda = \max\{|\lambda_2|,|\lambda_m|\}.$
Hence,
\begin{align}
&
\|{{\x}}_t-1 {\bar{{\x}}}_t\|^2
=
\|\Mt{{\x}}_{t-1} +\nu(\tilde{{\x}}_{t-1}-{\x}_{t-1})-1 [{\bar{{\x}}}_{t-1} +\nu(\frac{1}{m}\sum_i \tilde{{\x}_i }- {\x}_{t-1})]\|^2 \notag\\
&
\le
(1+c_1)\lambda^2\|{{\x}}_{t-1} -1 {\bar{{\x}}}_{t-1} \|^2 + (1+\frac{1}{c_1}) \nu^2\|\tilde{{\x}}_{t-1}- {{\x}}_{t-1}\|^2.  
\end{align}
For ${{\y}}_t$, we have
\begin{align}
\|{{\y}}_t-1 {\bar{{\y}}}_t\|^2 & \le (1+c_2)\lambda^2\|{{\y}}_{t-1} -1 {\bar{{\y}}}_{t-1} \|^2 + (1+\frac{1}{c_2}) {\eta}^2\|\tilde{{\y}}_{t-1}- {\bar{{\y}}}_{t-1}\|^2.
\end{align}

According to the update, we have 
\begin{align}
&
\|{{\x}}_t-{{\x}}_{t-1}\|^2 
= 
\|\Mt{{\x}}_{t-1} +\nu(\tilde{{\x}}_{t-1}-{\x}_{t-1})- {{\x}}_{t-1}\|^2 \notag\\
=
&
\|(\Mt -\I){{\x}}_{t-1} +\nu(\tilde{{\x}}_{t-1}-{\x}_{t-1}) \|^2
\le 2\|(\Mt -\I){{\x}}_{t-1} \|^2 + 2\nu^2 \|\tilde{{\x}}_{t-1}-{\x}_{t-1}\|^2 \notag\\
=
&
2\|(\Mt -\I)({{\x}}_{t-1} - 1 {\bar{{\x}}}_{t-1}) \|^2 + 2\nu^2 \|\tilde{{\x}}_{t-1}-{\x}_{t-1}\|^2 \notag\\
\stackrel{}{\le}
&
8\|({{\x}}_{t-1} - 1 {\bar{{\x}}}_{t-1}) \|^2 + 2\nu^2 \|\tilde{{\x}}_{t-1}-{\x}_{t-1}\|^2 \notag\\
\le
&8\Ec({{\x}}_{t-1}) + 2\nu^2 \|\tilde{{\x}}_{t-1}-{\x}_{t-1}\|^2 
\end{align}
and also 
\begin{align}
\|{{\y}}_t-{{\y}}_{t-1}\|^2 
\le
8\Ec({{\y}}_{t-1}) + 2\eta^2 \|\tilde{{\y}}_{t-1}-{\y}_{t-1}\|^2 
\end{align}
\end{proof}

\begin{lem}[Differential Bound on Estimator]\label{Lem: Differential Bound GT-GDA}
Under Assumption \ref{Assump: obj}, the following inequalities hold:
\begin{align}
&\sum_{t = 1}^{T}
\Eb\|\v_t - \v_{t-1}\|^2
 \le \sum_{t = 1}^{T} 3L_F^2\Eb\|{{\x}}_{t-1}-{{\x}}_{t}\|^2 + 3L_F^2\Eb\|{{\y}}_{t-1} - {{\y}}_{t}\|^2, \\
&\sum_{t = 1}^{T}
\Eb\|\u_t - \u_{t-1}\|^2
 \le \sum_{t = 1}^{T} 3L_F^2\Eb\|{{\x}}_{t-1}-{{\x}}_{t}\|^2 + 3L_F^2\Eb\|{{\y}}_{t-1} - {{\y}}_{t}\|^2 .
\end{align}
\end{lem}
\begin{proof}
For $\|\v_t - \v_{t-1}\|^2$, we have 
\begin{align}
&
\Eb\|\v_t - \v_{t-1}\|^2 = \Eb\|\v_t -\nabla_{{{\x}}}\F_t + \nabla_{{{\x}}}\F_t - \nabla_{{{\x}}}\F_{t-1} + \nabla_{{{\x}}}\F_{t-1} - \v_{t-1}\|^2  \notag\\
\le
& 3\Eb\|\v_t - \nabla_{{{\x}}} F_{t}\|^2  + 3 \Eb\|\nabla_{{{\x}}} F_{t}  - \nabla_{{{\x}}} F_{t-1}\|^2 + 3\Eb\|\nabla_{{{\x}}} F_{t-1} - \v_{t-1}\|^2 \notag\\
\le
& 3L_F \Eb\|{{\x}}_{t-1}-{{\x}}_{t}\|^2 + 3L_F^2\Eb\|{{\y}}_{t-1} - {{\y}}_{t}\|^2 .
\end{align}
Thus, we have: $\sum_{t = 1}^{T}
\Eb\|\v_t - \v_{t-1}\|^2
 \le \sum_{t = 1}^{T} 3L_F^2\Eb\|{{\x}}_{t-1}-{{\x}}_{t}\|^2 + 3L_F^2\Eb\|{{\y}}_{t-1} - {{\y}}_{t}\|^2$, and similarly, 
$
\sum_{t = 1}^{T}
\Eb\|\u_t - \u_{t-1}\|^2
 \le \sum_{t = 1}^{T} 3L_F^2\Eb\|{{\x}}_{t-1}-{{\x}}_{t}\|^2 + 3L_F^2\Eb\|{{\y}}_{t-1} - {{\y}}_{t}\|^2 $.
\end{proof}

\subsection{Proof of Lemma \ref{Lem: SRVR err}}

Next, we bound the error of the gradient estimators as the follows:
\begin{proof}
	From the algorithm update, we have:
	\begin{align}
		&
		\|\underbrace{\bd_{i,t} \!-\! \nabla_{{\bm{x}}} F_{i,t}}_{A_{i,t}}\|^2 
		\!=\!
		\| \bd_{i,t-1}  
		\!+\!  \frac{1}{|\Sc_{i,t}|}\!\! \sum_{j \in \Sc_{i,t}}\!\! \!\nabla_{{\bm{x}}}  f_{i,j}({\bm{x}}_{i,t},{{\y}}_{i,t})
		\!-\! \nabla_{{\bm{x}}}  f_{i,j}({\bm{x}}_{i,t-1},{{\y}}_{i,t-1})
		\!-\! \nabla_{{\bm{x}}} F_{i,t} \|^2 \notag\\
		\!=\!
		&
		\|
		\underbrace{\bd_{i,t-1} \!-\! \nabla_{{\bm{x}}} F_{i,t-1}}_{A_{i,t-1}} 
		\!+\!
		\underbrace{ \frac{1}{|\Sc_{i,t}|}\!\!\sum_{j \in \Sc_{i,t}}\!\! \! \nabla_{{\bm{x}}} f_{i,t}({\bm{x}}_{i,t},{{\y}}_{i,t})\!-\! \nabla_{{\bm{x}}}  f_{i,t}({\bm{x}}_{i,t\!-\!1},{{\y}}_{i,t\!-\!1})
			\!+ \!\nabla_{{\bm{x}}} F_{i,t\!-\!1}  \!-\! \nabla_{{\bm{x}}} F_{i,t}}_{B_{i,t}} \|^2 \notag \\
		=
		&
		\|A_{i,t-1}\|^2 + \|B_{i,t}\|^2 
		+ 2 \langle  A_{i,t-1}, B_{i,t}\rangle.
	\end{align}
	Note that $\Eb_t[B_{i,t}] = 0$, where the expectation is taken over the randomness in $t$th iteration. Thus,
	\begin{align}
		\Eb_t\|A_{i,t}\|^2
		=
		\|A_{i,t-1}\|^2 + \Eb_t\|B_{i,t}\|^2.
	\end{align}
	
	Also, with $|\Sc_{i,t}| = q$, we have 
	\begin{align}\label{eq69}
		&
		\Eb_t\|B_{i,t}\|^2 
		\!=\!
		\Eb_t\|\frac{1}{|\Sc_{i,t}|}\sum_{j \in \Sc_{i,t}} \!\!\nabla_{{\bm{x}}} f_{i,j}({\bm{x}}_{i,t},{{\y}}_{i,t}) \!- \!\nabla_{{\bm{x}}}  f_{i,j}({\bm{x}}_{i,t-1},{{\y}}_{i,t-1})
		\!-\!  \nabla_{{\bm{x}}} F_{i,t} \!+\! \nabla_{{\bm{x}}} F_{i,t-1}  \|^2\notag\\
		\le
		&
		\frac{1}{|\Sc_{i,t}|^2}\sum_{j \in \Sc_{i,t}}\!\!
		\Eb_t\| \nabla_{{\bm{x}}} f_{i,j}({\bm{x}}_{i,t},{{\y}}_{i,t}) \!-\! \nabla_{{\bm{x}}}  f_{i,j}({\bm{x}}_{i,t-1},{{\y}}_{i,t-1})
		\!- \! \nabla_{{\bm{x}}} F_{i,t} + \nabla_{{\bm{x}}} F_{i,t-1}  \|^2
		\notag\\
		\le
		&
		\frac{L_f^2}{q}\big(\|{\bm{x}}_{i,t} - {\bm{x}}_{i,t-1}\|^2 + \|{{\y}}_{i,t} - {{\y}}_{i,t-1}\|^2\big).
	\end{align}
	
	Taking full expectation and telescoping (\ref{eq69}) over $t$ from $(n_t-1)q + 1$ to $t$, where $t\le n_t q - 1$, we have
	\begin{align}
		\Eb\|A_{t}\|^2
		&
		\le 
		\Eb\|A_{t-1}\|^2 + 
		\frac{L_f^2}{q} \Eb\big(\|{\bm{x}}_{t} - {\bm{x}}_{t-1}\|^2 + \|{{\y}}_{t} - {{\y}}_{t-1}\|^2\big) \notag\\
		&
		\le 
		\Eb\|A_{(n_t-1)q}\|^2 + 
		\sum_{r = (n_t-1)q + 1}^{t}
		\frac{L_f^2}{q} \Eb\big(\|{\bm{x}}_{r} - {\bm{x}}_{r-1}\|^2 + \|{{\y}}_{r} - {{\y}}_{r-1}\|^2\big).
	\end{align}
	
	Thus, we have:
	\begin{align}
		&
		\sum_{k= 0}^{t}
		\Eb\|A_{k}\|^2 = 
		\sum_{k= 0}^{q-1} \Eb\|A_{k}\|^2
		+\cdots
		+
		\sum_{k=(n_t-1)q}^{t} \Eb\|A_{k}\|^2 \notag\\
		\le 
		&
		q\|A_{0}\|^2 + 
		\sum_{k= 1}^{q-1}\sum_{r= 1}^{k}
		\frac{L_f^2}{q}\big(\|{\bm{x}}_{r} - {\bm{x}}_{r-1}\|^2 + \|{{\y}}_{r} - {{\y}}_{r-1}\|^2\big) \notag\\
		&+\cdots\notag\\
		&+
		\big(t-(n_t-1)q\big)\|A_{(n_t-1)q}\|^2 
		+ \!\!\!\!\!\!\!\!
		\sum_{k= (n_t-1)q+1}^{t}
		\sum_{r= (n_t-1)q + 1}^{k}\!\!\!\!
		\frac{L_f^2}{q}\big(\|{\bm{x}}_{r} - {\bm{x}}_{r-1}\|^2 + \|{{\y}}_{r} - {{\y}}_{r-1}\|^2\big) \notag\\
		\le
		&
		q\|A_{0}\|^2 + 
		\sum_{r = 1}^{q-1}\sum_{k = r}^{q-1}
		\frac{L_f^2}{q}\big(\|{\bm{x}}_{r} - {\bm{x}}_{r-1}\|^2 + \|{{\y}}_{r} - {{\y}}_{r-1}\|^2\big) \notag\\
		&+\cdots\notag\\
		&+
		\big(t-(n_t-1)q\big)\|A_{(n_t-1)q}\|^2 
		+ \!\!\!\!\!\!\!\!
		\sum_{r= (n_t-1)q + 1}^{t}
		\sum_{k= r}^{t}
		\!
		\frac{L_f^2}{q}\big(\|{\bm{x}}_{r} - {\bm{x}}_{r-1}\|^2 + \|{{\y}}_{r} - {{\y}}_{r-1}\|^2\big) \notag\\
		\le
		&
		q\|A_{0}\|^2 + 
		\sum_{r = 1}^{q-1}
		L_f^2\big(\|{\bm{x}}_{r} - {\bm{x}}_{r-1}\|^2 + \|{{\y}}_{r} - {{\y}}_{r-1}\|^2\big) \notag\\
		&+\cdots\notag\\
		&+
		\big(t-(n_t-1)q\big)\|A_{(n_t-1)q}\|^2 
		+ \!\!\!\!\!\!\!\!
		\sum_{r= (n_t-1)q + 1}^{t}\!\!\!
		L_f^2\big(\|{\bm{x}}_{r} - {\bm{x}}_{r-1}\|^2 + \|{{\y}}_{r} - {{\y}}_{r-1}\|^2\big) \notag\\
		=
		&
		\sum_{r = 0}^{t}\|A_{(n_r-1)q}\|^2  + 
		\sum_{r = 1}^{t}
		L_f^2\big(\|{\bm{x}}_{r} - {\bm{x}}_{r-1}\|^2 + \|{{\y}}_{r} - {{\y}}_{r-1}\|^2\big).
	\end{align}
	
	Thus, we have:
	\begin{align}
		\sum_{t \!=\! 0}^{T}\|\bd_{t}\! -\! \nabla_{{\bm{x}}} F_{t}\|^2 \!\le \!
		\sum_{t \!=\! 0}^{T} \Eb\|\bd_{(n_t\!-\!1)q}\!\!-\!\nabla_{{\bm{x}}} F_{(n_t\!-\!1)q} )\|^2 \!\!\!+ \!
		\sum_{t \!=\! 1}^{T}L_f^2\big(\|{\bm{x}}_{t} \!-\! {\bm{x}}_{t\!-\!1}\|^2 \!\!+ \!\|{{\y}}_{t} \!-\! {{\y}}_{t\!-\!1}\|^2\big)
	\end{align}
	Similarly, we have:
	\begin{align}
		\sum_{t \!=\! 0}^{T}\|\bp_{t}\! -\! \nabla_{{{\y}}} F_{t}\|^2 \!\le \!
		\sum_{t \!=\! 0}^{T} \Eb\|\bp_{(n_t\!-\!1)q}\!-\!\nabla_{{{\y}}} F_{(n_t\!-\!1)q} )\|^2 \!\!\!+ \!
		\sum_{t \!=\! 1}^{T}
		L_f^2\big(\|{\bm{x}}_{t} \!-\! {\bm{x}}_{t\!-\!1}\|^2 \!+ \!\|{{\y}}_{t} \!-\! {{\y}}_{t\!-\!1}\|^2\big).
	\end{align}
	This completes the proof.
\end{proof}

\section{Proof for Theorem \ref{Thm: Prox-GT-SRVR} and Theorem \ref{Thm: Prox-GT-SRVR+}}

With Lemmas \ref{Lemma: Lip_J}-\ref{Lem: Differential Bound GT-GDA} and the defined potential function,
we have:
\begin{align}
	&Q({\bar{{\x}}}_{T+1}) - Q({\bar{{\x}}}_{0}) + \frac{4\nu L_F^2}{\beta}\big[\|{\bar{{\y}}}_{T+1} - {{\y}}_{T+1}^*\|^2  - \|{{\y}}_0^* - {\bar{{\y}}}_0\|^2\big]\notag\\
	\le& \underbrace{ \frac{75 {\eta\alpha}}{16 \mu}\frac{2}{m}\!\sum_{t\!=\!0}^{T}\! \|\nabla_{{{\y}}} F({{\x}}_{t},{{\y}}_{t}) - \bd_{t}\|^2  +\frac{\nu}{2 \beta } \frac{2}{m}\!\sum_{t\!=\!0}^{T}\! \|\nabla_{{{\x}}} F({{\x}}_{t},{{\y}}_{t}) - \bp_{t}\|^2  }_{R_1}-\frac{\nu {L_F}^2}{2} \!\sum_{t\!=\!0}^{T}\! \left\|\bar{{\y}}_{t}-{{\y}}_{t}^{*}\right\|^{2} \notag\\&
	+\frac{\nu \tau}{2 \beta m} \!\sum_{t\!=\!0}^{T}\! \left\|{\x}_t-1 \bar{{\x}}_{t}\right\|^{2}
	+\big[   \frac{\nu}{ \beta }  \frac{ L_F^2}{m}+  \frac{4 \nu L_F^2}{ \beta \mu  \eta\alpha } \frac{75 {\eta\alpha}}{16 \mu}\frac{2L_F^2}{m}   \big]  \!\sum_{t\!=\!0}^{T}\! \sum_{i=1}^{m}[\|{\bar{{\x}}}_t - {{\x}}_{i,t}\|^2 + \|{\bar{{\y}}}_t - {{\y}}_{i,t}\|^2 ]\notag\\
	&-\left(-\frac{17L_{{\y}}^{2}\nu^2 }{2 \mu m {\eta\alpha}}+\frac{\nu \tau}{m}-\frac{\nu^{2} L_J}{2 m}-\frac{\nu \beta}{m}-\frac{\nu \tau \beta}{2 m}\right)\!\sum_{t\!=\!0}^{T}\! \left\|\tilde{{\x}}_{t}-1 \bar{{\x}}_{t}\right\|^{2}- \frac{4 \nu L_F^2}{ \beta \mu{\eta\alpha} }\frac{3 {\eta}}{4}\left\|\tilde{{\y}}_{t}-1 \bar{{\y}}_{t}\right\|^{2} ,
\end{align}

With the defined potential function $\mathfrak{p}$, we have 
\begin{align}
\Eb \mathfrak{p}_{T+1} - \mathfrak{p}_{0} 
&\le
 \underbrace{ \frac{75 {\eta\alpha}}{16 \mu}\frac{2}{m}\!\sum_{t\!=\!0}^{T}\! \|\nabla_{{{\y}}} F({\bm{x}}_{t},{{\y}}_{t}) - \bd_{t}\|^2  +\frac{\nu}{2 \beta } \frac{2}{m}\!\sum_{t\!=\!0}^{T}\! \|\nabla_{{\bm{x}}} F({\bm{x}}_{t},{{\y}}_{t}) - \bp_{t}\|^2  }_{R_1}-\frac{\nu {L_f}^2}{2} \!\sum_{t\!=\!0}^{T}\! \left\|\bar{{\y}}_{t}-{{\y}}_{t}^{*}\right\|^{2} \notag\\&
+\frac{\nu \tau}{2 \beta m} \!\sum_{t\!=\!0}^{T}\! \left\|\bm{x}_t-1 \bar{\bm{x}}_{t}\right\|^{2}
+\big[   \frac{\nu}{ \beta }  \frac{ L_f^2}{m}+  \frac{4 \nu L_f^2}{ \beta \mu{\eta\alpha} } \frac{75 {\eta\alpha}}{16 \mu}\frac{2L_f^2}{m}   \big]  \!\sum_{t\!=\!0}^{T}\! \sum_{i=1}^{m}[\|{\bar{\bm{x}}}_t - {\bm{x}}_{i,t}\|^2 + \|{\bar{{\y}}}_t - {{\y}}_{i,t}\|^2 ]\notag\\
&-\left(-\frac{17L_{{\y}}^{2}\nu^2 }{2 \mu m {\eta\alpha}}+\frac{\nu \tau}{m}-\frac{\nu^{2} L_J}{2 m}-\frac{\nu \beta}{m}-\frac{\nu \tau \beta}{2 m}\right)\!\sum_{t\!=\!0}^{T}\! \left\|\tilde{\bm{x}}_{t}-1 \bar{\bm{x}}_{t}\right\|^{2}-  \frac{4 \nu L_f^2}{ \beta \mu\eta\alpha }\frac{3 {\eta}}{4}\!\sum_{t\!=\!0}^{T}\! \left\|\widetilde{{\y}}_{t+1}-\bar{{\y}}_{t}\right\|^{2},
\end{align}
 
For the term $R_1$, we have 
\begin{align}
&
 \frac{75 {\eta\alpha}}{16 \mu}\frac{2}{m} \sum_{t=0}^{T}\Eb\|\nabla_{{{\y}}} F_t - \bd_{t}\|^2
+\frac{\nu}{2 \beta } \frac{2}{m}\sum_{t=0}^{T}\Eb\|\nabla_{{\bm{x}}} F_t - \bp_{t}\|^2\notag\\
\le
&
 \frac{75 {\eta\alpha}}{16 \mu}\frac{2}{m}\Eb\Big(\sum_{t = 0}^{T}\|\bd_{(n_t-1)q}- \nabla_{{\bm{x}}} F_{(n_t-1)q}\|^2  + 
 \sum_{t = 1}^{T}
 L_f^2\big(\|{\bm{x}}_{t} - {\bm{x}}_{t-1}\|^2 + \|{{\y}}_{t} - {{\y}}_{t-1}\|^2\big)\Big)\notag\\
 &
 +
\frac{\nu}{2 \beta } \frac{2}{m}\Eb\Big(\sum_{t = 0}^{T}\|\bp_{(n_t-1)q}- \nabla_{{{\y}}} F_{(n_t-1)q}\|^2  + 
 \sum_{t = 1}^{T}
 L_f^2\big(\|{\bm{x}}_{t} - {\bm{x}}_{t-1}\|^2 + \|{{\y}}_{t} - {{\y}}_{t-1}\|^2\big)\Big)\notag\\
 =
 &
 L_f^2\big( \frac{75 {\eta\alpha}}{16 \mu}\frac{2}{m} + \frac{\nu}{2 \beta } \frac{2}{m} \big)
\sum_{t = 1}^{T}\Eb\big(\|{\bm{x}}_{t} - {\bm{x}}_{t-1}\|^2 + \|{{\y}}_{t} - {{\y}}_{t-1}\|^2\big) \notag\\
&+
 \frac{75 {\eta\alpha}}{16 \mu}\frac{2}{m}\sum_{t = 0}^{T}\Eb\|\bd_{(n_t-1)q}- \nabla_{{\bm{x}}} F_{(n_t-1)q}\|^2  +\frac{\nu}{2 \beta } \frac{2}{m}\sum_{t = 0}^{T}\Eb\|\bp_{(n_t-1)q}- \nabla_{{{\y}}} F_{(n_t-1)q}\|^2.
\end{align}

Plugging the above results, we have 
\begin{align}
\Eb \mathfrak{p}_{T+1} - \mathfrak{p}_{0} 
&\le-\frac{\nu {L_f}^2}{2} \!\sum_{t\!=\!0}^{T}\! \left\|\bar{{\y}}_{t}-{{\y}}_{t}^{*}\right\|^{2}
+\frac{\nu \tau}{2 \beta m} \!\sum_{t\!=\!0}^{T}\! \left\|\bm{x}_t-1 \bar{\bm{x}}_{t}\right\|^{2}
\notag\\
&-\big[ 1- (1+c_1)\lambda^2-  \frac{\nu}{ \beta }  \frac{ L_f^2}{m}-  \frac{4 \nu L_f^2}{ \beta \mu{\eta\alpha} } \frac{75 {\eta\alpha}}{16 \mu}\frac{2L_f^2}{m}   \big]  \!\sum_{t\!=\!0}^{T}\! \sum_{i=1}^{m}[\|{\bar{\bm{x}}}_t - {\bm{x}}_{i,t}\|^2 ]
\notag\\
&-\big[ 1- (1+c_2)\lambda^2-  \frac{\nu}{ \beta }  \frac{ L_f^2}{m}-  \frac{4 \nu L_f^2}{ \beta \mu{\eta\alpha} } \frac{75 {\eta\alpha}}{16 \mu}\frac{2L_f^2}{m}   \big]  \!\sum_{t\!=\!0}^{T}\! \sum_{i=1}^{m}[ \|{\bar{{\y}}}_t - {{\y}}_{i,t}\|^2 ]
\notag\\
&-\left(- (1+\frac{1}{c_1})\nu^2  -\frac{17L_{{\y}}^{2}\nu^2 }{2 \mu m {\eta\alpha}}+\frac{\nu \tau}{m}-\frac{\nu^{2} L_J}{2 m}-\frac{\nu \beta}{m}-\frac{\nu \tau \beta}{2 m}\right)\!\sum_{t\!=\!0}^{T}\! \left\|\tilde{\bm{x}}_{t}-1 \bar{\bm{x}}_{t}\right\|^{2}\notag\\&
- [   \frac{4 \nu L_f^2}{ \beta \mu{\eta\alpha} }\frac{3 {\eta}}{4} - (1+\frac{1}{c_2})\eta^2   ]  \!\sum_{t\!=\!0}^{T}\! \left\|\widetilde{{\y}}_{t+1}-\bar{{\y}}_{t}\right\|^{2}  
\notag\\
&+
L_f^2\big( \frac{75 {\eta\alpha}}{16 \mu}\frac{2}{m} + \frac{\nu}{2 \beta } \frac{2}{m} \big)
\sum_{t = 1}^{T}\Eb\big(\|{\bm{x}}_{t} - {\bm{x}}_{t-1}\|^2 + \|{{\y}}_{t} - {{\y}}_{t-1}\|^2\big) \notag\\
&+
\frac{75 {\eta\alpha}}{16 \mu}\frac{2}{m}\sum_{t = 0}^{T}\Eb\|\bd_{(n_t-1)q}- \nabla_{{\bm{x}}} F_{(n_t-1)q}\|^2  +\frac{\nu}{2 \beta } \frac{2}{m}\sum_{t = 0}^{T}\Eb\|\bp_{(n_t-1)q}- \nabla_{{{\y}}} F_{(n_t-1)q}\|^2\notag\\
&=-\frac{\nu {L_f}^2}{2} \!\sum_{t\!=\!0}^{T}\! \left\|\bar{{\y}}_{t}-{{\y}}_{t}^{*}\right\|^{2}
\notag\\
&- \underbrace{\big[ 1-8L_f^2\big( \frac{75 {\eta\alpha}}{16 \mu}\frac{2}{m} + \frac{\nu}{2 \beta } \frac{2}{m} \big)-\frac{\nu \tau}{2 \beta m}- (1+c_1)\lambda^2-  \frac{\nu}{ \beta }  \frac{ L_f^2}{m}-  \frac{4 \nu L_f^2}{ \beta \mu\eta^2 } \frac{75 {\eta\alpha}}{16 \mu}\frac{2L_f^2}{m}   \big] }_{C_1} \!\sum_{t\!=\!0}^{T}\! \sum_{i=1}^{m}[\|{\bar{\bm{x}}}_t - {\bm{x}}_{i,t}\|^2 ]
\notag\\
&- \underbrace{\left(-2\nu^2L_f^2\big( \frac{75 {\eta\alpha}}{16 \mu}\frac{2}{m} + \frac{\nu}{2 \beta } \frac{2}{m} \big)- (1+\frac{1}{c_1})\nu^2  -\frac{17L_{{\y}}^{2}\nu^2 }{2 \mu m {\eta\alpha}}+\frac{\nu \tau}{m}-\frac{\nu^{2} L_J}{2 m}-\frac{\nu \beta}{m}-\frac{\nu \tau \beta}{2 m}\right)}_{C_2}\!\sum_{t\!=\!0}^{T}\! \left\|\tilde{\bm{x}}_{t}-1 \bar{\bm{x}}_{t}\right\|^{2}\notag\\
&- \underbrace{\big[ 1-8L_f^2\big( \frac{75 {\eta\alpha}}{16 \mu}\frac{2}{m} + \frac{\nu}{2 \beta } \frac{2}{m} \big)- (1+c_2)\lambda^2-  \frac{\nu}{ \beta }  \frac{ L_f^2}{m}-  \frac{4 \nu L_f^2}{ \beta \mu{\eta\alpha} } \frac{75 {\eta\alpha}}{16 \mu}\frac{2L_f^2}{m}   \big] }_{C_3} \!\sum_{t\!=\!0}^{T}\! \sum_{i=1}^{m}[ \|{\bar{{\y}}}_t - {{\y}}_{i,t}\|^2 ]
\notag\\
&
-  \underbrace{[   \frac{4 \nu L_f^2}{ \beta \mu{\eta\alpha} }\frac{3 {\eta}}{4} - (1+\frac{1}{c_2}){\eta^2}  -2{\eta\alpha}L_f^2\big( \frac{75 {\eta\alpha}}{16 \mu}\frac{2}{m} + \frac{\nu}{2 \beta } \frac{2}{m} \big) ] }_{C_4} \!\sum_{t\!=\!0}^{T}\! \left\|\widetilde{{\y}}_{t}-\bar{{\y}}_{t}\right\|^{2}  
\end{align}

For \alg, the outer loop calculates the full gradients. 
Thus, we have $\Eb\|\bd_{(n_t-1)q}- \nabla_{{\bm{x}}} F_{(n_t-1)q}\|^2 = \Eb\|\bp_{(n_t-1)q}- \nabla_{{{\y}}} F_{(n_t-1)q}\|^2 = 0$.

Choosing $c_1=c_2= \frac{1-\lambda^2}{1+\lambda^2}$, we have
\begin{align}
&C_1={\big[ 1-8L_f^2\big( \frac{75 {\eta\alpha}}{16 \mu}\frac{2}{m} + \frac{\nu}{2 \beta } \frac{2}{m} \big)-\frac{\nu \tau}{2 \beta m}- (1+c_1)\lambda^2-  \frac{\nu}{ \beta }  \frac{ L_f^2}{m}-  \frac{4 \nu L_f^2}{ \beta \mu{\eta\alpha} } \frac{75 {\eta\alpha}}{16 \mu}\frac{2L_f^2}{m}   \big] }\notag\\
&
\geq{\big[ 1-8L_f^2\big( \frac{75 {\eta\alpha}}{16 \mu}\frac{2}{m} + \frac{\nu}{2 \beta } \frac{2}{m} \big)-\frac{\nu \tau}{2 \beta m}- (1-c_1)-  \frac{\nu}{ \beta }  \frac{ L_f^2}{m}-  \frac{4 \nu L_f^2}{ \beta \mu{\eta\alpha} } \frac{75 {\eta\alpha}}{16 \mu}\frac{2L_f^2}{m}   \big] }\notag\\
&
\geq c_1- \frac{c_1}{5}- \frac{c_1}{5}- \frac{c_1}{5}-\frac{c_1}{5}-\frac{c_1}{5}=0
\end{align}
\begin{align}
C_2&={\left(-2\nu^2L_f^2\big( \frac{75 {\eta\alpha}}{16 \mu}\frac{2}{m} + \frac{\nu}{2 \beta } \frac{2}{m} \big)- (1+\frac{1}{c_1})\nu^2  -\frac{17L_{{\y}}^{2}\nu^2 }{2 \mu m {\eta\alpha}}+\frac{\nu \tau}{m}-\frac{\nu^{2} L_J}{2 m}-\frac{\nu \beta}{m}-\frac{\nu \tau \beta}{2 m}\right)}\notag\\&\geq -\frac{\nu \tau}{6m}-\frac{\nu \tau}{6m}-\frac{\nu \tau}{6m}+\frac{\nu \tau}{m}-\frac{\nu \tau}{6m}-\frac{\nu \tau}{12m}-\frac{\nu \tau}{6m}>0
\end{align}

\begin{align}
C_3=&{\big[ 1-8L_f^2\big( \frac{75 {\eta\alpha}}{16 \mu}\frac{2}{m} + \frac{\nu}{2 \beta } \frac{2}{m} \big)- (1+c_2)\lambda^2-  \frac{\nu}{ \beta }  \frac{ L_f^2}{m}-  \frac{4 \nu L_f^2}{ \beta \mu{\eta\alpha} } \frac{75 {\eta\alpha}}{16 \mu}\frac{2L_f^2}{m}   \big] }\notag\\&
\geq
 {\big[ 1-8L_f^2\big( \frac{75 {\eta\alpha}}{16 \mu}\frac{2}{m} + \frac{\nu}{2 \beta } \frac{2}{m} \big)- (1-c_2)-  \frac{\nu}{ \beta }  \frac{ L_f^2}{m}-  \frac{4 \nu L_f^2}{ \beta \mu{\eta\alpha} } \frac{75 {\eta\alpha}}{16 \mu}\frac{2L_f^2}{m}   \big] }\notag\\&
 \geq c_2-\frac{c_2}{5}-\frac{c_2}{5}-\frac{c_2}{5}-\frac{c_2}{5}-\frac{c_2}{5}=0
\end{align}

\begin{align}
C_4={[   \frac{4 \nu L_f^2}{ \beta \mu{\eta\alpha} }\frac{3 {\eta}}{4} - (1+\frac{1}{c_2})\eta^2  -2{\eta\alpha}L_f^2\big( \frac{75 {\eta\alpha}}{16 \mu}\frac{2}{m} + \frac{\nu}{2 \beta } \frac{2}{m} \big) ] }\geq   \frac{4 \nu L_f^2}{ \beta \mu{\eta\alpha} }\frac{3 {\eta}}{4} -   \frac{4 \nu L_f^2}{ \beta \mu{\eta\alpha} }\frac{3 {\eta}}{8}-  \frac{4 \nu L_f^2}{ \beta \mu{\eta\alpha} }\frac{3 {\eta}}{8} =0
\end{align}
With parameters 
\begin{align} 
&\eta\leq  \min\{ \frac{c_1 m \mu}{ 375\alpha L_f^2},  \frac{15L_f^2}{\beta \mu \alpha^2c_1} ,\frac{3c_1^2 m}{10(1+c_1)\mu \alpha}\}\notag\\
&\nu\leq\min\{ \frac{c_1 m\beta}{40L_f^2} ,\frac{2c_1 m\beta}{5\tau} ,\frac{2c_1\beta\mu^2 m}{375L_f^4} ,\frac{5 \tau}{3m c_1} , \frac{\tau}{6 m (1+1/c_1)}, \frac{3\mu\eta \alpha\tau}{17L_f^2} ,\frac{\tau}{3(L_f+\frac{L_f^2}{\mu}) } \} \notag\\
&\beta\leq \min\{ \frac{\tau}{12}  ,\frac{1}{3}\},
\end{align}


we have the stated result for \alg:
\begin{align}
\sum_{t=0}^{T} \Big(\E[ 
\left\|\tilde{\bm{x}}_{t}-1 \bar{\bm{x}_t}\right\|^{2}+\left\|\bm{x}_{t}-1 \bar{\bm{x}}_{t}\right\|^{2}]
+ \Eb\|{{\y}}_t^* - {\bar{{\y}}}_t\|^2 
\Big) 
\le 
\frac{\Eb [\mathfrak{p}_{0} - \mathfrak{p}_{T+1}] }{\min\{ C_1,C_2,\nu L_f^2/2\}} .
\end{align}

For \algplusns, we have that 
\begin{align}
&\Eb\|\bd_{(n_t-1)q}- \nabla_{{\bm{x}}} F_{(n_t-1)q}\|^2 = \Eb\|\bp_{(n_t-1)q}- \nabla_{{{\y}}} F_{(n_t-1)q}\|^2=\frac{I_{\left(\mathcal{N}_{s}<M\right)}}{\mathcal{N}_{s}} \sigma^{2}
\end{align}
	
Recall that $\mathcal{N}_s =\min\{ c_{\gamma} \sigma^2(\gamma^{(k)})^{-1}, c_{\epsilon} \sigma^2\epsilon^{-1} ,M\}  $. Then we have
\begin{align}  \label{eqn:84}
\frac{I_{( \mathcal{N}_s <M)}}{ \mathcal{N}_s}  & \leq \frac{1}{ \min\{ c_{\epsilon} \sigma^2(\epsilon)^{-1} ,c_{\gamma} \sigma^2({\gamma^{(k)}})^{-1}     \}} \notag\\
& = \max\{ \frac{{\gamma^{(k)}}}{c_{\gamma}\sigma^2}, \frac{\epsilon}{c_{\epsilon} \sigma^2}  \}  \leq \frac{{\gamma^{(k)}}}{  c_{\gamma}\sigma^2}+\frac{\epsilon}{ c_{\epsilon}\sigma^2}.
\end{align}
Thus, we have
\begin{align}
&Q({\bar{\bm{x}}}_{T+1}) - Q({\bar{\bm{x}}}_{0}) + \frac{4\nu L_F^2}{\beta}\big[\|{\bar{{\y}}}_{T+1} - {{\y}}_{T+1}^*\|^2  - \|{{\y}}_0^* - {\bar{{\y}}}_0\|^2\big]\notag\\
\le& \underbrace{ \frac{75 {\eta\alpha}}{16 \mu}\frac{2}{m}\!\sum_{t\!=\!0}^{T}\! \|\nabla_{{{\y}}} F({\bm{x}}_{t},{{\y}}_{t}) - \bd_{t}\|^2  +\frac{\nu}{2 \beta } \frac{2}{m}\!\sum_{t\!=\!0}^{T}\! \|\nabla_{{\bm{x}}} F({\bm{x}}_{t},{{\y}}_{t}) - \bp_{t}\|^2  }_{R_1}-\frac{\nu {L_F}^2}{2} \!\sum_{t\!=\!0}^{T}\! \left\|\bar{{\y}}_{t}-{{\y}}_{t}^{*}\right\|^{2} \notag\\&
+\frac{\nu \tau}{2 \beta m} \!\sum_{t\!=\!0}^{T}\! \left\|\bm{x}_t-1 \bar{\bm{x}}_{t}\right\|^{2}
+\big[   \frac{\nu}{ \beta }  \frac{ L_F^2}{m}+  \frac{4 \nu L_F^2}{ \beta \mu  \eta\alpha } \frac{75 {\eta\alpha}}{16 \mu}\frac{2L_F^2}{m}   \big]  \!\sum_{t\!=\!0}^{T}\! \sum_{i=1}^{m}[\|{\bar{\bm{x}}}_t - {\bm{x}}_{i,t}\|^2 + \|{\bar{{\y}}}_t - {{\y}}_{i,t}\|^2 ]\notag\\
&-\left(-\frac{17L_{{\y}}^{2}\nu^2 }{2 \mu m {\eta\alpha}}+\frac{\nu \tau}{m}-\frac{\nu^{2} L_J}{2 m}-\frac{\nu \beta}{m}-\frac{\nu \tau \beta}{2 m}\right)\!\sum_{t\!=\!0}^{T}\! \left\|\tilde{\bm{x}}_{t}-1 \bar{\bm{x}}_{t}\right\|^{2}- \frac{4 \nu L_F^2}{ \beta \mu{\eta\alpha} }\frac{3 {\eta}}{4}\left\|\tilde{{\y}}_{t}-1 \bar{{\y}}_{t}\right\|^{2} \notag\\
\leq& ( \frac{75 {\eta\alpha}}{16 \mu}\frac{2}{m} +\frac{\nu}{2 \beta } \frac{2}{m}  )  \sum_{t\!=\!0}^{T}  ( \frac{{\gamma^{(t)}}}{  c_{\gamma}}+\frac{\epsilon}{ c_{\epsilon}})  -\frac{\nu {L_F}^2}{2} \!\sum_{t\!=\!0}^{T}\! \left\|\bar{{\y}}_{t}-{{\y}}_{t}^{*}\right\|^{2} \notag\\&
+\frac{\nu \tau}{2 \beta m} \!\sum_{t\!=\!0}^{T}\! \left\|\bm{x}_t-1 \bar{\bm{x}}_{t}\right\|^{2}
+\big[   \frac{\nu}{ \beta }  \frac{ L_F^2}{m}+  \frac{4 \nu L_F^2}{ \beta \mu  \eta\alpha } \frac{75 {\eta\alpha}}{16 \mu}\frac{2L_F^2}{m}   \big]  \!\sum_{t\!=\!0}^{T}\! \sum_{i=1}^{m}[\|{\bar{\bm{x}}}_t - {\bm{x}}_{i,t}\|^2 + \|{\bar{{\y}}}_t - {{\y}}_{i,t}\|^2 ]\notag\\
&-\left(-\frac{17L_{{\y}}^{2}\nu^2 }{2 \mu m {\eta\alpha}}+\frac{\nu \tau}{m}-\frac{\nu^{2} L_J}{2 m}-\frac{\nu \beta}{m}-\frac{\nu \tau \beta}{2 m}\right)\!\sum_{t\!=\!0}^{T}\! \left\|\tilde{\bm{x}}_{t}-1 \bar{\bm{x}}_{t}\right\|^{2}- \frac{4 \nu L_F^2}{ \beta \mu{\eta\alpha} }\frac{3 {\eta}}{4}\left\|\tilde{{\y}}_{t}-1 \bar{{\y}}_{t}\right\|^{2} \notag\\&
\end{align}
Since $\gamma_{t+1}=\frac{1}{q} \sum_{i=\left(n_{k}-1\right) q}^{k}\left\| \tilde{\bm{x}}_{t}-1 \bar{\bm{x}}_{t} \right\|^{2}$.

\begin{align}
&\Eb \mathfrak{p}_{T+1} - \mathfrak{p}_{0} 
\le -\frac{\nu {L_f}^2}{2} \!\sum_{t\!=\!0}^{T}\! \left\|\bar{{\y}}_{t}-{{\y}}_{t}^{*}\right\|^{2}
\notag\\
&- \underbrace{\big[ 1-8L_f^2\big( \frac{75 {\eta\alpha}}{16 \mu}\frac{2}{m} + \frac{\nu}{2 \beta } \frac{2}{m} \big)-\frac{\nu \tau}{2 \beta m}- (1+c_1)\lambda^2-  \frac{\nu}{ \beta }  \frac{ L_f^2}{m}-  \frac{4 \nu L_f^2}{ \beta \mu\eta^2 } \frac{75 {\eta\alpha}}{16 \mu}\frac{2L_f^2}{m}   \big] }_{C_1} \!\sum_{t\!=\!0}^{T}\! \sum_{i=1}^{m}[\|{\bar{\bm{x}}}_t - {\bm{x}}_{i,t}\|^2 ]
\notag\\
&- \underbrace{  \left(   c_{\gamma} ( \frac{75 {\eta\alpha}}{16 \mu}\frac{2}{m} +\frac{\nu}{2 \beta } \frac{2}{m}  )-2\nu^2L_f^2\big( \frac{75 {\eta\alpha}}{16 \mu}\frac{2}{m} + \frac{\nu}{2 \beta } \frac{2}{m} \big)- (1+\frac{1}{c_1})\nu^2  -\frac{17L_{{\y}}^{2}\nu^2 }{2 \mu m {\eta\alpha}}+\frac{\nu \tau}{m}-\frac{\nu^{2} L_J}{2 m}-\frac{\nu \beta}{m}-\frac{\nu \tau \beta}{2 m}\right)}_{C_2}\!\sum_{t\!=\!0}^{T}\! \left\|\tilde{\bm{x}}_{t}-1 \bar{\bm{x}}_{t}\right\|^{2}\notag\\
&- \underbrace{\big[ 1-8L_f^2\big( \frac{75 {\eta\alpha}}{16 \mu}\frac{2}{m} + \frac{\nu}{2 \beta } \frac{2}{m} \big)- (1+c_2)\lambda^2-  \frac{\nu}{ \beta }  \frac{ L_f^2}{m}-  \frac{4 \nu L_f^2}{ \beta \mu{\eta\alpha} } \frac{75 {\eta\alpha}}{16 \mu}\frac{2L_f^2}{m}   \big] }_{C_3} \!\sum_{t\!=\!0}^{T}\! \sum_{i=1}^{m}[ \|{\bar{{\y}}}_t - {{\y}}_{i,t}\|^2 ]
\notag\\
&
-  \underbrace{[   \frac{4 \nu L_f^2}{ \beta \mu{\eta\alpha} }\frac{3 {\eta}}{4} - (1+\frac{1}{c_2}){\eta^2}  -2{\eta\alpha}L_f^2\big( \frac{75 {\eta\alpha}}{16 \mu}\frac{2}{m} + \frac{\nu}{2 \beta } \frac{2}{m} \big) ] }_{C_4} \!\sum_{t\!=\!0}^{T}\! \left\|\widetilde{{\y}}_{t}-\bar{{\y}}_{t}\right\|^{2}  
\end{align}

Choosing $c_1=c_2= \frac{2\lambda_{m}(M)}{\lambda}$, we have
\begin{align}
&C_1={\big[ 1-8L_f^2\big( \frac{75 {\eta\alpha}}{16 \mu}\frac{2}{m} + \frac{\nu}{2 \beta } \frac{2}{m} \big)-\frac{\nu \tau}{2 \beta m}- (1+c_1)\lambda^2-  \frac{\nu}{ \beta }  \frac{ L_f^2}{m}-  \frac{4 \nu L_f^2}{ \beta \mu{\eta\alpha} } \frac{75 {\eta\alpha}}{16 \mu}\frac{2L_f^2}{m}   \big] }\notag\\
&
\geq{\big[ 1-8L_f^2\big( \frac{75 {\eta\alpha}}{16 \mu}\frac{2}{m} + \frac{\nu}{2 \beta } \frac{2}{m} \big)-\frac{\nu \tau}{2 \beta m}- (1-c_1)-  \frac{\nu}{ \beta }  \frac{ L_f^2}{m}-  \frac{4 \nu L_f^2}{ \beta \mu{\eta\alpha} } \frac{75 {\eta\alpha}}{16 \mu}\frac{2L_f^2}{m}   \big] }\notag\\
&
\geq c_1- \frac{c_1}{5}- \frac{c_1}{5}- \frac{c_1}{5}-\frac{c_1}{5}-\frac{c_1}{5}=0
\end{align}
\begin{align}
C_2'&={\left( c_{\gamma} ( \frac{75 {\eta\alpha}}{16 \mu}\frac{2}{m} +\frac{\nu}{2 \beta } \frac{2}{m}  )-2\nu^2L_f^2\big( \frac{75 {\eta\alpha}}{16 \mu}\frac{2}{m} + \frac{\nu}{2 \beta } \frac{2}{m} \big)- (1+\frac{1}{c_1})\nu^2  -\frac{17L_{{\y}}^{2}\nu^2 }{2 \mu m {\eta\alpha}}+\frac{\nu \tau}{m}-\frac{\nu^{2} L_J}{2 m}-\frac{\nu \beta}{m}-\frac{\nu \tau \beta}{2 m}\right)}\notag\\&
\geq -\frac{\nu \tau}{12m}-\frac{\nu \tau}{6m}-\frac{\nu \tau}{6m}-\frac{\nu \tau}{6m}+\frac{\nu \tau}{m}-\frac{\nu \tau}{6m}-\frac{\nu \tau}{12m}-\frac{\nu \tau}{6m}=0
\end{align}

\begin{align}
C_3=&{\big[ 1-8L_f^2\big( \frac{75 {\eta\alpha}}{16 \mu}\frac{2}{m} + \frac{\nu}{2 \beta } \frac{2}{m} \big)- (1+c_2)\lambda^2-  \frac{\nu}{ \beta }  \frac{ L_f^2}{m}-  \frac{4 \nu L_f^2}{ \beta \mu{\eta\alpha} } \frac{75 {\eta\alpha}}{16 \mu}\frac{2L_f^2}{m}   \big] }\notag\\&
\geq
{\big[ 1-8L_f^2\big( \frac{75 {\eta\alpha}}{16 \mu}\frac{2}{m} + \frac{\nu}{2 \beta } \frac{2}{m} \big)- (1-c_2)-  \frac{\nu}{ \beta }  \frac{ L_f^2}{m}-  \frac{4 \nu L_f^2}{ \beta \mu{\eta\alpha} } \frac{75 {\eta\alpha}}{16 \mu}\frac{2L_f^2}{m}   \big] }\notag\\&
\geq c_2-\frac{c_2}{5}-\frac{c_2}{5}-\frac{c_2}{5}-\frac{c_2}{5}-\frac{c_2}{5}=0
\end{align}

\begin{align}
C_4={[   \frac{4 \nu L_f^2}{ \beta \mu{\eta\alpha} }\frac{3 {\eta}}{4} - (1+\frac{1}{c_2})\eta^2  -2{\eta\alpha}L_f^2\big( \frac{75 {\eta\alpha}}{16 \mu}\frac{2}{m} + \frac{\nu}{2 \beta } \frac{2}{m} \big) ] }\geq   \frac{4 \nu L_f^2}{ \beta \mu{\eta\alpha} }\frac{3 {\eta}}{4} -   \frac{4 \nu L_f^2}{ \beta \mu{\eta\alpha} }\frac{3 {\eta}}{8}-  \frac{4 \nu L_f^2}{ \beta \mu{\eta\alpha} }\frac{3 {\eta}}{8} =0
\end{align}
With parameters 
\begin{align} &
c_{\gamma} \geq  (\frac{75 {\eta\alpha}}{8 \mu}\frac{1}{m} +\frac{\nu}{ \beta } \frac{1}{m} ) \frac{\nu\tau}{12} \notag\\
&\eta\leq  \min\{ \frac{c_1 m \mu}{ 375\alpha L_f^2},  \frac{15L_f^2}{\beta \mu \alpha^2c_1} ,\frac{3c_1^2 m}{10(1+c_1)\mu \alpha}\}\notag\\
&\nu\leq\min\{ \frac{c_1 m\beta}{40L_f^2} ,\frac{2c_1 m\beta}{5\tau} ,\frac{2c_1\beta\mu^2 m}{375L_f^4} ,\frac{5 \tau}{3m c_1} , \frac{\tau}{6 m (1+1/c_1)}, \frac{3\mu\eta \alpha\tau}{17L_f^2} ,\frac{\tau}{3(L_f+\frac{L_f^2}{\mu}) } \} \notag\\
&\beta\leq \min\{ \frac{\tau}{12}  ,\frac{1}{3}\},
\end{align}

Thus, for \algplusns, we have the following convergence results:
\begin{align}
&\frac{1}{(T+1)}\sum_{t=0}^{T} \Big(\E[ 
\left\|\tilde{\bm{x}}_{t}-1 \bar{\bm{x}_t}\right\|^{2}+\left\|\bm{x}_{t}-1 \bar{\bm{x}}_{t}\right\|^{2}]
+ \Eb\|{{\y}}_t^* - {\bar{{\y}}}_t\|^2 
\Big) \notag\\
\le 
&\frac{\Eb [\mathfrak{p}_{0} - \mathfrak{p}_{T+1}] }{(T+1)\min\{ C_1,C_2',\nu L_f^2/2\}} + 
 ( \frac{75 {\eta\alpha}}{16 \mu}\frac{2}{m} +\frac{\nu}{2 \beta } \frac{2}{m}  ) \frac{ \epsilon}{c_{\epsilon} }.
\end{align}

With $\mathfrak{p}_{T+1} \ge Q^*$, we reach the conclusion.

\section{Supporting lemmas}

\begin{lem}\label{Lemma: Lip_w}
Under Assumption~\ref{Assump: obj}, ${\bm{y}}^*({\bm{x}}) = \arg\max_{{\bm{y}}} F({\bm{x}},{\bm{y}})$ is Lipschitz continuous, i.e., there exists a positive constant $L_{{\bm{y}}}$, such that
\begin{align}
\|{\bm{y}}^*({\bm{x}}) - {\bm{y}}^*({\bm{x}}^\prime)\| \le L_{{\bm{y}}} \|{\bm{x}} -{\bm{x}}^\prime\|,~~\forall {\bm{x}}, {\bm{x}}^\prime \in \Rb^d,
\end{align}
where the Lipschitz constant is $L_{{\bm{y}}} = L_f/\mu$.
\end{lem}
\begin{proof}
See Lemma 4.3 in \cite{lin2020gradient}.
\end{proof}

\begin{lem}\label{Lemma: nabla_J}
Under Assumption~\ref{Assump: obj}, the function $J({\bm{x}}) = F({\bm{x}}, {\bm{y}}^*({\bm{x}}))$ satisfies that $\nabla J({\bm{x}}) = \nabla_{{\bm{x}}} F({\bm{x}}, {\bm{y}}^*({\bm{x}}))$.
\end{lem}
\begin{proof}
Since $J({\bm{x}}) = F({\bm{x}}, {\bm{y}}^*({\bm{x}}))$, by chain rule, we have
\begin{align}
d J({\bm{x}}) = \frac{\partial F({\bm{x}}, {\bm{y}})}{\partial {\bm{x}}}\Big|_{{\bm{y}} = {\bm{y}}^*({\bm{x}})} \cdot d{\bm{x}} + \frac{\partial F({\bm{x}}, {\bm{y}})}{\partial {\bm{y}}}\Big|_{{\bm{y}} = {\bm{y}}^*({\bm{x}})} \cdot \frac{\partial \omega^*({\bm{x}})}{\partial {\bm{x}}} \cdot d{\bm{x}},
\end{align}
where $\partial F({\bm{x}}, {\bm{y}})/\partial {\bm{x}}$ and $\partial F({\bm{x}}, {\bm{y}})/\partial {\bm{y}}$ are respectively the partial differential of $F$ w.r.t the first variate ${\bm{x}}$ and the second variate ${\bm{y}}$. Note that ${\bm{y}}^*({\bm{x}})$ is the unique optimal point such that $F({\bm{x}}, {\bm{y}})$ reaches the maximums. So, it follows that $\frac{\partial F({\bm{x}}, {\bm{y}})}{\partial {\bm{y}}}|_{{\bm{y}} = {\bm{y}}^*({\bm{x}})} = 0$ for all ${\bm{x}}$. Also, from Lemma \ref{Lemma: Lip_w}, we have $\partial \omega^*({\bm{x}})/\partial {\bm{x}}$ is bounded. Thus, it follows that 
\begin{align}\label{Eq: nabla_J = nablda F}
d J({\bm{x}}) = \frac{\partial F({\bm{x}}, {\bm{y}})}{\partial {\bm{x}}}\Big|_{{\bm{y}} = {\bm{y}}^*({\bm{x}})} \cdot d{\bm{x}},
\end{align}
which is $\nabla J({\bm{x}}) = \nabla_{{\bm{x}}} F({\bm{x}}, {\bm{y}}^*({\bm{x}}))$. 

\end{proof}

\end{document}